\documentclass[preprint,12pt]{elsarticle}

\usepackage{todonotes}
\usepackage{amssymb}
\usepackage{amsmath}
\usepackage{amsthm}
\usepackage{amsfonts}
\usepackage{todonotes}
\usepackage{float}
\usepackage{graphicx}
\usepackage{tikz}
\usepackage{stackengine}
\usepackage{algorithm}
\usepackage[noend]{algpseudocode}
\usepackage{bm}
\usepackage{subcaption}
\usepackage{tablefootnote}
\usepackage{pgfplots}
\pgfplotsset{compat=1.10}
\usetikzlibrary{patterns}
\usetikzlibrary{shapes.misc}
\tikzset{cross/.style={cross out, draw=black, minimum size=2*(#1-\pgflinewidth), inner sep=0pt, outer sep=0pt},
cross/.default={2pt}}
\usepackage{multirow}
\usepackage{url}

\usepackage{anyfontsize}
\newcommand{\model}[1]{\mathcal{M}_{#1}}

\newcommand{\abst}[2]{\mathcal{I}_{(#1,#2)}}

\newcommand{\card}[1]{\lvert#1\rvert}

\newcommand{\powerset}[1]{\mathbb{I}(#1)}

\DeclareMathOperator*{\argmin}{arg\,min}

\newcommand{\mshift}{S}
\newcommand{\dist}[3]{d_{#1}(#2,#3)}
\newcommand{\distance}[3]{\lVert #1 - #2\rVert_{#3}}

\newcommand{\vectorisation}[1]{\mathbf{vec}(#1)}

 \newcommand{\smallquad}{\hspace{0.5em}} 

\newcommand{\R}{\mathbb{R}}
\newcommand{\name}{\Delta}

\newcommand{\candidates}{\mathcal{S}}
\newcommand{\dataset}{\mathcal{D}}
\newcommand{\tree}{\mathcal{T}}

\newtheorem{lemma}{Lemma}

 \newcommand{\relua}{\draw[line width=1.5pt] (-1em,0) -- (0,0)
                                (0,0) -- (0.75em,0.75em);}

 \newcommand{\sigmoida}{\draw[line width=1.5pt,scale=0.08,samples=100,domain=-6:6,smooth,variable=\x] plot ({\x},{4/(1+exp(-\x))});}

 \newcommand{\deleteFT}[1]{\textcolor{pink}{%$\times$ #1
}}

\newcommand{\FL}[1]{\textcolor{black}{#1}}
\newcommand{\FT}[1]{\textcolor{black}{#1}}

\newcommand{\JJ}[1]{\textcolor{black}{#1}}
\newcommand{\red}[1]{\textcolor{black}{#1}}
\newcommand{\blue}[1]{\textcolor{black}{#1}}

\newcommand{\ouralgo}{RNCE}

\newcommand{\revision}[1]{\textcolor{black}{#1}}

\newtheorem{definition}{Definition}
\newtheorem{remark}{Remark}
\newtheorem{example}{Example}

\algnewcommand\algorithmicforeach{\textbf{for each}}
\algdef{S}[FOR]{ForEach}[1]{\algorithmicforeach\ #1\ \algorithmicdo}

\journal{Artificial Intelligence Journal}

\begin{document}

\begin{frontmatter}



\title{Interval Abstractions for \\ Robust Counterfactual Explanations}


% \author{Junqi Jiang, Francesco Leofante, Antonio Rago, Francesca Toni\corref{cor1}}
\author{Junqi Jiang\corref{cor1}}
\ead{junqi.jiang@imperial.ac.uk}
\author{Francesco Leofante}
\ead{f.leofante@imperial.ac.uk}
\author{Antonio Rago}
\ead{a.rago@imperial.ac.uk}
\author{Francesca Toni}
\ead{f.toni@imperial.ac.uk}

% \ead{{junqi.jiang,f.leofante,a.rago,f.toni}@imperial.ac.uk}

\cortext[cor1]{Corresponding author.}
\affiliation{organization={Department of Computing, Imperial College London},%Department and Organization
            addressline={180 Queen’s Gate},
            city={\\London},
            postcode={SW7 2AZ}, 
            country={United Kingdom}}

\begin{abstract}
Counterfactual Explanations (CEs) have emerged as a major paradigm in explainable AI research, providing recourse recommendations for users affected by the decisions of machine learning models. %However, a well-known problem reveals that CEs could become invalid explanations when slight changes occur in the parameters of the underlying model, often caused by regular model retraining. 
\JJ{However, \FL{CEs found by existing methods often become invalid when slight changes occur in the parameters of the model they were generated for}}. 
%\todo{keeping next sentence as it is to avoid looking the same as AAAI 23 abstract} 
The literature lacks a way to \FL{provide exhaustive} robustness guarantees for CEs under model changes, in that existing methods to improve CEs' robustness are \FL{mostly} heuristic, and the robustness performances are evaluated empirically using only a limited number of retrained models. 
To bridge this gap, we propose a novel interval abstraction technique for parametric machine learning models, \JJ{which allows us to obtain provable robustness guarantees \FL{for} CEs under \FL{a} possibly infinite %family of classification models obtained by applying a predefined 
set of plausible model changes $\Delta$.} %, termed $\Delta$. 
\FL{Based on this idea,} we formalise \FL{a} robustness notion for CEs, \FL{which we call} $\Delta$-robustness, in both binary and multi-class classification settings. We \FL{present} procedures to verify $\Delta$-robustness based on Mixed Integer Linear Programming, using which we further propose algorithms to generate CEs that are $\Delta$-robust. In an extensive empirical study \FL{involving neural networks and logistic regression models}, we demonstrate \FL{the practical applicability of our approach. We} discuss two strategies for determining the appropriate hyperparameter\FL{s} in our method, and we quantitatively benchmark CEs generated by eleven methods, highlighting the effectiveness of our algorithms in finding robust CEs.
\end{abstract}




\end{frontmatter}

% \linenumbers

\section{Introduction}
\label{sec:intro}

As the field of explainable AI (XAI) has matured, \emph{counterfactual explanations} (CEs) have risen to prominence as one of the dominant post-hoc methods for explaining the outputs of AI models (see \cite{Guidotti_22,DBLP:journals/csur/KarimiBSV23survey} for overviews). 
For a given input to a model, a CE essentially presents a user with a \FL{minimally} modified input which results in a different output from the model,
thus \FL{revealing how a different outcome could be achieved if the proposed changes were to be enforced}.
% pointing to the %input features responsible
% \JJ{reasons} for the original output. %The use of CEs has been advocated in AI 
CEs have been \FL{shown} 
to improve human understanding and trust~\cite{Miller_19}, \FL{as they provide information about alternative possibilities that humans can use to build rich mental representations}~\cite{Celar2023}. 
\FL{Initial approaches to generate CEs}~\cite{Tolomei_17,Wachter_17} were optimised for \emph{validity},
i.e. \FL{the ability of a CE to correctly change} the output, and \emph{proximity},
with respect to some distance measure between the original and modified inputs.
{Since then,} 
%many methods for generating CEs have been proposed which optimise for 
additional metrics have been proposed (see \cite{Guidotti_22} for an overview), such as \emph{diversity} \cite{mothilal2020explaining}, i.e. how widely 
{CEs} differ from one another, and \emph{plausibility}~\cite{dhurandhar2018explanations}, i.e. whether the %proposed %counterfactual lies 
{CEs lie} within the data distribution.

A further metric which is receiving increasing attention of late is \emph{robustness}~\cite{DBLP:journals/corr/jiangsurvey24}, i.e. how the validity of a CE is affected by changes in the scenario for which the CE was initially generated. 
%We refer to \cite{DBLP:journals/corr/mishrasurvey21,DBLP:journals/corr/jiangsurvey24} for recent overviews. 
% This includes, for example, changes in the dataset used to train an AI model~\cite{pawelczykRTBF23}, noise applied to the CE itself~\cite{LeofanteLomuscio23b,DBLP:journals/corr/abs-2203-06768},
In this paper, we focus on robustness to slight changes in the AI model parameters induced by, %e.g.
for example, retraining~\cite{upadhyay2021towards,DBLP:conf/iclr/BuiNN22,blackconsistent,pmlr-v162-dutta22a,nguyen2022robust,oursaaai23,pmlr-v222-jiang24a}.  This form of robustness is of critical importance in practice.
% \delete{ In particular, we focus on
% has received less attention is CEs' \emph{robustness to model changes} (hereafter referred to as simply \emph{robustness}), \JJ{do we need to mention robustness more broadly first? apart from robustness to input perturbations, there is a relevant notion of robustness: the validity of CEs after applying a small perturbation on CEs. } 
% namely how valid CEs remain after small changes to model parameters are induced, e.g. after retraining. 
% It is easy to see the critical importance of robustness:}
For illustration, consider a mortgage applicant who was rejected by a model and received a CE demonstrating changes they could make to their situation in order to have their application accepted. If retraining occurs while the applicant makes those %modifications
changes, without robustness, their modified case may still result in a rejected application, leaving the mortgage provider liable due to their conflicting statements. This is %particularly
especially concerning when CEs are optimised for %metrics such as 
proximity, as they are likely to be close to the model's decision boundary %, which raises the risk of them 
and thus at high risk of being invalid if small changes in this boundary occur \cite{upadhyay2021towards}. %Recently, methods have been introduced which target this problem; however, they 

Methods recently introduced to target this problem 
typically 
%\delete{
%make undesirable sacrifices in order to achieve robustness. 
%Specifically, existing
%works} \JJ{the claims here look dangerous} 
\FL{induce robustness to model changes either by using dedicated continuous optimisation procedures} \cite{nguyen2022robust,upadhyay2021towards}, or  %in a post-hoc style 
\FL{by applying post-hoc refinements to} candidate CEs found by non-robust CE generation methods \cite{blackconsistent,pmlr-v162-dutta22a,DBLP:conf/icml/HammanNMMD23}. 
\FL{However}, due to their \FL{approximate} nature, these methods \FL{are unable to provide formal robustness} guarantees, which are advocated as being vital towards 
%the goal of 
achieving trustworthy AI \cite{Marques-Silva_22}. %This means that the CEs found by existing methods cannot be asserted to be provably robust against any type of model change. 

% In this work, we aim to fill the above gap. We propose a novel \emph{interval abstraction} technique for parametric machine learning models (including neural networks and logistic regressions) to formally and deterministically verify the robustness of CEs under the bounded plausible model parameter changes termed $\Delta$. %we use interval abstractions to 
% % Specifically, we show that whether a CE is provably robust against $\Delta$ can be conveniently formulated as an output range analysis problem\todo{this sentence is very convoluted (was also before my change, which made it worse)}. 
% Specifically, we show that the interval abstraction over-approximates the output node ranges when a CE is passed into the model affected by $\Delta$.

% This is achieved by building theoretical connections between the notion of $\Delta$ and the parameters in constructing the interval abstractions, which are then used for estimating the output node ranges when a CE is passed in. Unlike most previous studies which apply to binary classification only, the fact that we are focusing on mechanistic properties of the model output layer allows our method to also work on multi-class classification settings. \todo{this sentence is disconnected from the previous ones - should go later in contributions? Also, this paragraph needs to be rewritten to streamline what we do. JJ: is the above better?}

In this work, we \FL{fill this gap by} present\FL{ing} a method which provides \FL{formal} robustness guarantees for CEs \FL{under model shifts}. Specifically, we \FL{target plausible model shifts~\cite{upadhyay2021towards} that can be encoded by a set $\Delta$ of norm-bounded perturbations to the parameters of machine learning models}. 
% first \JJ{define the plausible model changes, $\Delta$, which {represents} our robustness target. 
\FL{Given this setting}, we propose a novel \emph{interval abstraction} technique, which over-approximates the output ranges of parametric machine learning models (including neural networks and logistic regressions) when subject to \JJ{ the model shifts encoded in $\Delta$.} %the predefined plausible model parameter changes termed $\Delta$. 
This technique is inspired by the abstraction method presented in \cite{PrabhakarA19}, originally proposed for estimating neural networks' outputs by grouping weight edges into weight intervals. We show that using our interval abstraction, a \FL{precise notion of robustness for CEs under $\Delta$} can be stated and formally verified. \FL{In the following, we refer to such notion of robustness as the} $\Delta$-robustness of CEs. 
Unlike most previous robust CE methods which only apply to binary classification, our focus on %mechanistic properties of the model output layer 
\JJ{computing output ranges} 
allows our method to also work on multi-class classification. %To practically test $\Delta$-robustness, we deployed a Mixed Integer Linear Programming (MILP) encoding. 

Using \FL{interval abstractions, we propose two procedures to generate provably robust CEs}: an iterative algorithm \FL{augmenting existing CE generation methods,} and \JJ{a sound and complete} Robust Nearest-neighbour Counterfactual Explanations (RNCE) method. \FL{Since this work uses} \JJ{Mixed Integer Linear Programming (MILP) to practically test $\Delta$-robustness, \FL{the ensuing presentation will only focus on machine learning models whose forward pass can be encoded into a MILP program}.}

Finally, in an extensive empirical study, we demonstrate the effectiveness of our CE generation algorithms in a benchmarking study against seven existing methods. Notably, \JJ{one configuration of our iterative algorithm finds $\Delta$-robust CEs with the lowest costs among all the robust baselines, and our new RNCE
 algorithm %achieves a balance between provable robustness and cost
 achieves perfectly robust results} while finding CEs close to the data manifold.

The paper is structured as follows. In Section~\ref{sec:related} we cover the related work, and in Section~\ref{sec:background} we introduce background notions on computing CEs. Sections~\ref{sec:inns}, \ref{sec:interval_abs_multi_class}, and \ref{sec:milp_delta_rob} introduce a formalisation of $\Delta$-robustness and a MILP encoding for testing it. Then, the two CEs generation algorithms are presented in Section~\ref{sec:algorithms}, which are extensively evaluated through experiments in Section~\ref{sec:experiments}. We conclude in Section~\ref{sec:conclusion} with future research directions. The core contributions of our work are summarised as follows:
\begin{itemize}
    \item \JJ{We propose a novel interval abstraction method to test $\Delta$-robustness, \FL{formally} verifying whether a CE is robust against plausible model changes}. %To the best of our knowledge, our method is the first to do so.%\JJ{To the best of our knowledge, }
    \item Our method explicitly characterises robustness for CEs in multi-class classification, { and, to the best of our knowledge, is the first to do so}.
    \item We present a principled workflow demonstrating the usefulness of $\Delta$-robustness for evaluating and generating robust CEs in practice. 
    \item We introduce an iterative algorithm and the RNCE algorithm to generate provably robust CEs, which are demonstrated to have superior performances against seven baselines.
    % \item \delete{To the best of our knowledge, our method is the first to obtain provable guarantees for the robustness of CEs against model changes.}
\end{itemize}

This paper builds upon our previous work \cite{oursaaai23} with significant extensions. 
% \todo{AR: shall we have this as bullet points too?} \JJ{it's kind of trivial tbh} 
Specifically, Section~\ref{sec:inns} extends the corresponding section in \cite{oursaaai23} to account for different parametric machine learning models in addition to feed-forward neural networks. We present in-depth discussions and relaxations for the soundness of $\Delta$ (Definition~\ref{def:delta_robustness}), and formalisation for multi-class classifications are included with an empirical study (Sections~\ref{sec:interval_abs_multi_class}, \ref{ssec:experiments_multi_class}). Section~\ref{sec:milp_delta_rob} formalises testing procedures for $\Delta$-robustness in terms of MILP programs, which were only briefly mentioned in Appendix B of our previous work. Section~\ref{sec:algorithms} significantly extends the algorithm proposed in \cite{oursaaai23}, which could fail to find provably robust CEs. \revision{We additionally propose a new algorithm, RNCE, (Algorithm~\ref{alg:algo1}), which is guaranteed to return provably robust CEs, while also addressing plausibility, another desirable property of CEs.} In Section~\ref{sec:experiments}, we propose and comprehensively investigate two strategies to find the optimal hyperparameters in our approach, which is not presented in the previous work. Further, the empirical study additionally includes \revision{benchmarking of the runtime performance of our approach,} new results obtained for logistic regression models, and four more CE generation baselines, two of which generate robust CEs, giving a more thorough experimental evaluation of both our approach and the \FL{robustness} research landscape in general. 
% \todo{where to put the next sentence? {looks fine here}} 
Finally, throughout the paper, we have added discussion and examples to give more intuition on the introduced concepts, as well as a more in-depth view of the existing literature.

\section{Related work}
\label{sec:related}

\subsection{Counterfactual explanations}
\label{ssec:related_counterfactual_Explanation}

Various methods for generating CEs in classification tasks have been proposed throughout the recent surge in XAI research, often optimising for one or more metrics characterising desirable properties of CEs. \citet{Tolomei_17} focused on tree-based classifiers and evaluated the CEs' validity, whereas \citet{Wachter_17} formulated the CE search problem for differentiable models as a gradient-based optimisation problem and evaluated also CEs' proximity (we refer their method as \emph{GCE} in Section~\ref{sec:experiments}).
These \FL{two} metrics remain a prominent research focus, e.g. %Mohammadi et al. 
\citet{mohammadi2021scaling} treat CE generation in neural networks as a constrained optimisation problem such that formal validity and proximity guarantees can be given. Various works have considered plausibility, e.g. \citet{NiceNNCE} find dataset points which are naturally on {the} data manifold as CEs (we refer to their method as \emph{NNCE} in Section~\ref{sec:experiments}). Meanwhile, variational auto-encoders have been used to generate plausible\footnote{\JJ{``Plausible'' has been used to describe both the property of CEs and the form of model changes; the specific meanings are clear from the context.}} CEs %, as proposed in 
\citep{dhurandhar2018explanations,pawelczyk2020learning,van2021interpretable}. %,van2021interpretable}, %where~\citet{van2021interpretable} can 
%with the latter alternatively using k-d %\todo{need to choose this notation or KD as used later} 
%trees for the same aim. %\delete{\cite{van2021interpretable} use \delete{both} variational auto-encoders \JJ{or} k-d trees to generate class prototypes with the same aim}, and 
%Mothilal et al. 
Actionability ensures that the CEs only coherently change the mutable features. This is usually dealt with by customising constraints in the optimisation process of finding CEs \cite{ustun2019actionable}.  
\citet{mothilal2020explaining} and \citet{dandl2020multi} build optimisation frameworks %which considers plausibility while also 
for the diversity of the generated CEs. Another line of research focuses on building links between CEs and the causality literature, formulating the problem of finding CEs as intervention operations in causal frameworks \cite{DBLP:conf/nips/KarimiKSV20,DBLP:conf/fat/KarimiSV21}.
%Proximity, plausibility and diversity are also targeted in \cite{karimi2020model} and \cite{dandl2020multi}, formulating CEs search as satisfiability and multi-objective optimisation problems, respectively.

Methods for generating CEs have also been defined for other classifiers, e.g. %Ustun et al.  
\citet{ustun2019actionable} consider different types of linear classification models, %Albini et al. 
\citet{albini2020relation} %address the problem for 
focus on different forms of Bayesian classifier,  and %Kanamori et al. 
\citet{Kanamori_20} target logistic regression and random forest classifiers. Outside the scope of tabular data classification, studies have also investigated CEs for, e.g., graph data tasks \cite{DBLP:conf/nips/BajajCXPWLZ21}, visual tasks \cite{DBLP:conf/nips/AugustinBC022}, time series prediction tasks \cite{DBLP:conf/iccbr/DelaneyGK21}, etc. We refer to \cite{Guidotti_22,DBLP:journals/csur/KarimiBSV23survey} for recent overviews.

\subsection{Robustness of counterfactual explanations}
\label{ssec:related-robustcounterfactuals}

%None of the aforementioned approaches considers the robustness of CEs.
In this work, we consider the robustness of CEs against model changes. %despite %the fact that 
When using traditional methods (i.e. methods that focus on the properties introduced in Section~\ref{ssec:related_counterfactual_Explanation}) to find CEs for some classification models, the resulting CEs are highly unlikely to remain valid when the model parameters are updated  \cite{Rawal_20X,upadhyay2021towards}. As illustrated in Section~\ref{sec:intro}, this could cause issues for both the users receiving the CEs and the {providers of} the explanations. 

Many research studies have been conducted to tackle this problem, aiming at \FL{generating} high-quality, robust CEs. \citet{upadhyay2021towards} adopt a gradient-based robust optimisation approach to generate CEs that are robust to model parameter changes. A similar gradient-based approach is taken by \cite{nguyen2022robust,DBLP:conf/iclr/BuiNN22,DBLP:conf/iclr/NguyenBN23} under probabilistic frameworks where model changes are expressed by probability distributions associated with ambiguity sets. \citet{robovertime} proposes a retraining procedure using counterfactual data augmentation to mitigate the invalidation of previously generated non-robust CEs, while \citet{DBLP:conf/cikm/GuoJCSY23} introduce a robust training framework which jointly optimises the accuracy of neural networks and the robustness of CEs. Another line of work~\cite{blackconsistent,pmlr-v162-dutta22a,DBLP:conf/icml/HammanNMMD23} places more focus on designing heuristics aimed at increasing the model confidence (predicted class probability) to induce more robust CEs . These heuristics are then used as part of certain search-based refining processes to improve the robustness of CEs found by any base CEs generation methods. \revision{Differently from previous works, we work target exhaustive robustness guarantees for CEs against norm-bounded model parameter changes. To practically verify and to compute provably robust CEs, we adopt MILP-based approaches. Though some studies introduce methods {which provide empirical measures of CEs' robustness \cite{DBLP:conf/iclr/BuiNN22,pmlr-v162-dutta22a,DBLP:conf/icml/HammanNMMD23}, no existing approach is able to give the strong formal guarantees which our method affords, to the best of our knowledge.}}

Other forms of robustness of CEs have also been studied in the literature. 
Robustness against input perturbations requires that the CEs generation method not produce drastically different CEs for very similar inputs \cite{Slack_21,DBLP:journals/corr/leofanteaaai24}. Robustness against changes in the training dataset, especially those that resulted in data deletion requests, aims to ensure the CEs' validity under the consequently retrained models \cite{DBLP:conf/icml/KrishnaML23,pawelczykRTBF23righttoforgotten}. The methods in \cite{pawelczyk2020counterfactual,Leofante2023modelmultiplicity,DBLP:journals/corr/jiangaamas24} focus on the implications of the inconsistencies of CEs under predictive multiplicity, where multiple comparable models exist for the same task while assigning conflicting labels for the same inputs. Finally, \JJ{users might only implement the recourse indicated by CEs to an approximate level, instead of achieving the prescribed feature values exactly. It is therefore desirable that CEs stay valid when {subject to} such noisy execution 
%it is also desirable that the validity stays unaffected when small noises are applied to the CEs themselves so that the users can implement recourse indicated by CEs to an approximate level 
\cite{Dominguez-Olmedo_22,LeofanteLomuscio23b,LeofanteLomuscio23a,Virgolin_23,maragno2024finding}. }
The study of these forms of robustness is outside the scope of this paper as our focus is on model changes, we refer to \cite{DBLP:journals/corr/mishrasurvey21,DBLP:journals/corr/jiangsurvey24} for recent surveys.

\subsection{Robustness and verification in machine learning}

The robustness problem has been extensively studied in the machine learning literature. \FL{Studies in this area are} usually concerned with \FL{proving} the consistency of neural network predictions when various types of small perturbations occur in the input \cite{Carlini_17,Weng_18} %, where the perturbed data points invalidating the original input point's prediction results are termed adversarial examples \cite{DBLP:journals/corr/SzegedyZSBEGF13}, 
or in the model parameters \cite{DBLP:conf/aaai/WengZ0CLD20,TsaiHYC21}. 
%, which are often used to augment the training dataset to enhance the robustness of neural networks \cite{DBLP:conf/iclr/KurakinGB17,DBLP:conf/iclr/MadryMSTV18}. Other methods to 
%To address the robustness challenge, adversarial examples are often used to augment the training dataset \cite{DBLP:conf/iclr/KurakinGB17,DBLP:conf/iclr/MadryMSTV18} to induce better empirical robustness. Going further, 
One popular approach to check the robustness of learning models such as neural networks relies on abstraction techniques. These methods derive over-\FL{approximations} of the neural networks' output ranges, which are then integrated into the training loop to ensure provable robustness guarantees against small perturbations \cite{DBLP:conf/icml/WongK18,DBLP:conf/icml/MirmanGV18,DBLP:conf/iccv/GowalDSBQUAMK19,DBLP:conf/iclr/ZhangCXGSLBH20,DBLP:conf/aaai/HenriksenL23}. %,DBLP:conf/iclr/BalunovicV20
%These works are robustness-specific instantiations of the formal verification for neural networks. 
Notably, interval bound propagation-based methods have been found effective, using interval arithmetic to propagate the perturbations in the input space to the output layer  \cite{DBLP:conf/icml/MirmanGV18,DBLP:conf/iccv/GowalDSBQUAMK19}. Tighter bounds are further obtained by symbolic interval propagation \cite{DBLP:conf/iclr/ZhangCXGSLBH20,DBLP:conf/aaai/HenriksenL23}. Finally, the most relevant work to our study is the Interval Neural Networks (INNs) proposed by \citet{PrabhakarA19}. INNs also use interval arithmetic to estimate the model output ranges, but with a focus on aggregating adjacent weights to simplifying the neural network architecture for easier verification. By solving MILP programs, over-approximations of model output ranges can be obtained. Differently from previous work, we propose to use INNs to obtain a novel abstraction technique that represents the robustness property of CEs under a pre-defined set of plausible model changes.

\section{Background}
\label{sec:background}
\paragraph{\textbf{Notation}} We use $[k]$ to denote the set $\{1,\ldots,k\}$, for $k \in \mathbb{Z}^+$. Given a vector $x \in \mathbb{R}^n$ we use $x_i$ to denote the $i$-th component. Similarly, for a matrix $W \in \mathbb{R}^n \times \mathbb{R}^m$, we use $W_{i,j}$ to denote the $i,j$-th element and $\vectorisation{W}$ to refer to $W$'s vectorisation $\vectorisation{W} = [W_{1,1}, \ldots, W_{n,1}, W_{1,2}, \ldots, W_{n,2}, \ldots, W_{n,m}]$. Finally, $\powerset{\mathbb{R}}$ denotes the set of all closed intervals over $\mathbb{R}$.

\paragraph{\textbf{Classification models}} A classification model is a parametric model 
characterised by a set of equations over a parameter space $\Theta \subseteq \R^d$, for some %integer 
$d >0$. We use $\model{\Theta}$ to denote the family of classifiers spanning $\Theta$ and %use 
$\model{\theta}$ to refer to a specific concretisation obtained for some $\theta \in \Theta$. The latter is typically obtained by training on a set of labelled inputs%, drawn from an input space $\mathcal{X}$. The goal of training is to learn a $\theta$ that can be used for inference: 
. Then, for any unlabelled input $x \in \mathcal{X}$,  $\model{\theta}$ %should predict its label.
can be used to infer (predict) %the 
its label. 

\begin{example}
    Consider an input $x \in \mathcal{X}$ and assume $\model{\Theta}$ implements a \emph{logistic regression classifier}, characterised by the following equation: 
    
    $$\model{\Theta}(x) = \sigma \left( \sum_{i=1}^{d-1} \theta_i x_i + \theta_d \right)$$
    
    where $\sigma$ is a logistic function defined as usual. Then, ${\theta} = [\theta_1, \ldots, \theta_d]$.
    \label{ex:log_reg}
\end{example}

\begin{example}
    Consider an input $x \in \mathcal{X}$ and assume $\model{\Theta}$ implements a \emph{fully connected neural network} with $k$ hidden layers, characterised by the following equations:
    \begin{itemize}
    \item $V^{(0)} = x${;}
    \item $V^{(i)} \!=\! \phi(W^{(i)} \cdot V^{(i-1)} + B^{(i)})$ for $i \in [k]$, where $W^{(i)}$ and $B^{(i)}$ are %the 
    weights and biases, respectively, associated with layer $i$, and $\phi$ is an activation function applied element-wise; %\JJ{$V^{(i)} \!=\! \phi(W^{(i)} \cdot V^{(i-1)} + B^{(i)})$}
    \item $\model{\Theta}(x) = \sigma(V^{(k+1)}) = \sigma(W^{(k+1)} \cdot V^{(k)} + B^{(k+1)})$, where $\sigma$ is a logistic function defined as usual. %\JJ{$\model{\Theta}(x) = \sigma(V^{(k+1)}) = \sigma(W^{(k+1)} \cdot V^{(k)} + B^{(k+1)})$}
    \end{itemize}
    Then, ${\theta} = [\vectorisation{W^{(1)}} \smallquad \vectorisation{B^{(1)}} \ldots \vectorisation{W^{(k+1)}} \smallquad \vectorisation{B^{(k+1)}}]$.  \\ %\JJ{${\theta} = [\vectorisation{W^{(1)}} \smallquad \vectorisation{B^{(1)}} \ldots \vectorisation{W^{(k+1)}} \smallquad \vectorisation{B^{(k+1)}}]$}
    \label{ex:nn}
\end{example}

We now define the classification outcome of %a concrete 
$\model{\theta}$; while we focus on binary classification tasks for legibility, i.e. with %class
label $c \in \{0,1\}$, the ensuing definitions can be generalised to any classification setting.

\begin{definition} Given input $x  \in \mathcal{X}$ and model $\model{{\theta}}$, we say that $\model{{\theta}}$ \emph{classifies $x$ as $1$} if $\model{{\theta}}(x) \geq 0.5$, and otherwise $x$ is classified as $0$.
\label{def:classification}
\end{definition}

In the following, we abuse notation and use $\model{{\theta}}(x) = 1$ (respectively $\model{{\theta}}(x) = 0$) to denote when $x$ is classified as $1$ (respectively $0$). \JJ{Note that classes $0$ and $1$ are often assumed {to be} the undesirable and the desirable classes for the user \cite{upadhyay2021towards,mohammadi2021scaling,pmlr-v162-dutta22a}.}

\paragraph{\textbf{Counterfactual explanations}} Consider a classification model $\model{{\theta}}$ trained to solve a binary classification problem. Assume an input $x$ is given for which %\new{the model produces a negative outcome\delete{, i.e.}}, i.e. 
$\model{{\theta}}(x) =  0$. Intuitively, a counterfactual explanation is a new input $x'$ which is somehow similar to $x$, e.g. in terms of some specified distance between features values, and for which $\model{{\theta}}(x') = 1$. Formally, existing methods in the literature may be understood to compute counterfactuals as follows. 

\begin{definition}\label{def:cfx} 
Consider an input $x \in \mathcal{X}$ and a (binary) classification model $\model{{\theta}}$ s.t. $\model{{\theta}}(x) = 0$. Given a distance metric $d: \mathcal{X} \times \mathcal{X} \rightarrow \mathbb{R}^+$, a \emph{counterfactual explanation} is any $x'$ such that:
\begin{subequations}
\begin{alignat}{2}
&\argmin_{x'\in \mathcal{X}}  && d(x,x')\label{eqn:obj} \\
&\text{subject to} && \quad \model{{\theta}}(x') = 1\label{eqn:eq-1}
\end{alignat}
\label{eqn:all-lines}
\end{subequations}
\end{definition}

A counterfactual explanation thus corresponds to the closest input $x'$ (Eq.~\ref{eqn:obj}) belonging to the original input space that makes the classification flip (Eq.~\ref{eqn:eq-1}). 
Eq.~\ref{eqn:obj} and ~\ref{eqn:eq-1} are typically referred to as \emph{proximity} and \emph{validity} properties, respectively, of counterfactual explanations. A common choice for the distance metric $d$ is the normalised $L_1$ distance~\cite{Wachter_17} \JJ{for sparse changes in the CEs}, which we also adopt here. Under this choice of $d$, we note that when $\model{{\theta}}$ is a piece-wise linear model, an exact solution {to} Eqs.~\ref{eqn:all-lines}a-b can be computed with MILP -- see, e.g.,~\cite{mohammadi2021scaling}. %Finally, note that the optimisation problem can also be extended to account for additional properties \JJ{described}. 

\begin{example}
    \label{ex:cfx}
    (Continuing from Example \ref{ex:log_reg}.) 
    Assume a logistic regression classification model $\model{{\theta}}(x) = \sigma(-x_1 + x_2)$ for any input $x = [x_1,x_2]$, where $\sigma$ is the standard sigmoid function. For a concrete input $x = [0.7, 0.5]$, we have
    $\model{{\theta}}(x) = 0$. A possible counterfactual explanation for $x$ may be $x' = [0.7,0.7]$, for which $\model{{\theta}}(x') = 1$.
\end{example}

\begin{example}
    \label{ex:ffnn}
    (Continuing from Example \ref{ex:nn}.) 
     Consider the fully-connected feed-forward neural network $\model{{\theta}}$ below, where weights are as indicated in the diagram, biases are zero. Hidden layers use ReLU activations, whereas the output node uses a sigmoid function.
%  
%  
        % \[
        % W_1 = \begin{bmatrix}
        %     1 &  -1      \\
        %     -1  &  1      
        % \end{bmatrix}
        % \hspace{0.5cm}
        % W_2 = \begin{bmatrix}
        %     1  &  0      \\
        %     0  &  1      
        % \end{bmatrix} 
        % \hspace{0.5cm}
        % B_1 = B_2 = \begin{bmatrix}
        %     0  &  0      \\
        %     0  &  0      
        % \end{bmatrix}
        % \]
        % 
	The network receives a two-dimensional input $x = [x_1,x_2]$ and produces output $\model{{\theta}}(x)$. 
	
	\begin{figure}[ht]
		\centering
		\scalebox{1}{% \begin{tikzpicture}

%   \node[] (phantom_1) at (-1,0) {$x_0$};
%   \node[] (phantom_2) at (-1,-1.5) {$x_1$};
  
%   \node[circle,draw=black, minimum width=0.5cm] (input_1) at (0,0) {};
%   \node[circle,draw=black, minimum width=0.5cm] (input_2) at (0,-1.5) {};
  
%   \node[circle,draw=black, minimum width=0.5cm] (hidden_1) at (2,0) 
%   {\scriptsize R};
%   \node[circle,draw=black, minimum width=0.5cm] (hidden_2) at (2,-1.5) 
%   {\scriptsize R};
  
%     \node[circle,draw=black, minimum width=0.5cm] (output_1) at (4,0) 
%   {};
%   \node[circle,draw=black, minimum width=0.5cm] (output_2) at (4,-1.5) 
%   {};
  
%   \node[] (phantom_3) at (5,0) {$y_0$};
%   \node[] (phantom_4) at (5,-1.5) {$y_1$};
  
%   \draw[->] (phantom_1) -- (input_1);
%   \draw[->] (phantom_2) -- (input_2);
  
%   \draw[->] (input_1) edge node[above]{\tiny{$1$}} (hidden_1);
%   \draw[->] (input_1) edge node[below, xshift=-0.7cm, yshift=-0.1cm]{\tiny$-1$} 
%   (hidden_2);
  
%   \draw[->] (input_2) edge node[above, xshift=-0.7cm, yshift=0.1cm]{\tiny$-1$} 
%   (hidden_1);
%   \draw[->] (input_2) edge node[below]{\tiny$1$} 
%   (hidden_2);
  
%   \draw[->] (hidden_1) edge node[above]{\tiny{$1$}} (output_1);
%   \draw[->] (hidden_1) edge node[below, xshift=-0.7cm, yshift=-0.1cm]{\tiny{$0$}} (output_2);

%   \draw[->] (hidden_2) edge node[above, xshift=-0.7cm, yshift=0.1cm]{\tiny$0$}  (output_1);
%   \draw[->] (hidden_2) edge node[below]{\tiny$1$}  (output_2);

%   \draw[->] (output_1) -- (phantom_3);

%   \draw[->] (output_2) -- (phantom_4);
  
% \end{tikzpicture}

\begin{tikzpicture}[scale=0.8, every node/.style={scale=0.7}]

  \node[] (phantom_1) at (-1.5,0) {\Large$x_1$};
  \node[] (phantom_2) at (-1.5,-2) {\Large$x_2$};
  
  \node[circle,draw=black, minimum width=1cm] (input_1) at (0,0) {};
  \node[circle,draw=black, minimum width=1cm] (input_2) at (0,-2) {};
  
  \node[circle,draw=black, minimum width=1cm] (hidden_1) at (3,0) {};
  \begin{scope}[xshift=3cm,scale=0.7]
        % flexible selection of activation function
        \relua
        % \stepfunc
    \end{scope}
  
  \node[circle,draw=black, minimum width=1cm] (hidden_2) at (3,-2) {};
  \begin{scope}[xshift=3cm, yshift=-2cm,scale=0.7]
        % flexible selection of activati7n function
        \relua
        % \stepfunc
    \end{scope}
  
    \node[circle,draw=black, minimum width=1cm] (output_1) at (5,-1) {};
    \begin{scope}[xshift=5cm,scale=0.65,yshift=-1.7cm]
        % flexible selection of activation function
        \sigmoida
        % \stepfunc
    \end{scope}

  \node[] (phantom_3) at (6.5,-1) {};%\Large$\model{\theta}$}

  \draw[->] (phantom_1) -- (input_1);
  \draw[->] (phantom_2) -- (input_2);
  \draw[->] (output_1) -- (phantom_3);
  
  \draw[->] (input_1) edge node[above]{{$1$}} (hidden_1);
  \draw[->] (input_1) edge node[below, xshift=-0.4cm, yshift=-0.5cm]{$0$} 
  (hidden_2);
  
  \draw[->] (input_2) edge node[above, xshift=-0.4cm, yshift=0.5cm]{$0$} 
  (hidden_1);
  \draw[->] (input_2) edge node[below]{$1$} 
  (hidden_2);
  
  \draw[->] (hidden_1) edge node[above]{{$1$}} (output_1);

  \draw[->] (hidden_2) edge node[below, xshift=-0.15cm, yshift=-0.1cm]{$-1$}  (output_1); 
  
\end{tikzpicture}}
		\label{fig:original_net}
	\end{figure}

 The symbolic expressions for the output is $\model{\theta}(x) = \sigma(\max(0,x_1) - \max(0, x_2))$, where $\sigma$ denotes a sigmoid function with the usual meaning.  
	Given a concrete input $x=[1,2]$, we have $\model{{\theta}}(x) = 0$. A possible counterfactual explanation may be $x'= [2.1,2]$, for which $\model{{\theta}}(x') = 1$.
\end{example}

\section{Interval abstractions and robustness for binary classification}
\label{sec:inns}

In this section, we introduce a novel interval abstraction technique inspired by INNs to reason about the provable robustness guarantees of CEs, under the possibly infinite family of classification models obtained by applying a pre-defined set of plausible model changes, $\name$. We formalise the novel notion of $\Delta$-robustness, which, once satisfied, will ensure that the validity of CEs is not compromised by any model parameter change encoded in $\Delta$.

\subsection{Plausible model changes $\Delta$}

%Interval Neural Networks (INNs)~\cite{PrabhakarA19} are used to represent the possibly infinite family of networks that can be obtained under a pre-defined set of plausible model changes, $\name$, when applied to the network parameters. The resulting abstraction is then used to certify a property termed \emph{$\name$-robustness}, which ensures that the validity of counterfactual explanations is not compromised by any change encoded in $\name$. As the original formulation of~\cite{oursaaai23} only considered neural networks, in the following we present a more general framework to account for a wider range of classification models. To this end, w
First, in this section, we formalise the type of model changes central to our method. We begin by defining a notion of distance between two concretisations of a parametric classifier.

\begin{definition}
Let $\model{\theta}$ and $\model{\theta'}$ be two concretisations of a parametric classification model $\model{\Theta}$. 
For $0 \leq p  \leq \infty$, the \emph{p-distance between $\model{\theta}$ and $\model{\theta'}$} is defined as:

$$\dist{p}{\model{\theta}}{\model{\theta'}} = \distance{\theta}{\theta'}{p}$$.
\label{def:distance_between_models}
\end{definition}

\begin{example}
    Consider two models $\model{\theta} = \sigma(-x_1 + x_2)$ and $\model{\theta'} = \sigma(0.8 \cdot x_1 + x_2)$. Assume $p=\infty$. Then, their p-distance is $\dist{p}{\model{\theta}}{\model{\theta'}} = \max_i \lvert \theta_i - \theta'_i\rvert = 1.8$.
    \label{ex:comp}
\end{example}

Intuitively $p$-distance compares $\model{\theta}$ and $\model{\theta'}$ in terms of their parameters and computes the distance between them as the $p$-norm of the difference of their parameter vectors. Using this notion, we next characterise a model shift as follows. 

\begin{definition}
Given $0 \leq p \leq \infty$, a \emph{model shift} is a function $\mshift$ mapping a classification model $\model{\theta}$ into another $\model{\theta'} = \mshift(\model{\theta})$ such that:   
\begin{itemize}
    \item $\model{\theta}$ and $\model{\theta'}$ are concretisations of the same parameterised family $\model{\Theta}$;
    \item $\dist{p}{\model{\theta}}{\model{\theta'}} > 0$.
\end{itemize}
\label{def:model_shift}
\end{definition}

Model shifts are typically observed in real-world applications when a model is regularly retrained to incorporate new data. 
%\td{AR: wasn't there a paper which assessed this, should we mention it? JJ: don't think there are any}. 
In such cases, models are likely to see only small changes at each update. In the same spirit as~\cite{upadhyay2021towards}, we capture this %with the following definition
as follows.

\begin{definition}
Given a classification model $\model{\theta}$, a threshold $\delta \in \mathbb{R}_{>0}$ and $0 \leq p \leq \infty$, the \emph{set of plausible model shifts} is defined as:

$$\Delta = \{ \mshift \mid \dist{p}{\model{\theta}}{\mshift(\model{\theta})} \leq \delta\}.$$ 
\label{def:set_of_plausible_shifts}
\end{definition}

%\JJ{Plausibility directly bounds the magnitude of weight and bias changes that can be affected by a model shift $\mshift$: }

\JJ{When a set of plausible model shifts $\Delta$ is applied on $\model{\theta}$, the resulting maximum change on each model parameter in the original model is conservatively upper-bounded:}

\begin{lemma}
\JJ{Consider a classification model $\model{\theta}$ and a set of plausible model shifts $\Delta$ with threshold $\delta$ and $0 \leq p \leq \infty$. Then, $\forall \model{\theta'} = \mshift(\model{\theta})$, $\mshift \in \Delta$, and $\forall \theta_i$, $\theta'_i$, $i\in [d]$, we have $\theta'_i \in [\theta_i - \delta, \theta_i + \delta]$.}
%Let $\model{\theta'} = \mshift(\model{\theta})$ for any $\mshift \in \Delta$. The $p$-distance between $\model{\theta'}$ and $\model{\theta}$ is bounded above.
\label{lemma:bounds}
\end{lemma}

\begin{proof}
    
% 
% \begin{lemma}
% Consider a classification model $\model{\theta}$ and a set of plausible model shifts $\Delta$. Let $\model{\theta'} = \mshift(\model{\theta})$ for any $\mshift \in \Delta$. The $p$-distance between $\model{\theta'}$ and $\model{\theta}$ is bounded above.
% \label{lemma:bounds}
% \end{lemma}
% 
% \begin{proof}

combining Definition~\ref{def:set_of_plausible_shifts} with Definition~\ref{def:distance_between_models}, we obtain:

\begin{equation*}
    \left(\sum_{i=1}^{d}\lvert \theta_i - \theta'_i\rvert^p\right)^{\frac{1}{p}} \leq \delta
\end{equation*}

We raise both sides to the power of $p$:

\begin{equation*}
    \sum_{i=1}^{d}\left(\lvert \theta_i - \theta'_i\rvert^p\right) \leq \delta^p
\end{equation*}

where the inequality is preserved as both sides are always positive. We now observe this inequation bounds each addend from above, i.e.,

\begin{equation*}
 \lvert \theta_i - \theta'_i\rvert^p \leq \delta^p 
\end{equation*}

% for $i \in [k+1]$, $j \in \card{N_i}$, $l \in \card{N_{i-1}}$

Solving the inequation for each addend we obtain $\theta'_i \in [\theta_i - \delta, \theta_i + \delta]$, which gives a conservative upper-bound on the maximum change that can be applied to each parameter in $\model{\theta'}$. 
\end{proof}

\subsection{Interval abstractions}

%Despite weight changes being bounded, several different model shifts may satisfy the plausibility constraint. 

To guarantee robustness to the plausible model changes $\Delta$, \JJ{methods are needed to compactly} represent and reason about the \JJ{behaviour of a CE under the} potentially infinite family of models originated by applying each $\mshift \in \Delta$ to $\model{\theta}$. In the following, we introduce an abstraction framework that can be used to this end.

\begin{definition} Consider a classification model $\model{\theta}$ with $\theta = [\theta_1, \ldots, \theta_d]$. Given a set of plausible model shifts $\Delta$, we define the \emph{interval abstraction of $\model{\theta}$ under $\Delta$} as $\abst{\theta}{\Delta}: \mathcal{X} \rightarrow \powerset{\mathbb{R}}$ such that:
\begin{itemize}
\item $\model{\theta}$ and $\abst{\theta}{\Delta}$ \revision{have the same model architecture}; %are concretisations of the same parameterised family $\model{\Theta}$;
\item $\abst{\theta}{\Delta}$ is parameterised by an interval-valued vector $\bm{\theta} = [\bm{\theta}_1, \ldots, \bm{\theta}_d]$;
\item $\bm{\theta}_i$, for $i \in [d]$,  encodes the range of possible changes induced by the application of any $\mshift{} \in \Delta$ to $\model{\theta}$ such that $\bm{\theta}_i = [\theta_i - \delta, \theta_i + \delta]$, where $\delta$ is the maximum shift obtainable as per Definition~\ref{def:set_of_plausible_shifts}.
\end{itemize}
\label{def:interval_abstraction}
\end{definition}

\begin{lemma}
$\abst{\theta}{\Delta}$ over-approximates the set of models $\model{\theta'}$ that can be obtained from $\model{\theta}$ via $\Delta$. 
\label{lemma:over-approx}
\end{lemma}

\begin{proof}
Lemma~\ref{lemma:over-approx} states that $\abst{\theta}{\Delta}$ captures all models $\model{\theta'}$ that can be obtained from $\model{\theta}$ applying a $\mshift\in\Delta$, and possibly more \JJ{models} that violate the plausibility constraint \JJ{(Definition~\ref{def:set_of_plausible_shifts}, the p-distances are upper-bounded by $\delta$).} The former can be seen by observing that each parameter $\theta_i$ in $\abst{\theta}{\Delta}$ is initialised in such a way to contain all possible parameterisations obtainable starting from the initial value $\theta_i$ and perturbing it up to $\pm \delta$ \JJ{(Lemma~\ref{lemma:bounds})}. As a result, $\abst{\theta}{\Delta}$ captures all shifted models by construction. \JJ{$\abst{\theta}{\Delta}$ also captures more concrete models for which the $p$-distance is greater than $\delta$, as the bounds used in the interval abstraction, $\bm{\theta}_i = [\theta_i - \delta, \theta_i + \delta]$, are conservative. }%\JJ{Specifically, $\abst{\theta}{\Delta}$ exactly represents all shifted models by applying $\Delta$ to $\model{\theta}$ when $p=\infty$.}  
\end{proof}

The following two examples demonstrate the interval abstraction in the context of different classification models.

\begin{example} %\delete{Consider the model $\model{\theta}$ introduced in Example~\ref{ex:cfx} and} 
Continuing from Example \ref{ex:cfx}. Consider the same model $\model{\theta}$ and assume a set of plausible model shifts $\Delta = \{ \mshift \mid \dist{\infty}{\model{\theta}}{\mshift(\model{\theta})} \leq 0.1\}$. The interval abstraction $\abst{\theta}{\Delta}$ is defined by the model equation $\abst{\theta}{\Delta}(x) = \sigma \left( {\bm{\theta}}_1 x_1 + {\bm{\theta}}_2 x_2 \right)$,
% $\abst{\Theta}{\Delta}(x) = \sigma \left( \sum_{i=1}^d {\bm{\theta}}_i x_i + {\bm{\theta}}_0 \right)$,
where $\bm{\theta}_1 = [-1 - 0.1,  -1 + 0.1]$ and $\bm{\theta}_2 = [1 - 0.1,  1 + 0.1]$. Then, $\bm{\theta} = [{\bm{\theta}}_1, {\bm{\theta}}_2]$.
\label{ex:interval_abst}
\end{example}

	\begin{figure}[h]
		\centering
		\scalebox{1}{\begin{tikzpicture}[scale=0.8, every node/.style={scale=0.7}]

  \node[] (phantom_1) at (-1.5,0) {\Large$x_1$};
  \node[] (phantom_2) at (-1.5,-2) {\Large$x_2$};
  
  \node[circle,draw=black, minimum width=1cm] (input_1) at (0,0) {};
  \node[circle,draw=black, minimum width=1cm] (input_2) at (0,-2) {};
  
  \node[circle,draw=black, minimum width=1cm] (hidden_1) at (3,0) {};
  \begin{scope}[xshift=3cm,scale=0.7]
        % flexible selection of activation function
        \relua
        % \stepfunc
    \end{scope}
  
  \node[circle,draw=black, minimum width=1cm] (hidden_2) at (3,-2) {};
  \begin{scope}[xshift=3cm, yshift=-2cm,scale=0.7]
        % flexible selection of activati7n function
        \relua
        % \stepfunc
    \end{scope}
  
    \node[circle,draw=black, minimum width=1cm] (output_1) at (5,-1) {};
    \begin{scope}[xshift=5cm,scale=0.65,yshift=-1.7cm]
        % flexible selection of activation function
        \sigmoida
        % \stepfunc
    \end{scope}

  \node[] (phantom_3) at (6.5,-1) {};%\Large$\model{\theta}$}

  \draw[->] (phantom_1) -- (input_1);
  \draw[->] (phantom_2) -- (input_2);
  \draw[->] (output_1) -- (phantom_3);
  
  \draw[->] (input_1) edge node[above]{{$[0.95,1.05]$}} (hidden_1);
  
  \draw[->] (input_1) edge node[below, xshift=-2cm, yshift=0.1cm]{$[-0.05,0.05]$} 
  (hidden_2);
  
  \draw[->] (input_2) edge node[above, xshift=-2cm, yshift=0.cm]{$[-0.05,0.05]$} 
  (hidden_1);
  \draw[->] (input_2) edge node[below,yshift=-0.1cm]{$[0.95,1.05]$} 
  (hidden_2);
  
  \draw[->] (hidden_1) edge node[above, xshift=0.8cm, yshift=0.1cm]{{$[0.95,1.05]$}} (output_1);

  \draw[->] (hidden_2) edge node[below, xshift=1cm, yshift=-0.1cm]{$[-1.05,-.095]$}  (output_1); 
  
\end{tikzpicture}}
        \caption{Illustration of Example~\ref{ex:interval_ffnn}.}
        		\label{fig:interval_net}
	\end{figure}
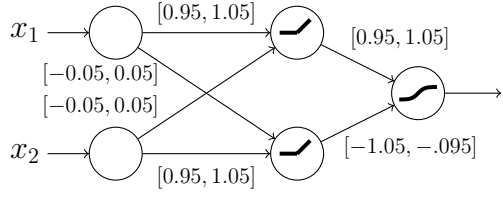

\begin{example}
    \label{ex:interval_ffnn}
     Continuing from Example \ref{ex:ffnn}. Consider the same model $\model{\theta}$ and assume a set of plausible model shifts $\Delta = \{ \mshift \mid \dist{\infty}{\model{\theta}}{\mshift(\model{\theta})} \leq 0.05\}$. The resulting interval abstraction $\abst{\theta}{\Delta}$ is \FL{represented in Figure~\ref{fig:interval_net},} where the symbolic expression of the output interval is $\sigma([0.95,1.05] \cdot \max(0, [0.95,1.05] \cdot x_1 + [-0.05,0.05] \cdot x_2) + [-1.05, -0.95] \cdot \max(0,[-0.05,0.05] \cdot x_1 +[0.95,1.05] \cdot x_2 ))$.

\end{example}

\revision{Given an input, our interval abstraction maps it to an output interval, constructed by taking the lowest and the highest values from all the possible classification outcomes by any shifted model $\model{\theta'} \in \Delta$. This output interval therefore over-approximates the outputs obtainable from any shifted models in $\Delta$.} 
%\delete{Our interval abstractions map input points to output intervals representing all possible classification outcomes that can be produced by any shifted model $\model{\theta'}$.} 
The classification semantics of the \revision{interval abstraction} thus departs from Definition~\ref{def:classification} and needs to be generalised to account for this new behaviour.

\begin{definition}
    Let $\abst{\theta}{\Delta}$ be the interval abstraction of a classification model $\model{\theta}$. Given an input $x \in \mathcal{X}$, let $[l,u]$ be the output interval obtained by applying $\abst{\theta}{\Delta}$ to $x$. We say that \emph{$\abst{\theta}{\Delta}$ classifies $x$ as $1$}, if $l \geq 0.5$\JJ{, $0$ if $u < 0.5$, and \emph{undefined} otherwise}.
    \label{def:classification_interval}
\end{definition}

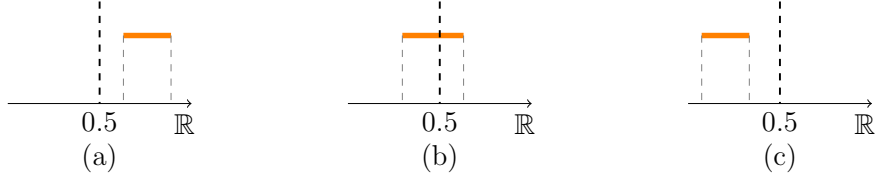
\begin{figure}
    \centering
    \scalebox{0.9}{\begin{tikzpicture}

%   \node[] (phantom_1) at (-1,0) {};
%   \node[] (phantom_2) at (-1,-1.5) {};
 
%       \node[circle,draw=black, minimum width=0.5cm] (output_1) at (0,0) 
%   {};
%   \node[circle,draw=black, minimum width=0.5cm] (output_2) at (0,-1.5) 
%   {};
  
%   \node[] (phantom_3) at (1,0) {\color{orange}$y_c$};
%   \node[] (phantom_4) at (1.2,-1.5) {\color{blue}$y_{1-c}$};
 
%   \draw[->] (phantom_1) edge (output_1);
%   \draw[->] (phantom_1) edge (output_2);

%   \draw[->] (phantom_2) edge  (output_1);
%   \draw[->] (phantom_2) edge  (output_2);
  
%   \draw[->] (output_1) -- (phantom_3);

%   \draw[->] (output_2) -- (phantom_4);
  
%   ------------------------------------------

%  class c
\draw[line width=0.8mm, color=orange] (4,0) -- (4.7,0);
\draw[dashed,gray] (4,0) -- (4,-1);
\draw[dashed,gray] (4.7,0) -- (4.7,-1);

% class 1-c
% \draw[line width=0.8mm,color=blue] (2.5,-0.5) edge node[above]{$1-c$} (3.5,-0.5);
% \draw[dashed] (2.5,-0.5) -- (2.5,-1);
\draw[dashed,black, thick] (3.65,0.5) edge node[above,yshift=-1.3cm]{\color{black}$0.5$} (3.65,-1);

% real line
\draw[->] (2.3,-1) -- (5,-1);
\node[] (phantom_r1) at (4.9,-1.3) {$\R$};

% class outcome
\node[] (phantom_1) at (3.65,-1.8) {(a)};

%   ------------------------------------------

% class c
\draw[line width=0.8mm, color=orange] (8.1,0) -- (9.0,0);
\draw[dashed,gray] (8.1,0) -- (8.1,-1);
\draw[dashed,gray] (9.,0) -- (9.,-1);

% class 1-c
% \draw[line width=0.8mm,color=blue] (5.5,-0.5) edge node[above, xshift=0.0cm]{$1-c$} (7.,-0.5);
% \draw[dashed] (5.5,-0.5) -- (5.5,-1);
\draw[dashed,black, thick] (8.65,0.5) edge node[above,yshift=-1.3cm]{\color{black}$0.5$} (8.65,-1);

% real line
\draw[->] (7.3,-1) -- (10,-1);
\node[] (phantom_r2) at (9.9,-1.3) {$\R$};

% class outcome
\node[] (phantom_1) at (8.65,-1.8) {(b)};

 %   ------------------------------------------

%  class 1-c
% \draw[line width=0.8mm, color=blue] (10,-0.5) edge node[above]{$1- c$} (10.7,-0.5);
\draw[dashed,black, thick] (13.65,0.5) edge node[above,yshift=-1.3cm]{\color{black}$0.5$} (13.65,-1);
% \draw[dashed] (10.7,-0.5) -- (10.7,-1);

% class c
\draw[line width=0.8mm,color=orange] (12.5,0) -- (13.2,0);
\draw[dashed,gray] (12.5,0) -- (12.5,-1);
\draw[dashed,gray] (13.2,0) -- (13.2,-1);

% real line
\draw[->] (12.3,-1) -- (15,-1);
\node[] (phantom_r3) at (14.9,-1.3) {$\R$};

% class outcome
\node[] (phantom_1) at (13.65,-1.8) {(c)};

\end{tikzpicture}}
    \caption{Visual representation of Definition~\ref{def:classification_interval}. In (a), $\abst{\theta}{\Delta}$ classifies an input as $1$ because the output range for that input is always greater than $0.5$. In (b), the output range includes value 0.5 therefore the classification result \JJ{is undefined}. In (c), {in a} similar {manner} to {the way in which the input in} (a) {is classified as $1$}, the input is classified as $0$.}
    
    \label{fig:inn_classfication}
\end{figure}

% \todo{JJ: I don't think we need a symbolic expression for ``undefined'' in the next sentence since we don't use it elsewhere {seems okay to me but I'd check with Francesco}} 
As in Section \ref{sec:background}, we will slightly abuse notation and use $\abst{\theta}{\Delta}(x) = 1$ (respectively, $\abst{\theta}{\Delta}(x) = 0$) to denote the case where the abstraction classifies an input as $1$ (respectively, $0$). A visual representation of this interval-based classification semantics is given in Figure~\ref{fig:inn_classfication}. In this work, we are interested in the conservative worst-case robustness of CEs. Therefore, we leave for future work (Section~\ref{sec:conclusion}) the explorations of the undefined case, as in Figure~\ref{fig:inn_classfication} (b).

\subsection{$\Delta$-robustness}

Using the interval abstraction as a tool to represent infinite families of shifted models, we now present $\Delta$-robustness, a notion of robustness to model changes which is central to this contribution. First, we define the conditions within which the robustness of a counterfactual can be assessed by introducing a notion of soundness as follows.

\begin{definition}
    Consider an input $x \in \mathcal{X}$ and a model $\model{\theta}$ such that $\model{\theta}{(x)} = 0$. Let $\abst{\theta}{\Delta}$ be the interval abstraction of $\model{\theta}$ for a set of plausible model shifts $\Delta$. We say that \emph{$\Delta$ is sound for $x$} iff $\abst{\theta}{\Delta}(x) = 0$.
             \label{def:soundness_interval}
    \end{definition}

In other words, soundness requires that shifts in $\Delta$ do not alter the class predicted for the original input $x$. This is a safety requirement that we introduce to ensure consistency in the predictions produced by the interval abstraction. 

We then can reason about the robustness properties defined as follows.

\begin{definition}
    Consider an input $x \in \mathcal{X}$ and a model $\model{\theta}$ such that $\model{\theta}{(x)} = 0$. Let $\abst{\theta}{\Delta}$ be the interval abstraction of $\model{\theta}$ for a set of plausible model shifts $\Delta$.  We say that a counterfactual explanation $x'$ is:  
 
        \begin{itemize}
        \item \emph{$\name$-robust} iff $\abst{\theta}{\Delta}(x') = 1$ and 
        
        \item \emph{strictly $\name$-robust} iff it is $\name$-robust \& 
        $\Delta$ is sound for $x$.
        \end{itemize}
\label{def:delta_robustness}
\end{definition}

The soundness of $\Delta$ ensures the theoretical correctness of $\Delta$-robustness. We point out that, in practice, this requirement may be relaxed such that 
for any input $x$ labelled as class 0 by the original model, one can find a CE $x'$ that is classified as class 1 by all the model shifts in $\Delta$. In this case, strictly speaking, $x'$ cannot be said to be \emph{valid} for $\Delta$ because it is possible that for some $\mathcal{M}_i$ induced by $\Delta$, $\mathcal{M}_i(x)=\mathcal{M}_i(x')=1$. However, in many applications soundness with respect to $\Delta$ is not required (as are the setups in e.g. \citep{upadhyay2021towards,pmlr-v162-dutta22a}), in which case \FL{computing $\Delta$-robust CEs may be more practical than 
{computing those which are} strictly $\Delta$-robust}. %\revision{For example, consider the scenario where a customer has been rejected a loan by a bank, which is responsible for providing recourse for the customer. Since the rejection has already happened, the focus is then on ensuring the recourse stays valid after model updates take place.}\todo{Did the reviewers ask to add this? I feel the example is a bit weak and doesn't really clarify the difference between delta robust and strictly delta robust.}

We conclude with an example summarising the main concepts presented in this section.

\begin{example}    \label{ex:rob}
    (Continuing from Examples~\ref{ex:cfx} and~\ref{ex:interval_abst}). We observe that $\abst{\theta}{\Delta}(x)$ produces an output interval $[0.42,0.48]$, i.e. $\abst{\theta}{\Delta} (x) = 0$. Since $\model{\theta}(x) = 0$, we can conclude that the set of model shifts $\Delta$ is sound. 
    We then check whether the previously computed counterfactual explanation $x' = [0.7,0.7]$ is $\name$-robust. Running $x'$ through $\abst{\theta}{\Delta}$ we obtain an output interval $[0.45,0.53]$, i.e. $\abst{\theta}{\Delta}(x) = 0$. Therefore, the counterfactual is not $\name$-robust. Assume a new counterfactual is obtained as $x'' = [0.7, 0.86]$. The output interval produced by $\abst{\theta}{\Delta}$ is now $[0.5, 0.58]$, i.e. $\abst{\theta}{\Delta}(x'') = 1$. Thus, the counterfactual $x''$ is strictly $\name$-robust. 
\end{example}

\section{Interval abstractions and robustness for multi-class classification}
\label{sec:interval_abs_multi_class} 

Next, we introduce the notion of $\Delta$-robustness in multi-class classification settings. 
% One notable difference \FL{with respect to the binary case} is that in the classifier, the final layer contains multiple output nodes, followed by a softmax activation, as opposed to one output node with a sigmoid function\todo{Why are we using neural nets terminology here? I thought we wanted to be general (and in fact the presentation has been agnostic of the underlying model so far). JJ: I think the above (softmax after model outputs) is a standard for multiclass classification models. FT: I AGREE WE SHOULD TRY TO BE GENERAL. CAN WE SAY SOMETHING MORE GENERIC AND THEN SAY E.G. IF THE MODELS IS A NN...}.
Unlike the binary case, where it is usually assumed that class $0$ (unwanted class) is the original prediction result for the input and class $1$ is the desirable target class for the CE to achieve, in the multi-class settings, the classes need to be explicitly specified \cite{dandl2020multi,mothilal2020explaining}. Given a set of labels $\{1,\ldots,\ell\}$, we generalise Definitions~\ref{def:classification} and \ref{def:cfx} as follows.

\begin{definition} Given input $x \in \mathcal{X}$ and model $\model{{\theta}}$, we say that $\model{{\theta}}$ \emph{classifies $x$ as $c \in \{1, \ldots,\ell\}$} if $\model{{\theta}}(x)_c \geq \model{{\theta}}(x)_{c'}$, for all $c' \in \{1, \ldots,\ell\}$ such that $c' \neq c$.
\label{def:multi_classification}
\end{definition}

\begin{definition}\label{def:cfx_multiclass} 
Consider an input $x \in \mathcal{X}$ and a classification model $\model{{\theta}}$ s.t. $\model{{\theta}}(x) = c \in \{1, \ldots,\ell\}$. \JJ{Assume} a desirable target class $c' \in \{1, \ldots,\ell\}$ such that $c' \neq c$ and given a distance metric $d: \mathcal{X} \times \mathcal{X} \rightarrow \mathbb{R}^+$, a \emph{counterfactual explanation} for class $c'$ is any $x'$ such that:
\begin{subequations}
\begin{alignat*}{2}
&\argmin_{x'\in \mathcal{X}}  && d(x,x')\\%\label{eqn:obj} \\
&\text{subject to} && \quad \model{{\theta}}(x') = c'
%\label{eqn:eq-1}
\end{alignat*}
\label{eqn:all-lines-multiclass}
\end{subequations}
\end{definition}

Example~\ref{ex:ffnn_multi} illustrates \FL{Definition~\ref{def:cfx_multiclass} on a} toy example where the classifier is a neural network. 

\begin{example}
    \label{ex:ffnn_multi}
     %Consider a neural network for a three-class classification task $\model{\theta}$. Assume a set of plausible model shifts $\Delta = \{ \mshift \mid \dist{\infty}{\model{\theta}}{\mshift(\model{\theta})} \leq 0.05\}$. The resulting interval abstraction $\abst{\theta}{\Delta}$ is defined as:
     Consider the fully-connected feed-forward neural network $\model{{\theta}}$ below, where weights are as indicated in the diagram, biases are zero. Hidden layers use ReLU activations, whereas the output layer uses a softmax function. Output classes are $\{c_1, c_2, c_3\}$.

	\begin{figure}[h!]
		\centering
		\scalebox{1}{\begin{tikzpicture}[scale=0.8, every node/.style={scale=0.7}]

  \node[] (phantom_1) at (-1.5,0) {\Large$x_1$};
  \node[] (phantom_2) at (-1.5,-2) {\Large$x_2$};
  
  \node[circle,draw=black, minimum width=1cm] (input_1) at (0,0) {};
  \node[circle,draw=black, minimum width=1cm] (input_2) at (0,-2) {};
  
  \node[circle,draw=black, minimum width=1cm] (hidden_1) at (3,0) {};
  \begin{scope}[xshift=3cm,scale=0.7]
        % flexible selection of activation function
        \relua
        % \stepfunc
    \end{scope}
  
  \node[circle,draw=black, minimum width=1cm] (hidden_2) at (3,-2) {};
  \begin{scope}[xshift=3cm, yshift=-2cm,scale=0.7]
        % flexible selection of activati7n function
        \relua
        % \stepfunc
    \end{scope}
  
    \node[circle,draw=black, minimum width=1cm] (output_1) at (5,1) {};
    \begin{scope}[xshift=5cm,scale=0.65,yshift=1.4cm]
        % flexible selection of activation function
        \sigmoida
        % \stepfunc
    \end{scope}

    \node[circle,draw=black, minimum width=1cm] (output_2) at (5,-1) {};
    \begin{scope}[xshift=5cm,scale=0.65,yshift=-1.7cm]
        % flexible selection of activation function
        \sigmoida
        % \stepfunc
    \end{scope}

    \node[circle,draw=black, minimum width=1cm] (output_3) at (5,-3) {};
    \begin{scope}[xshift=5cm,scale=0.65,yshift=-4.8cm]
        % flexible selection of activation function
        \sigmoida
        % \stepfunc
    \end{scope}

    \node[] (phantom_3) at (6.5,1) {\Large$c_1$};
    \node[] (phantom_4) at (6.5,-1) {\Large$c_2$};
    \node[] (phantom_5) at (6.5,-3) {\Large$c_3$};

  \draw[->] (phantom_1) -- (input_1);
  \draw[->] (phantom_2) -- (input_2);
  \draw[->] (output_1) -- (phantom_3);
  \draw[->] (output_2) -- (phantom_4);
  \draw[->] (output_3) -- (phantom_5);
  
  \draw[->] (input_1) edge node[above]{{$1$}} (hidden_1);
  \draw[->] (input_1) edge node[below, xshift=-0.4cm, yshift=-0.5cm]{$0$} 
  (hidden_2);
  
  \draw[->] (input_2) edge node[above, xshift=-0.4cm, yshift=0.5cm]{$0$} 
  (hidden_1);
  \draw[->] (input_2) edge node[below]{$1$} 
  (hidden_2);
  
  \draw[->] (hidden_1) edge node[above]{{$1$}} (output_1);
  \draw[->] (hidden_1) edge node[above, xshift=0.7cm, yshift=-0.2cm]{{$0$}} (output_2);
  \draw[->] (hidden_1) edge node[below, xshift=1cm, yshift=-0.6cm]{{$-1$}} (output_3);

    \draw[->] (hidden_2) edge node[below, xshift=1cm, yshift=1.2cm]{$-1$}  (output_1); 
    \draw[->] (hidden_2) edge node[above, xshift=0.8cm, yshift=-0.5cm]{$0.5$}  (output_2); 
    \draw[->] (hidden_2) edge node[below, xshift=-0.15cm, yshift=-0.1cm]{$1$}  (output_3); 
  
\end{tikzpicture}}
		\label{fig:interval_net_multiclass}
	\end{figure}
     The symbolic expressions for the output is $\model{\theta}(x) = Softmax([\max(0,x_1) - \max(0, x_2), 0.5\max(0, x_2), \max(0,x_2) - \max(0, x_1)])$.  
	Given a concrete input $x=[2,2]$, we have $\model{{\theta}}(x) = c_2$. A possible counterfactual explanation for class $c_1$ may be $x'= [3,1]$, for which $\model{{\theta}}(x') = c_1$.
 
    %where the symbolic expression of the output is $\model{\theta}(x) = \sigma([0.95,1.05] \cdot \max(0, [0.95,1.05] \cdot x_0 + [-0.05,0.05] \cdot x_1) + [-1.05, -0.95] \cdot \max(0,[-0.05,0.05] \cdot x_0 +[0.95,1.05] \cdot x_1 ))$.
\end{example}

Then, the classification semantics of an interval abstraction for multiclass problems can be obtained by extending Definition~\ref{def:classification_interval} as follows.

\begin{definition}
    Let $\abst{\theta}{\Delta}$ be the interval abstraction of a classification model $\model{\theta}$. Given an input $x \in \mathcal{X}$, let $[l_{i},u_{i}]$ be the output interval obtained for each $i \in \{1,\ldots,\ell\}$ by applying $\abst{\theta}{\Delta}$ to $x$. We say that \emph{$\abst{\theta}{\Delta}$ classifies $x$ as class $c$}, \JJ{if %$\exists c \in \{1,\ldots,\ell\}$ such that 
    $l_c \geq u_{c'}$ for all $c' \in \{1, \ldots,\ell\} \setminus \{c\}$. Otherwise we say the classification result from $\abst{\theta}{\Delta}$ is \emph{undefined} if $\forall c' \in \{1, \ldots,\ell\}$, $\exists c'' \in \{1, \ldots,\ell\}\setminus \{c'\}$ such that $u_{c''} \geq l_{c'}$.}
    \label{def:multi_classification_interval}
\end{definition}

 %\todo{AR: we could just give this since space is not an issue? up to you if you think it's too trivial though}

%\todo{AR: what classification do we give for case b in Figure 2?}

Figure~\ref{fig:multi_inn_classfication} intuitively illustrates the classification semantics of interval abstractions in multi-class settings. Next, we formalise $\Delta$-robustness for CEs. %Note that the number of output intervals matches the number of output nodes in the classifier, thus the number of output classes.

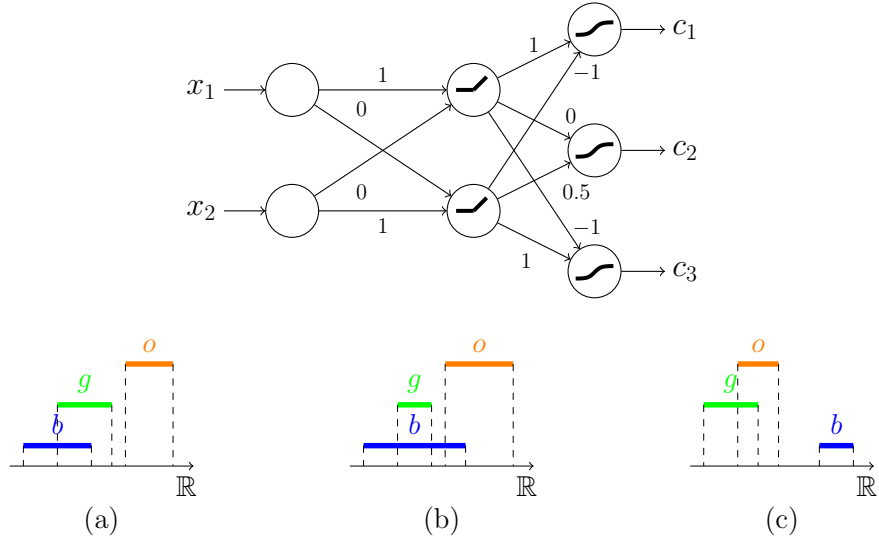
\begin{figure}
    \centering
    \scalebox{0.9}{\begin{tikzpicture}

%   \node[] (phantom_1) at (-1,0) {};
%   \node[] (phantom_2) at (-1,-1.5) {};
 
%       \node[circle,draw=black, minimum width=0.5cm] (output_1) at (0,0) 
%   {};
%   \node[circle,draw=black, minimum width=0.5cm] (output_2) at (0,-1.5) 
%   {};
  
%   \node[] (phantom_3) at (1,0) {\color{orange}$y_c$};
%   \node[] (phantom_4) at (1.2,-1.5) {\color{blue}$y_{1-c}$};
 
%   \draw[->] (phantom_1) edge (output_1);
%   \draw[->] (phantom_1) edge (output_2);

%   \draw[->] (phantom_2) edge  (output_1);
%   \draw[->] (phantom_2) edge  (output_2);
  
%   \draw[->] (output_1) -- (phantom_3);

%   \draw[->] (output_2) -- (phantom_4);
  
%   ------------------------------------------

%  class c
\draw[line width=0.8mm, color=orange] (4,0.5) edge node[above]{$o$} (4.7,0.5);
\draw[dashed] (4,0.5) -- (4,-1);
\draw[dashed] (4.7,0.5) -- (4.7,-1);

% class 1-c
\draw[line width=0.8mm,color=blue] (2.5,-0.7) edge node[above]{$b$} (3.5,-0.7);
\draw[dashed] (2.5,-0.7) -- (2.5,-1);
\draw[dashed] (3.5,-0.7) -- (3.5,-1);

% class 1-c
\draw[line width=0.8mm,color=green] (3,-0.1) edge node[above]{$g$} (3.80,-0.1);
\draw[dashed] (3,-0.1) -- (3,-1);
\draw[dashed] (3.8,-0.1) -- (3.8,-1);

% real line
\draw[->] (2.3,-1) -- (5,-1);
\node[] (phantom_r1) at (4.9,-1.3) {$\R$};

% class outcome
\node[] (phantom_1) at (3.65,-1.8) {(a)};

%   ------------------------------------------

% class c
\draw[line width=0.8mm, color=orange] (8.7,0.5) edge node[above]{$o$} (9.7,0.5);
\draw[dashed] (8.7,0.5) -- (8.7,-1);
\draw[dashed] (9.7,0.5) -- (9.7,-1);

% class 1-c
\draw[line width=0.8mm,color=blue] (7.5,-0.7) edge node[above, xshift=0.0cm]{$b$} (9,-0.7);
\draw[dashed] (7.5,-0.7) -- (7.5,-1);
\draw[dashed] (9,-0.7) -- (9,-1);

\draw[line width=0.8mm,color=green] (8,-0.1) edge node[above]{$g$} (8.5,-0.1);
\draw[dashed] (8,-0.1) -- (8,-1);
\draw[dashed] (8.5,-0.1) -- (8.5,-1);

% real line
\draw[->] (7.3,-1) -- (10,-1);
\node[] (phantom_r2) at (9.9,-1.3) {$\R$};

% class outcome
\node[] (phantom_1) at (8.65,-1.8) {(b)};

 %   ------------------------------------------

%  class c
\draw[line width=0.8mm, color=blue] (14.2,-0.7) edge node[above]{$b$} (14.7,-0.7);
\draw[dashed] (14.2,-0.7) -- (14.2,-1);
\draw[dashed] (14.7,-0.7) -- (14.7,-1);

% class 1-c
\draw[line width=0.8mm,color=orange] (13,0.5) edge node[above]{$o$} (13.6,0.5);
\draw[dashed] (13,0.5) -- (13,-1);
\draw[dashed] (13.6,0.5) -- (13.6,-1);

% class 1-c
\draw[line width=0.8mm,color=green] (12.5,-0.1) edge node[above]{$g$} (13.3,-0.1);
\draw[dashed] (12.5,-0.1) -- (12.5,-1);
\draw[dashed] (13.3,-0.1) -- (13.3,-1);

% real line
\draw[->] (12.3,-1) -- (15,-1);
\node[] (phantom_r3) at (14.9,-1.3) {$\R$};

% class outcome
\node[] (phantom_1) at (13.65,-1.8) {(c)};

\end{tikzpicture}}
    \caption{Visual representation of Definition~\ref{def:multi_classification_interval} instantiated with three colour-coded classes \{`\underline{o}range', `\underline{g}reen', `\underline{b}lue'\}. In (a), $\abst{\theta}{\Delta}$ classifies an input as $o$ since the output range for that class is always greater than those of any other classes. In (b), the output ranges for $o$ overlap with $b$, the classification result \JJ{is therefore undefined}. In (c), {in a} similar {manner} to {the way in which the input in} (a) {is classified as $o$}, $\abst{\theta}{\Delta}$ classifies an input as $b$.}
    \label{fig:multi_inn_classfication}
\end{figure}

\begin{definition}
    Consider an input $x \in \mathcal{X}$ and a model $\model{\theta}$ such that $\model{\theta}{(x)} = c \in \{1,\ldots,\ell\}$. Let $\abst{\theta}{\Delta}$ be the interval abstraction of $\model{\theta}$ for a set of plausible model shifts $\Delta$. We say that \emph{$\Delta$ is sound for $x$} iff $\abst{\theta}{\Delta}(x) = c$.
             \label{def:soundness_interval_multiclass}
    \end{definition}

\begin{definition}
Consider an input $x \in \mathcal{X}$ and a model $\model{\theta}$ such that $\model{\theta}{(x)} = c \in \{1,\ldots,\ell\}$. Let $\abst{\theta}{\Delta}$ be the interval abstraction of $\model{\theta}$ for a set of plausible model shifts $\Delta$. \JJ{Assume} a desirable target class $c' \in \{1, \ldots,\ell\}$ such that $c' \neq c$, we say that a counterfactual explanation $x'$ is:  
 
        \begin{itemize}
        \item \emph{$\name$-robust} for class $c'$ iff $\abst{\theta}{\Delta}(x') = c'$ and 
        
        \item \emph{strictly $\name$-robust} for class $c'$ iff it is $\name$-robust \& 
        $\Delta$ is sound for $x$.
        \end{itemize}
\label{def:delta_robustness_multiclass}
\end{definition}

The definition of $\name$-robustness transfers to the multi-class semantics \JJ{by estimating lower and upper bounds for each output class and then checking the robustness property}. In Example~\ref{ex:interval_ffnn_multi}, we instantiate an interval abstraction for a neural network and show how the output intervals can be calculated, thus how $\Delta$-robustness can be examined. We introduce a more general solution for testing $\Delta$-robustness in the next section.

\begin{example}
    \label{ex:interval_ffnn_multi}
     Continuing from Example \ref{ex:ffnn_multi}. Consider the same model $\model{\theta}$ and assume a set of plausible model shifts $\Delta = \{ \mshift \mid \dist{\infty}{\model{\theta}}{\mshift(\model{\theta})} \leq 0.05\}$. The resulting interval abstraction $\abst{\theta}{\Delta}$ is defined \FL{as in Figure~\ref{fig:net_multiclass}}.

	\begin{figure}[ht]
		\centering
		\scalebox{1}{\begin{tikzpicture}[scale=0.8, every node/.style={scale=0.7}]

  \node[] (phantom_1) at (-1.5,0) {\Large$x_1$};
  \node[] (phantom_2) at (-1.5,-2) {\Large$x_2$};
  
  \node[circle,draw=black, minimum width=1cm] (input_1) at (0,0) {};
  \node[circle,draw=black, minimum width=1cm] (input_2) at (0,-2) {};
  
  \node[circle,draw=black, minimum width=1cm] (hidden_1) at (3,0) {};
  \begin{scope}[xshift=3cm,scale=0.7]
        % flexible selection of activation function
        \relua
        % \stepfunc
    \end{scope}
  
  \node[circle,draw=black, minimum width=1cm] (hidden_2) at (3,-2) {};
  \begin{scope}[xshift=3cm, yshift=-2cm,scale=0.7]
        % flexible selection of activati7n function
        \relua
        % \stepfunc
    \end{scope}
  
    \node[circle,draw=black, minimum width=1cm] (output_1) at (5,1) {};
    \begin{scope}[xshift=5cm,scale=0.65,yshift=1.4cm]
        % flexible selection of activation function
        \sigmoida
        % \stepfunc
    \end{scope}

    \node[circle,draw=black, minimum width=1cm] (output_2) at (5,-1) {};
    \begin{scope}[xshift=5cm,scale=0.65,yshift=-1.7cm]
        % flexible selection of activation function
        \sigmoida
        % \stepfunc
    \end{scope}

    \node[circle,draw=black, minimum width=1cm] (output_3) at (5,-3) {};
    \begin{scope}[xshift=5cm,scale=0.65,yshift=-4.8cm]
        % flexible selection of activation function
        \sigmoida
        % \stepfunc
    \end{scope}

    \node[] (phantom_3) at (6.5,1) {\Large$c_1$};
    \node[] (phantom_4) at (6.5,-1) {\Large$c_2$};
    \node[] (phantom_5) at (6.5,-3) {\Large$c_3$};

  \draw[->] (phantom_1) -- (input_1);
  \draw[->] (phantom_2) -- (input_2);
  \draw[->] (output_1) -- (phantom_3);
  \draw[->] (output_2) -- (phantom_4);
  \draw[->] (output_3) -- (phantom_5);
  
  \draw[->] (input_1) edge node[above]{{$[0.95,1.05]$}} (hidden_1);
  
  \draw[->] (input_1) edge node[below, xshift=-2cm, yshift=0.1cm]{$[-0.05,0.05]$} 
  (hidden_2);
  
  \draw[->] (input_2) edge node[above, xshift=-2cm, yshift=0.cm]{$[-0.05,0.05]$} 
  (hidden_1);
  \draw[->] (input_2) edge node[below,yshift=-0.1cm]{$[0.95,1.05]$} 
  (hidden_2);
  
  \draw[->] (hidden_1) edge node[above, xshift=-0.5cm, yshift=0.3cm]{{$[0.95,1.05]$}} (output_1);

    \draw[->] (hidden_2) edge node[above, xshift=-0.5cm, yshift=-1cm]{{$[0.95,1.05]$}} (output_3);

    \draw[->] (hidden_1) edge node[above, xshift=2cm, yshift=-0.2cm]{{$[-0.05,0.05]$}} (output_2);
  
    \draw[->] (hidden_2) edge node[above, xshift=1.85cm, yshift=-0.6cm]{{$[0.45,0.55]$}} (output_2);
    
  \draw[->] (hidden_2) edge node[below, xshift=2.2cm, yshift=1.3cm]{$[-1.05,-0.95]$}  (output_1); 

  \draw[->] (hidden_1) edge node[below, xshift=2.2cm, yshift=-0.6cm]{$[-1.05,-0.95]$}  (output_3); 
  
\end{tikzpicture}}
    \caption{Illustration for Example~\ref{ex:interval_ffnn_multi}}
		\label{fig:net_multiclass}
	\end{figure}
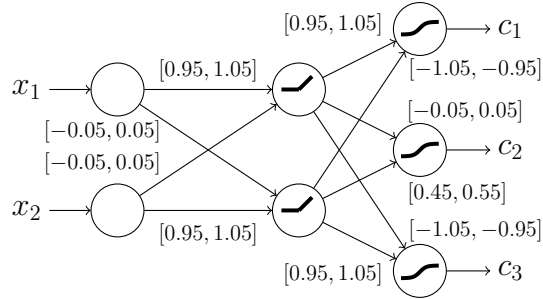

    The symbolic expression of the output interval is obtained by applying softmax activation on a vector of intervals $[\model{\theta}(x)_{c_1}, \model{\theta}(x)_{c_2}, \model{\theta}(x)_{c_3}]$ where each entry is the pre-softmax output interval for output classes \{$c_1, c_2, c_3$\}, then sorting the corresponding class interval for each class. We denote the node interval for the hidden nodes as $\model{\theta}(x)_{h_1}=\max(0, [0.95,1.05] \cdot x_1 + [-0.05,0.05] \cdot x_2)$ and $\model{\theta}(x)_{h_2}=\max(0, [-0.05,0.05] \cdot x_1 + [0.95,1.05] \cdot x_2)$. Then, the output node intervals can be expressed as $\model{\theta}(x)_{c_1} = [0.95,1.05] \cdot \model{\theta}(x)_{h_1} + [-1.05, -0.95]\cdot \model{\theta}(x)_{h_2},  \text{ }
    \model{\theta}(x)_{c_2} = [-0.05,0.05] \cdot \model{\theta}(x)_{h_1} + [0.45, 0.55]\cdot \model{\theta}(x)_{h_2}, \text{ }
    \model{\theta}(x)_{c_3} = [-1.05, -0.95] \cdot \model{\theta}(x)_{h_1} + [0.95, 1.05] \cdot \model{\theta}(x)_{h_2}$.

        %where the symbolic expression of the output is $\model{\theta}(x) = \sigma([0.95,1.05] \cdot \max(0, [0.95,1.05] \cdot x_1 + [-0.05,0.05] \cdot x_2) + [-1.05, -0.95] \cdot \max(0,[-0.05,0.05] \cdot x_1 +[0.95,1.05] \cdot x_2 ))$.
    %(Continuing from Examples~\ref{ex:cfx} and~\ref{ex:interval_abst}). 
    
    For the input $x=[2,2]$, we observe that $\abst{\theta}{\Delta}(x)$ produces %pre-softmax 
    output intervals %$[-0.6,0.6], [0.7, 1.32], [-0.6, 0.6]$ 
    $[0.114,0.322], [0.356, 0.51], [0.114, 0.322]$ respectively for each class, therefore the lower bound of class output node $c_2$ $(0.356)$ is greater than the upper bound of the other two classes $(0.322)$, i.e. $\abst{\theta}{\Delta} (x) = c_2$. Since $\model{\theta}(x) = c_2$, we can conclude that the set of model shifts $\Delta$ is sound for $c_2$. 
    We then check whether the previously computed counterfactual explanation $x' = [3,1]$ is $\name$-robust. Running $x'$ through $\abst{\theta}{\Delta}$ we obtain output intervals 
    %$[1.4, 2.6], [0.2, 0.82], [-2.6, -1.4]$
    $[0.617, 0.91], [0.08, 0.345], [0.01, 0.038]$, i.e. $\abst{\theta}{\Delta}(x) = c_1$. Therefore, the counterfactual is $\name$-robust for class $c_1$. Furthermore, this CE is strictly $\Delta$-robust. %Assume a new counterfactual is obtained as $x'' = [0.7, 0.86]$. The output interval produced by $\abst{\theta}{\Delta}$ is now $[0.5, 0.58]$, i.e. $\abst{\theta}{\Delta}(x'') = 1$. Thus, the counterfactual $x''$ is strictly $\name$-robust. 
\end{example}

\section{Computing $\name$-robustness with MILP}
\label{sec:milp_delta_rob}
To determine whether a CE $x'$ is $\Delta$-robust, we are interested in the output ranges of the interval abstraction $\abst{\theta}{\Delta}$ when $x'$ is passed in. \citet{PrabhakarA19} proposed an approach to compute the output ranges of an INN, which we adapt in this section to compute the reachable intervals for the output node in our interval abstraction $\abst{\theta}{\Delta}$. The output range estimation problem can be encoded in MILP\FL{; in the following,} we instantiate the encoding for fully connected neural networks with $k$ hidden layers (as first introduced in Example \ref{ex:nn}) with ReLU activation functions. The encoding introduces:
\begin{itemize}
    \item a real variable $v^{(0)}_{j}$ for $j \in [\card{V^{(0)}}]$, used to model the input of $\abst{\theta}{\Delta}$;
    \item a real variable $v^{(i)}_{j}$ to model the value of each hidden and output node $V^{(i)}$, for $i \in [k+1]$ and $j \in [\card{V^{(i)}}]$;
    \item a binary variable $\xi^{(i)}_{j}$ to model the activation state of each node in $V^{(i)}$, for $i \in [k]$ and $j \in [\card{V^{(i)}}]$.
\end{itemize}

% hidden layers
Then, for each layer index $i \in [k]$ and neuron index $j \in [\card{V^{(i)}}]$, the following set of constraints are asserted:

\begin{equation}
\begin{alignedat}{2}
      & \revision{C^{(i)}_{j}}   &&= \Bigl\{ v^{(i)}_{j} \geq 0, v^{(i)}_{j} \leq M(1 - \xi^{(i)}_{j}), \\
      & && v^{(i)}_{j} \leq \sum_{l=1}^{\card{V^{(i-1)}}} (W^{(i)}_{j,l} + \delta) \revision{v^{(i-1)}_{l}} + (B^{(i)}_j + \delta) + M\xi^{(i)}_{j}, \\
      & && v^{(i)}_{j} \geq \sum_{l=1}^{\card{V^{(i-1)}}} (W^{(i)}_{j,l} - \delta) \revision{v^{(i-1)}_{l}} + (B^{(i)}_j - \delta) \Bigr\}
\end{alignedat}    
\label{eqn:node}
\end{equation}

where $M$ is a sufficiently large constant, and $\delta$ is the magnitude of model shifts in $\Delta$. Each $C^{(i)}_{j}$ uses the standard big-M formulation to encode the ReLU activation~\cite{LomuscioM17} and estimate the lower and upper bounds of nodes in the INN.

Then, constraints pertaining to the output layer $k+1$ are asserted for each class $j \in \card{V^{(k+1)}}$.

\begin{equation}
\begin{alignedat}{3}
      & \revision{C^{(k+1)}_{j}}   = \Bigl\{ &&\revision{v^{(k+1)}_{j}} &&\leq \sum_{l=1}^{\card{V^{(k)}}} (W^{(k+1)}_{j,l} + \delta) \revision{v^{(k)}_l} + (B^{(k+1)}_j + \delta), \\
      & && v^{(k+1)}_{j} &&\geq \sum_{l=1}^{\card{V^{(k)}}} (W^{(k+1)}_{j,l} - \delta) \revision{v^{(k)}_l} + (B^{(k+1)}_j \revision{- \delta})
      \Bigr\}
\end{alignedat}    
\label{eqn:output}
\end{equation}

For more details about the encoding and its properties, we refer to the original work~\cite{PrabhakarA19}. Note that, instead of directly computing the model output, $\model{\Theta}(x) = \sigma(V^{(k+1)})$, our MILP program analyses the interval of the final-layer node values before applying the final sigmoid or softmax function, e.g. $V^{(k+1)} = W^{(k+1)} \cdot V^{(k)} + B^{(k+1)}$. The model output interval can be subsequently obtained by applying these final activation functions.

For the case of binary classification (e.g. as in Example~\ref{ex:ffnn}), testing $\Delta$-robustness of a CE amounts to solving one optimisation problem as shown in the following. %If for a CE $x'$ such lower bound satisfies ${v^{(k+1)}_1}_{lb} \geq 0$, then after the final sigmoid function, $\sigma({v^{(k+1)}_1}_{lb}) \geq 0.5$, meaning that $\abst{\theta}{\Delta}(x') = 1$, $x'$ is thus $\Delta$-robust. More formally,

\begin{definition}
    Consider an input $x \in \mathcal{X}$ and a fully connected neural network $\model{\theta}$ with $k$ hidden layers, and $\model{\theta}{(x)} = 0$. Let $\abst{\theta}{\Delta}$ be the interval abstraction of $\model{\theta}$ for a set of plausible model shifts $\Delta$. Let ${v^{(k+1)}_1}_{lb}$ be the solution of the optimisation problem 
\begin{subequations}
\begin{alignat*}{2}
&\underset{v, \xi}{\text{min}}  && \quad{v^{(k+1)}_1} \\
&\text{subject to} &&\quad C^{(i)}_{j}, \text{  }i \in [k+1], j \in [\card{V^{(i)}}] %\\
%& && \quad C^{(k+1)}_{1}
\end{alignat*}
\label{eqn:delta_robustness_test}
\end{subequations}
We say that a counterfactual explanation $x'$ \JJ{is \emph{$\Delta$-robust}} if ${v^{(k+1)}_1}_{lb} \geq 0$.

\label{def:delta_robustness_test}
\end{definition}

Indeed, if for a CE $x'$ the lower bound of the output node satisfies ${v^{(k+1)}_1}_{lb} \geq 0$, then after the final sigmoid function, $\sigma({v^{(k+1)}_1}_{lb}) \geq 0.5$, meaning that $\abst{\theta}{\Delta}(x') = 1$, $x'$ is thus $\Delta$-robust. 

For multi-class classification (e.g. as in Example~\ref{ex:ffnn_multi}), the exact output range for each class $j \in \card{V^{(k+1)}}$ can be computed by solving two optimisation problems that minimise, respectively maximise, variable $v^{(k+1)}_{j}$ subject to constraints~\ref{eqn:node}-\ref{eqn:output}. Referring to Definitions~\ref{def:multi_classification_interval} and \ref{def:delta_robustness_multiclass}, we are interested in comparing the lower bound of the desirable target class output interval with the upper bound of output intervals from other classes. Therefore, testing $\Delta$-robustness in multi-class classification amounts to solving $\card{V^{(k+1)}}$ optimisation problems.

\begin{definition}
    Consider an input $x \in \mathcal{X}$ and a fully connected neural network $\model{\theta}$ with $k$ hidden layers with $\model{\theta}{(x)} = c \in \{1,\ldots,\ell\}$ and the desirable target class $c' \in \{1,\ldots,\ell\}$ such that $c' \neq c$. Let $\abst{\theta}{\Delta}$ be the interval abstraction of $\model{\theta}$ for a set of plausible model shifts $\Delta$. Let ${v^{(k+1)}_{c'}}_{lb}$ be the solution of the optimisation problem 
\begin{subequations}
\begin{alignat*}{2}
&\underset{v, \xi}{\text{min}}  && \quad{v^{(k+1)}_{c'}} \\
&\text{subject to} &&\quad C^{(i)}_{j}, \text{  }i \in [k+1], j \in [\card{V^{(i)}}] %\\
%& && \quad C^{(k+1)}_{j}, j \in [\card{V^{(k+1)}}]
\end{alignat*}
\label{eqn:delta_robustness_test_multi_min_1}
\end{subequations}

For each class $c'' \in \{1,\ldots,\ell\}$ such that $c'' \neq c'$, let ${v^{(k+1)}_{c''}}_{ub}$ be the solution of the optimisation problems 
\begin{subequations}
\begin{alignat*}{2}
&\underset{v, \xi}{\text{max}}  && \quad{v^{(k+1)}_{c''}} \\
&\text{subject to} &&\quad C^{(i)}_{j}, \text{  }i \in [k+1], j \in [\card{V^{(i)}}] %\\
%& && \quad C^{(k+1)}_{j}, j \in [\card{V^{(k+1)}}]
\end{alignat*}
\label{eqn:delta_robustness_test_multi_min_2}
\end{subequations}

We say that a counterfactual explanation $x'$ \JJ{is \emph{$\Delta$-robust} for class $c'$} if ${v^{(k+1)}_{c'}}_{lb} \geq {v^{(k+1)}_{c''}}_{ub}$.

\label{def:delta_robustness_test_multiclass}
\end{definition}

%\td{add definitions for delta robustness test for multiclass too}
Similar to the binary case, if a CE passes the above $\Delta$-robustness test, then after applying the final softmax function, the predicted class probability of the desirable target class will always be greater than that of the other classes.

In practice, with white-box access to the classifier, the above optimisation problems can be conveniently encoded using any off-the-shelf optimisation solver. In this work we use the Gurobi engine\footnote{https://www.gurobi.com/solutions/gurobi-optimizer/}.

\section{Algorithms}
\label{sec:algorithms}
\JJ{We now show how the $\Delta$-robustness tests introduced above can be leveraged to} generate CEs with formal robustness guarantees. 

\subsection{Embedding $\Delta$-robustness tests in existing algorithms}

%We will now show how $\name$-robustness can be used to guide CE generation algorithms toward %generating explanations 
% {CEs} with formal robustness guarantees.

We propose an approach (Algorithm~\ref{alg:algo}) that can be applied on top of any CE generation algorithms which explicitly parameterise the tradeoff between validity and cost. For example, if the method is optimising a loss function containing a validity loss term and a cost loss term with a tradeoff hyperparameter (similar to \cite{Wachter_17}), then modifying the hyperparameter to allow more costly CE with better validity (higher class probability) could lead to more robust CEs. Such heuristics, identified as necessary conditions for more robust CEs, are discussed in several recent studies \cite{blackconsistent,pmlr-v162-dutta22a,DBLP:conf/icml/HammanNMMD23,DBLP:conf/icml/KrishnaML23}. 

The algorithm proceeds as follows. First, an interval abstraction is constructed for the classifier $\model{\theta}$ and targeted plausible model change $\Delta$; the latter is then optionally checked for soundness (Definition~\ref{def:delta_robustness}) \revision{using the same MILP programs for $\Delta$-robustness tests}, depending on whether we are aiming at finding strict or non-strict robust CEs. 
Then, the algorithm performs a $\Delta$-robustness test for the CE generated by the base method. If the test passes, then the algorithm terminates and returns the solution. Otherwise, the search continues iteratively, at each step the hyperparameter of the base method is modified such that CEs of increasing {distance} can be found. 
These steps are repeated until a %maximum
{threshold} number of iterations $t$ is reached. As a result, the algorithm is incomplete, \FL{as} it may report that no $\name$-robust CE can be found within $t$ steps (while one may exist for larger $t$). In such cases, the last CE found before the algorithm terminates is returned\JJ{, assuming the CE found at each iteration is more robust than the previous iteration.} %The search space is also limited to that of the base method, which might not cover any $\Delta$-robust CE solutions, meaning the algorithm is also not sound. 
\JJ{%Whether or not a $\Delta$-robust CE can be obtained largely depends on the base CE method used. 
As we will see in Section \ref{sec:experiments}, we have identified a configuration of our iterative algorithm with a MILP-based method \cite{mohammadi2021scaling} as the base CE generation method, which empirically overcomes the above limitations and is always able to find $\Delta$-robust CEs.}

%\td{AR: this transition is a bit clanky still, do we want to add more blah blah? JJ: is the above better? \AR{yes, thanks}}

\begin{algorithm}[t]
\caption{Iterative algorithm embedding $\Delta$-robustness tests}\label{alg:algo}
\begin{algorithmic}[1]
\Require Classifier $\model{\theta}$, %factual 
input $x$ such that $\model{\theta}(x) = 0$,
\State $\qquad\quad$ set of plausible model shifts $\Delta$ and threshold $t$
\State \textbf{Step 1}: build \FL{encoding for} interval abstraction $\abst{\model}{\Delta}$.
\State \textbf{Step 2} (Optional): check the soundness of $\Delta$
%\If{$\Delta$ is sound}
\While{iteration number $< t$}
    \State \textbf{Step 3}: compute CE $x'$ for $x$ and $\model{\theta}$ using base method
    %\State $\quad\quad\quad$using base method
    \If{$\abst{\theta}{\Delta}(x') = 1$}
        \State \textbf{return} $x'$
    \Else
        \State \textbf{Step 4}: %``relax cost'' 
         {increase allowed distance} of next CFX 
        \State \textbf{Step 5}: increase iteration number
    \EndIf
\EndWhile
%\EndIf
\State \textbf{return} the last $x'$ found
\end{algorithmic}
\end{algorithm}

\subsection{Robust Nearest-neighbour Counterfactual Explanations (RNCE)}
\label{ssec:rnce}

We also propose a robust and plausible CE algorithm shown in Algorithm~\ref{alg:algo1}, which is complete under mild assumptions. After some initialisation steps, an interval abstraction is constructed for the model $\model{\theta}$ and set $\name$ in Step 1 (Alg.~\ref{alg:algo1}, 6). \revision{Similarly to Algorithm~\ref{alg:algo}, the checking for soundness of $\Delta$ step is optional.} The algorithm then moves on to Step 2 %\delete{and 3 where a NNCE is constructed (Alg.~\ref{alg:algo1}, ll.~8-10). First, } 
where the dataset $\dataset$ used to train $\model{\theta}$ is traversed to identify potential CEs for $x$ and filter out unsuitable inputs as described in Algorithm~\ref{alg:algo2}. 
In a nutshell, %the 
Algorithm{~\ref{alg:algo2}} iterates through $\dataset$ and picks instances that satisfy the counterfactual requirement, parameterised according to a robustness criterion specified by the user via the \texttt{robustInit} parameter. When \texttt{robustInit} is %$\top$
\texttt{T} (\texttt{True}), the interval abstraction is used to check, for every instance in the dataset, whether it satisfies $\name$-robustness (Definition~\ref{def:delta_robustness}) prior to adding it to $\mathcal{S}$ (Alg.~\ref{alg:algo2}, 9). 
Alternatively, when setting \texttt{robustInit} to %$\bot$
\texttt{F} (\texttt{False}), instances are added to the set of candidate CEs $\mathcal{S}$ as long as their {predicted} label differs from that of $x$ (Alg.~\ref{alg:algo2}, 8). In the latter case, the $\name$-robustness guarantees are postponed to Step 4.

Once the set $\candidates$ of potential candidates is obtained from $\dataset$ (Alg.~\ref{alg:algo1}, 8), we fit a k-d tree 
$\tree$ to efficiently store spatial information about them {(Step 3)}. The algorithm then enters its final step, where the best CE is selected from $\candidates$ as described in Algorithm~\ref{alg:algo3} and returned to the user (Alg.~\ref{alg:algo1}, 10). First, $\tree$ is queried to obtain the closest robust
 {NN}CE $x'_{nn}$. When Algorithm~\ref{alg:algo2} is instantiated with \texttt{robustInit} %$= \top$
$= \texttt{T}$, the first nearest neighbour returned from $\tree$ is guaranteed to be robust. Otherwise, several queries might be needed to obtain an instance from $\tree$ that also satisfies $\name$-robustness (Alg.~\ref{alg:algo3}, 4-6).\footnote{For clarity, {Algorithm~\ref{alg:algo3}} describes the latter case; the robustness check in omitted in the code when \texttt{robustInit} 
is $\texttt{T}$.}
%$= \top$.}. Note that in the latter case, applying interval abstraction over all dataset points is avoided. Therefore, unless the dataset size is small and the number of input points to be explained is large, the computation is quicker when \texttt{robustInit}$=\bot$, as is the case in our empirical studies (Table~\ref{tab:time}).  Once an instance is obtained, a robust CE is found and returned.  {Note that} instantiations of \ouralgo{} with different \texttt{robustInit} will return the same CE, since both 
 {Thus, instantiations of different \texttt{robustInit}} settings help to find the nearest point satisfying $\name$-robustness. %\todo{should this be a remark?} 
As the nearest neighbour may not be optimal in terms of proximity to the original instance $x$, further optimisation steps can be triggered by setting the parameter \texttt{optimal} to %\top$
$\texttt{T}$. This starts a line search to find the closest robust CE to $x$ (Alg.~\ref{alg:algo3}, 7-10). The CE computed by Algorithm~\ref{alg:algo3} is then returned to the user. %  {(Alg.~\ref{alg:algo3}, line~14)}. 

We stress that when \texttt{robustInit}$=$\texttt{T} and \texttt{optimal}$=$\texttt{F}, once $\tree$ has been fitted, the CE generation time (Alg.~\ref{alg:algo1}, 10) for any input is $O(log(|\dataset|))$, and $\tree$ can be used to query CEs for any number of inputs efficiently.
When the number of inputs requiring CEs is far smaller than $|\dataset|$, \texttt{robustInit}$=$\texttt{F} may result in faster computation than \texttt{robustInit}$=$\texttt{T} (including time for obtaining $\tree$), see \ref{app:a_robustinit_computation_time} for more details.

%\todo{JJ: comment on the equivalency of CE when robustInit is True or False, and some intuition behind this parameter related to dataset or task size? for example if there are 1000 instances and want to find CEs for 1000 incoming users then might want to use robustInit=True; if 100000 instances and need to explain 50 perhaps want to use robustInit=False} 

\begin{algorithm}[t!]
\caption{\ouralgo: main routine} \label{alg:algo1}
\begin{algorithmic}[1]
\Require input $x$, model $\model{\theta}$, 
\State $\qquad$ training dataset $\dataset = \{(x_1,y_1),\ldots,(x_n,y_n)\}$, 
\State $\qquad$ \revision{(sound)} set of plausible model shifts $\Delta$,
\State $\qquad$ parameters \texttt{robustInit}, \texttt{optimal} $\in %\{\bot, \top\}$  
\{\texttt{F, T}\}$
\State \textbf{Init}: $c \gets \model{\theta}(x)$; $x' \gets \emptyset$
%\State \textbf{Step 1}: create interval abstraction
\State \textbf{Step 1}: \revision{build \FL{encoding for} interval abstraction $\abst{\model}{\Delta}$}
%\State $\abst{\theta}{\Delta} \gets \texttt{buildAbstraction}(\model{\theta})$    
% \State \revision{build interval abstraction $\abst{\model}{\Delta}$}
% \If{$\abst{\Theta}{\Delta}(x) \neq c$}
% \State \textbf{return} $x'$
% \EndIf
\State \textbf{Step 2}: select candidate counterfactuals
% \If{\texttt{robustInit} is $\bot$}
% \State $\candidates \gets \texttt{filterCandidates}(x, \model{\Theta}, \dataset)$
% \Else 
\State $\candidates \!\!\gets\!\! \texttt{getCandidates}(x,\! \model{\theta}, \!\dataset\!, \abst{\theta}{\Delta},\! \texttt{robustInit})$
% \EndIf
\State \textbf{Step 3}: fit k-d tree $\tree$ on $\candidates$
\State \textbf{Step 4}: $x' \!\!\gets\!\! \texttt{getRobustCE}(x,\! \model{\theta}, \abst{\theta}{\Delta}, \!\tree\!,\! \texttt{optimal})$
\State \textbf{return} $x'$
\end{algorithmic}
\end{algorithm}

\begin{algorithm}[t!]
\caption{\ouralgo: \texttt{getCandidates} } \label{alg:algo2}
\begin{algorithmic}[1]
\Require input $x$, model $\model{\theta}$,
\State $\qquad$ dataset $\dataset= \{(x_1,y_1),\ldots,(x_n,y_n)\} $
\State $\qquad$ interval abstraction $\abst{\theta}{\Delta}$
\State $\qquad$ parameter \texttt{robustInit} $\in %\{\bot, \top\}$
\{\texttt{F, T}\}$
\State \textbf{Init}: $c \gets \model{\theta}(x)$, $\candidates \gets \{\}$
\State \textbf{Step 1}: find candidate subset of the training dataset
%\State $\candidates=\{\}$
\For{$(x_i, y_i) \in \dataset$}
    \If{\texttt{robustInit} %is $\bot$
    is $\texttt{F}$}
        \State Add $(x_i, y_i)$ to $\candidates$ if $\model{\theta}(x_i) = 1-c$
    \Else \: Add $(x_i, y_i)$ to $\candidates$ if $\abst{\theta}{\Delta}(x_i) = 1-c$
        %\State Add $(x_i, y_i)$ to $\candidates$ if $\abst{\new{\theta}}{\Delta}(x_i) = 1-c$ \todo{put in one line}
    \EndIf
\EndFor
\State \textbf{return} $\candidates$
\end{algorithmic}
\end{algorithm}

\begin{algorithm}[t!]
\caption{\ouralgo: \texttt{getRobustCE} } \label{alg:algo3}
\begin{algorithmic}[1]
\Require input $x$, model $\model{\theta}$, interval abstraction $\abst{\theta}{\Delta}$, 
\State $\qquad$ k-d tree $\tree$, parameter \texttt{optimal} $\in %\{\bot, \top\}$
\{\texttt{F, T}\}$
% \State \textbf{Init}: initialise interval abstraction if required
% \If{\texttt{robustInit} is $\top$}
%     \State $\model{\theta} \gets \texttt{buildAbstraction}(\model{\theta})$
% \EndIf
\State \textbf{Init}: $c \gets \model{\theta}(x)$; %$l \gets 0$; $u \gets 1$; 
$x' \gets \emptyset$; $a \gets 1$; $s \gets 0.05$
\State \textbf{Step 1}: find robust CE
\While{$\tree.\texttt{queryNextNeighbour}(x)$ is not $\emptyset$} 
\State $x'_{nn} \gets \tree.\texttt{queryNextNeighbour}(x)$
\If{$\abst{\theta}{\Delta}(x'_{nn}) = 1-c$} \: $x' \gets x'_{nn}$
%\State $x' \JJ{\gets} x'_{nn}$
\EndIf
\EndWhile
\If{\texttt{optimal} is $\texttt{T}$}
%$\top$} 
% \While{$ u - l > 0$}
% \State $m = \frac{u-l}{2}$
% \State $z = m\cdot x + (1 - m) \cdot x_{nn}$
% \If{$\abst{\Theta}{\Delta}(z) = 1-c$ }
% \State $x' \gets z$; $u \gets m$
% \Else
% \State $l \gets m$
% \EndIf
% \EndWhile
%\State $x' \gets \texttt{doLineSearch}(x, x')$

%\State $x'_{opt} = x'$
\While{$a > 0$} %$\abst{\Theta}{\Delta}(x_{opt}) = 1-c$}
\State $x'_{line} \gets ax' + (1-a)x$; \: $a \gets a - s$
%\State $a \gets a - s$
\If{$\abst{\theta}{\Delta}(x'_{line})=1-c$} \: $x' \gets x'_{line}$
%\State $x' \JJ{\gets} x'_{line}$
\EndIf
\EndWhile

\EndIf
\State \textbf{return} $x'$
\end{algorithmic}
\end{algorithm}

\begin{remark}
\label{remark:soundness}
\ouralgo{} is sound.
\end{remark}

Soundness of the procedure is guaranteed by construction, as solutions can only be chosen among the set of instances {for} which the classification label flips (see Algorithm~\ref{alg:algo2}). This is also true when \ouralgo{} is instantiated with \texttt{optimal} %\top
%is
set to \texttt{T}, as an additional check is performed in line 10 of Algorithm~\ref{alg:algo3} to ensure that only valid CEs are returned.

%\todo{FL:Robust init must be true, S not empty alone does not guarantee that we can find a robust CE}
\begin{remark}
\label{remark:completeness}
\ouralgo{} is complete if there exists an $(x', y') \in \dataset$ such that $\abst{\theta}{\Delta}(x') \neq \model{\theta}(x)$. %\texttt{robustInit} $= \top$ and $\candidates \neq \emptyset$.
\end{remark}

Completeness of the procedure is conditioned on the existence of at least one $\name$-robust instance $x'$ in $\dataset$ whose label differs from that of the original input $x$. %When the procedure is instantiated with \texttt{robustInit} $= \top$, Algorithm~\ref{alg:algo2} is guaranteed to identify such input as a candidate CE and add it to set $\candidates$. 
{Completeness is not affected by the configurations of parameters in \ouralgo{}.} When \texttt{robustInit} is $\texttt{T}$, Algorithm~\ref{alg:algo2} is guaranteed to identify such input as a candidate CE and add it to set $\candidates$; when \texttt{robustInit}
%$=\bot$
is $\texttt{F}$, the multiple queries of the next nearest neighbour and the following $\name$-robustness test (Alg.~\ref{alg:algo3}%\todo{should this be alg 2? -> it's alg 3}
, l\FT{ines} 4-6) will also identify a feasible result. For the parameter \texttt{optimal}, Algorithm~\ref{alg:algo3} will always return one among $x'$ or an optimised version of it (in terms of distance) for which robustness is still guaranteed (Alg.~\ref{alg:algo3}, 10). %We notice that for all other configurations of \ouralgo{} are not covered by Remark~\ref{remark:completeness} completeness cannot be guaranteed; however, 
Our experimental analysis in Section~\ref{sec:experiments} shows that our approach is always able to find $\Delta$-robust CEs.

\begin{remark}
\label{remark:equivalencetoplainnnce}
\ouralgo{} is equivalent to a plain NNCE algorithm if \texttt{optimal} %=\bot$
is \texttt{F}
and $\Delta=\emptyset$.
\end{remark}

\JJ{\ouralgo{} collapses to a plain NNCE algorithm if the set of plausible model shifts $\Delta$ is an empty set, thus $\abst{\theta}{\Delta}$ will be equivalent to the original model, $\model{\theta}$}. Then, the choice of the parameter \texttt{robustInit} makes no impact as lines 8 and 9 of Alg.~\ref{alg:algo2} become identical (the multiple queries specified in Alg.~\ref{alg:algo3}, 4-6 are also not needed). Note that the line search controlled by \texttt{optimal} can also be performed when $\Delta=\emptyset$.

\begin{figure*}[t!]
	\centering
	\begin{subfigure}{0.4\textwidth}
		\centering
		\scalebox{2}{\begin{tikzpicture}

\draw[fill=red, fill opacity=0.05] (0,0) -- (0,1.2)  -- (0.5,1.2) .. controls (1.,1.1) and (0.4,0.6).. (1.3,0) -- cycle ;

\draw[fill=green, fill opacity=0.05] (2.,0) -- (2.,1.2)  -- (0.5,1.2) .. controls (1.,1.1) and (0.4,0.6).. (1.3,0) -- cycle ;

\draw[dashed] (0.2,1.2) .. controls (1.,1.1) and (0.5,0.5).. (1.8,0) ;

\draw[red,thick,fill=red] (0.6,0.4) circle[radius=1pt];

\draw (1,0.4) node[cross=2pt,green,thick] {};
\draw (1.5,0.9) node[cross=2pt,gray,thick] {};

\end{tikzpicture}}	
		\caption{}
		\label{fig:a}
	\end{subfigure}
	% \hfil
	\begin{subfigure}{0.4\textwidth}
		\centering
  		\scalebox{2}{\begin{tikzpicture}

\draw[fill=red, fill opacity=0.05] (0,0) -- (0,1.2)  -- (0.5,1.2) .. controls (1.,1.1) and (0.4,0.6).. (1.3,0) -- cycle ;
\draw[fill=green, fill opacity=0.05] (2.,0) -- (2.,1.2)  -- (0.5,1.2) .. controls (1.,1.1) and (0.4,0.6).. (1.3,0) -- cycle ;

\draw[dashed] (0.2,1.2) .. controls (1.,1.1) and (0.5,0.5).. (1.8,0) ;

\draw[dotted] (1.05,0.4) -- (0.6,0.4);

\draw[red,thick,fill=red] (0.6,0.4) circle[radius=1pt];

\draw (1.05,0.4) node[cross=2pt,gray,thick] {};
\draw (1.5,0.9) node[cross=2pt,gray,thick] {};subfigure
\draw (0.88,0.4) node[cross=2pt,green,thick] {};

\end{tikzpicture}}	
		\caption{}
		\label{fig:b}
	\end{subfigure}

	\begin{subfigure}{0.4\textwidth}
		\centering
		\scalebox{2}{\begin{tikzpicture}

\draw[fill=red, fill opacity=0.05] (0,0) -- (0,1.2)  -- (0.5,1.2) .. controls (1.,1.1) and (0.4,0.6).. (1.3,0) -- cycle ;

\draw[fill=green, fill opacity=0.05] (2.,0) -- (2.,1.2)  -- (0.5,1.2) .. controls (1.,1.1) and (0.4,0.6).. (1.3,0) -- cycle ;

\draw[dashed] (0.2,1.2) .. controls (1.,1.1) and (0.5,0.5).. (1.8,0) ;

\draw[red,thick,fill=red] (0.6,0.4) circle[radius=1pt];

\draw (1.5,0.9) node[cross=2pt,green,thick] {};
\draw (1.0,0.4) node[cross=2pt,gray,thick] {};

\end{tikzpicture}}	
		\caption{}
		\label{fig:c}
	\end{subfigure}
	% \hfil
	\begin{subfigure}{0.4\textwidth}
		\centering
		\scalebox{2}{\begin{tikzpicture}

\draw[fill=red, fill opacity=0.05] (0,0) -- (0,1.2)  -- (0.5,1.2) .. controls (1.,1.1) and (0.4,0.6).. (1.3,0) -- cycle ;

\draw[fill=green, fill opacity=0.05] (2.,0) -- (2.,1.2)  -- (0.5,1.2) .. controls (1.,1.1) and (0.4,0.6).. (1.3,0) -- cycle ;

\draw[dashed] (0.2,1.2) .. controls (1.,1.1) and (0.5,0.5).. (1.8,0) ;

\draw[dotted] (0.6,0.4) -- (1.5,0.9);

\draw[red,thick,fill=red] (0.6,0.4) circle[radius=1pt];

\draw (1.5,0.9) node[cross=2pt,gray,thick] {};
\draw (1.0,0.4) node[cross=2pt,gray,thick] {};
\draw (0.95,0.58) node[cross=2pt,green,thick] {};

\end{tikzpicture}}	
		\caption{}
		\label{fig:d}
	\end{subfigure}
	\caption{%\delete{\ouralgo{} behaviours based on parameter \texttt{optimal}, $\name$ \FT{set to\AR{:}} $\JJ{\texttt{F}}$, empty (\ref{fig:a})\AR{;} %$\top$ $\JJ{\texttt{T}}$, empty (\ref{fig:b})\AR{;} %$\bot$ $\JJ{\texttt{F}}$, non-empty (\ref{fig:c})\AR{;} and %$\top$ $\JJ{\texttt{T}}$, non-empty (\ref{fig:d}). }
  \ouralgo{} behaviours when \JJ{$\texttt{optimal}=\texttt{F}$, $\Delta=\emptyset$ (\ref{fig:a}); $\texttt{optimal}=\texttt{T}$, $\Delta=\emptyset$ (\ref{fig:b}); $\texttt{optimal}=\texttt{F}$, $\Delta\neq\emptyset$ (\ref{fig:c});  $\texttt{optimal}=\texttt{T}$, $\Delta\neq\emptyset$ (\ref{fig:d})}. 
 Each figure depicts the CE that is chosen (green cross) among a set of candidate CEs (grey crosses) for a given input (red circle) and model $\model{\theta}$. 
 The continuous curved line represents the decision boundary of $\model{\theta}$; the dashed line represents a possible change in the decision boundary under $\name$. Figures~\ref{fig:a} and \ref{fig:b} show configurations under which \ouralgo{} may return CEs that are not $\name$-robust, whereas Figures~\ref{fig:c} and \ref{fig:d} depict robust ones. }%\todo{FL: What is the meaning of "$\Delta$ set to empty?". The algorithm has a parameter, robustinit, which is never mentioned in this caption} \todo{JJ: we explained these in Remark 3 and wanted to show the difference between NNCE and OURS. robustInit doesn't affect the CE results.}}
	\label{fig:comparison_configurations}
\end{figure*}
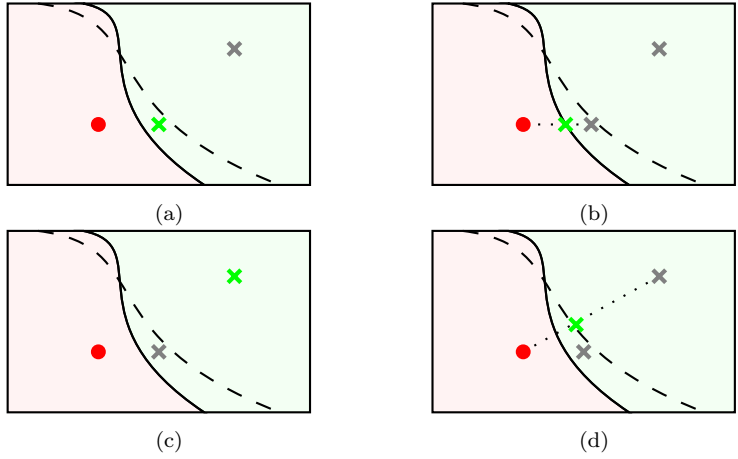

Figure~\ref{fig:comparison_configurations} shows a pictorial representation of the different behaviours that can be obtained from \ouralgo{} based on the configuration of the parameter \texttt{optimal} and whether $\Delta$ is empty. When $\Delta=\emptyset$ and \texttt{optimal} %$=\bot$\
is \texttt{F}
(Figure~\ref{fig:a}), %parameters \texttt{robustInit} and \texttt{optimal}. When the parameters are both set to $\bot$ (figure~\ref{fig:a}, 
the algorithm behaves like a standard \FT{algorithm producing} NNCE\FT{s} and returns the closest counterfactual instance. However, this CE may not be robust. Analogous results may be obtained when $\Delta=\emptyset$ and \texttt{optimal} %$=\top$
is $\texttt{T}$ (Figure~\ref{fig:b}). %We note that robustness may still be obtained if the closest CE also happens to be robust to begin with; however the algorithm cannot offer guarantees in this setting. 
Conversely, when $\Delta\neq\emptyset$ (Figures~\ref{fig:c} and \ref{fig:d}), \ouralgo{} will only operate on robust instances as candidates, thus guaranteeing the CEs' $\name$-robustness. 

\section{Experiments}
\label{sec:experiments}
In this section, we demonstrate through experiments how $\Delta$-robustness can be \FL{used in practice}. We first described the experimental setup used for our experiments, and present two strategies to find realistic magnitudes in the plausible model changes given a dataset and the corresponding original classifier. Then, using these magnitudes, we construct interval abstractions to test the robustness of existing CE generation methods, showing a lack of $\Delta$-robustness in state-of-the-art methods. Finally, we benchmark the proposed algorithms to compute $\Delta$-robust CEs with evaluation metrics for proximity, plausibility, and robustness, showing the effectiveness of our methods.

Additionally, we demonstrate the applicability of our methods to multi-class classification tasks, while most existing methods focus on obtaining robust CEs only for binary classifiers. \FL{Finally, we present experiments aimed at benchmarking the runtime performance of our procedures and compare them with existing approaches.}

The code for our implementation and experiments is available at:
\begin{center}
    \url{https://github.com/junqi-jiang/interval-abstractions}.
\end{center}

\subsection{Setup and evaluation metrics}

\begin{table}[ht]
    \centering
    \resizebox{0.9\columnwidth}{!}{
    \begin{tabular}{ccccc}
        \hline 
        \textbf{dataset} & \textbf{data points} & \textbf{attributes} & \textbf{NN accuracy} & \textbf{LR accuracy} \\
        \hline
        adult & 48832 & 13 & .847$\pm$.006 & .828$\pm$.004 \\
        compas & 6172 & 7 & .844$\pm$.010 & .833$\pm$.011 \\
        gmc & 115527 & 10 & .860$\pm$.001 & .852$\pm$.002 \\
        heloc & 9871 & 21 & .728$\pm$.015 & .725$\pm$.013 \\
        \hline
        
    \end{tabular}
    }
    \caption{Dataset and classifier details.}
    \label{tab:datasetsandclassifiers}
\end{table}

We experiment on four popular tabular datasets for benchmarking performances of CE generation algorithms, \JJ{\emph{adult} income dataset \cite{dua2017uci}, \emph{compas} recidivism dataset \cite{compas}, give me some credit (\emph{gmc}) dataset \cite{gmc}, Home equity line of credit (\emph{heloc}) dataset \cite{heloc}, }from the CARLA library \cite{pawelczyk2021carla}. All datasets contain min-max scaled continuous features, one-hot encoded binary discrete features, and two output classes. Class 0 is the unwanted class while class 1 is the target class for CEs. We train neural networks with two hidden layers and 6 to 20 neurons in each layer, and logistic regression models (used in Section \ref{ssec:experiments_benchmarking}) on the datasets. Table~\ref{tab:datasetsandclassifiers} reports the dataset sizes, the number of attributes of the datasets, and the 5-fold cross-validation accuracy of the classifiers obtained. 

We randomly split each dataset $\dataset$ into two halves, $\dataset_1$ and $\dataset_2$, each including a training and a test set. We use $\dataset_1$ to train the classifiers as stated above, and we call these original models $\model{1}$. $\dataset_2$ is used to simulate scenarios with incoming data after $\model{1}$ are deployed, explained in the next sections. %From the $\dataset2$ test set, we randomly select 50 data points classified to class 0, for which CEs are generated. These CEs are then evaluated against three aspects using the standard metrics in the literature. For proximity, we calculate the average \emph{$\ell1$ cost} between the test input and its CE, which captures both closeness of CEs and sparsity of changes \cite{Wachter_17}. For plausibility, we report the average 10-local outlier factor score \emph{lof} \cite{lof} which quantifies the local data density. A {lof} score close to value 1 indicates an inlier; the more it deviates from 1, the less plausible the CE. For robustness, we report validity, i.e. the percentage of CEs correctly classified to class 1, under the following three types of model perturbations. Also, we measure \emph{time}, the total computation time (s) for generating CEs for 50 inputs. 
%We identify three model retraining scenarios under which previous works evaluate robustness. \emph{Incremental retraining} refers to the case where $\model{1}$ is periodically fine-tuned by gradient descent on some portions of $\dataset_2$ with a small number of iterations.
%\emph{Complete retraining} means retraining a model with the same hyperparameter setting using the concatenation of $\dataset_1$ and $\dataset_2$. \emph{Leave-one-out retraining} concerns obtaining a new model using a subset of $\dataset_1$ with 1\% of data points removed. %\td{where to put} We train 10 new classifiers with different random seeds under each of the three scenarios to evaluate the empirical robustness of CEs in subsequent tasks, and we call these updated models $\model{2}$. %and evaluate the validity of CEs on these models, reported as \emph{validity after retraining (vr)}.

\subsection{Identifying $\delta$ values}
\label{ssec:experiments_identify_delta_values}

\JJ{Recall from Definition~\ref{def:set_of_plausible_shifts}  that }$\delta$ upper-bounds the magnitude of the model shifts in $\Delta$, as measured \FL{by the} p-distance \FL{(Definition~\ref{def:distance_between_models})}. When \JJ{practically using} $\Delta$-robustness, $\delta$, values\JJ{, regarded as the hyperparameter in our method}, need to be first determined. We now propose two realistic strategies for obtaining $\delta$ values, depending on how the underlying classifier is retrained with new data in the application. 

\paragraph{Incremental retraining} This refers to the case where $\model{1}$ is periodically fine-tuned by gradient descent on some portions of $\dataset_2$ with a small number of iterations. 
In this scenario, depending on how \JJ{many data points} in $\dataset_2$ are used for retraining, the magnitude of parameter changes could be small. Therefore, the $\delta$ values can be directly estimated by calculating the p-distance between $\model{1}$ and the updated classifiers. By observing the p-distances when retraining on various portions of $\dataset_2$, the model developers could potentially link the estimated $\delta$ values to time intervals in real-world applications, depending on \FL{the frequency at which} new data are collected. %For example, consider a fictional dataset where retraining a new dataset with size $1000$ would result in a $\delta$ value of $0.08$. If it typically takes 3 weeks to collect $1000$ new data points, then a $\Delta$-robust CE with $\delta=0.08$ can be estimated to be valid for 3 weeks.

\paragraph{Complete or leave-one-out retraining} \emph{Complete retraining} \FL{refers to a scenario where} a model is retrained with the same hyperparameter setting \FL{as $\model{1}$}, using the concatenation of $\dataset_1$ and $\dataset_2$. \emph{Leave-one-out retraining} concerns obtaining a new model using a subset of $\dataset_1$ with 1\% of data points removed. As mentioned in \cite{DBLP:conf/icml/HammanNMMD23}, when retraining from scratch using the concatenation of $\dataset_1$ and $\dataset_2$, it becomes unrealistic to upper-bound the weights and biases differences in classifiers as the p-distance can be arbitrarily large. In this case, we can treat $\delta$ as a hyperparameter and empirically find the \revision{minimum value which can ensure robustness on a validation set}. Similarly to the standard train-validation-test split for evaluating the accuracy of machine learning models, we propose to use a held-out validation set to estimate $\delta$ values which lead to sufficient robustness under such retraining scenarios. The procedure can be summarised as follows:

\begin{enumerate}
    \item Retrain from scratch some new classifiers using $\dataset_1$ and $\dataset_2$\footnote{Complete retraining is possible in an experimental environment. In practice, however, if $\dataset_2$ is not available at the time of robust CE generation, leave-one-out retraining on $\dataset_1$ could be a viable option in step 1.}, set initial $\delta$ value to a sufficiently small value;
    \item Generate $\Delta$-robust CEs using the current $\delta$;
    \item Evaluate the percentage of the explanations which are valid under all the retrained models;
    \item Examine the above empirical validity, if it has not reached 100\% then increase the $\delta$ value and repeat steps 2-3. Choose the smallest $\delta$ value which results in 100\% validity.
\end{enumerate}

The \FL{termination} condition in step 4 balances the robustness-cost tradeoff. Once the empirical validity on multiple retrained models reaches 100\% for the validation set, increasing $\delta$ further will negatively affect the proximity evaluations. It is also expected that when finding $\Delta$-robust CEs with the same $\delta$ value for the test set, a similar level of robustness can be observed. 

We report the $\delta$ value results using both strategies in Figure~\ref{fig:analysis_delta}\JJ{, referred to as $\delta_{inc}$ (incremental retraining) and $\delta_{val}$ (validation set)}. Specifically, for the first scenario, we record and plot $\delta$ values as the average $\infty$-distances \JJ{(in the experiments, we use $p=\infty$)} between $\model{1}$ and five incrementally retrained\footnote{We re-implemented the partial\_fit function in Scikit-learn library: {https://scikit-learn.org/stable/modules/generated/sklearn.neural\_network.MLPClassifier.html\#sklearn.\\neural\_network.MLPClassifier.partial\_fit}} models using increasing sizes of the retraining dataset (a\% of $\dataset_2$). As can be observed, $\delta$ values increase with slight fluctuations as incrementally retraining on more data. The magnitudes are classifier- and dataset-dependent, though in our setting we obtain values from 0 to about 0.3. {The} $\delta_{inc}$ {values} \JJ{are obtained by recording the} $\delta$ values when retraining on $10\%$ of $\dataset_2$, and use them as one robustness target in the next experiments. 

For the second scenario, we use five completely retrained and five leave-one-out retrained classifiers as the new models in Step 1. We use RNCE-\texttt{FF} to find $\Delta$-robust CEs in Step 2. We record the obtained $\delta_{val}$ as another robustness target, and we highlight these values in the same plots in Figure~\ref{fig:analysis_delta}. After matching their magnitudes to the ones obtained for incremental retraining, we identify that these $\delta_{val}$ are much smaller than $\delta_{inc}$ and they correspond to the magnitudes of retraining on only 2\% of $\dataset_2$. The fact that being $\Delta$-robust against a very small magnitude results in 100\% empirical robustness (validity) in a validation set confirms the conservative nature of $\Delta$-robustness, as it guarantees the CE's robustness against all possible model shifts entailed by $\Delta$ (Lemma~\ref{lemma:over-approx}).

\begin{figure*}[ht!]
\centering
  
  \begin{subfigure}{0.48\columnwidth}
    \includegraphics[width=\linewidth]{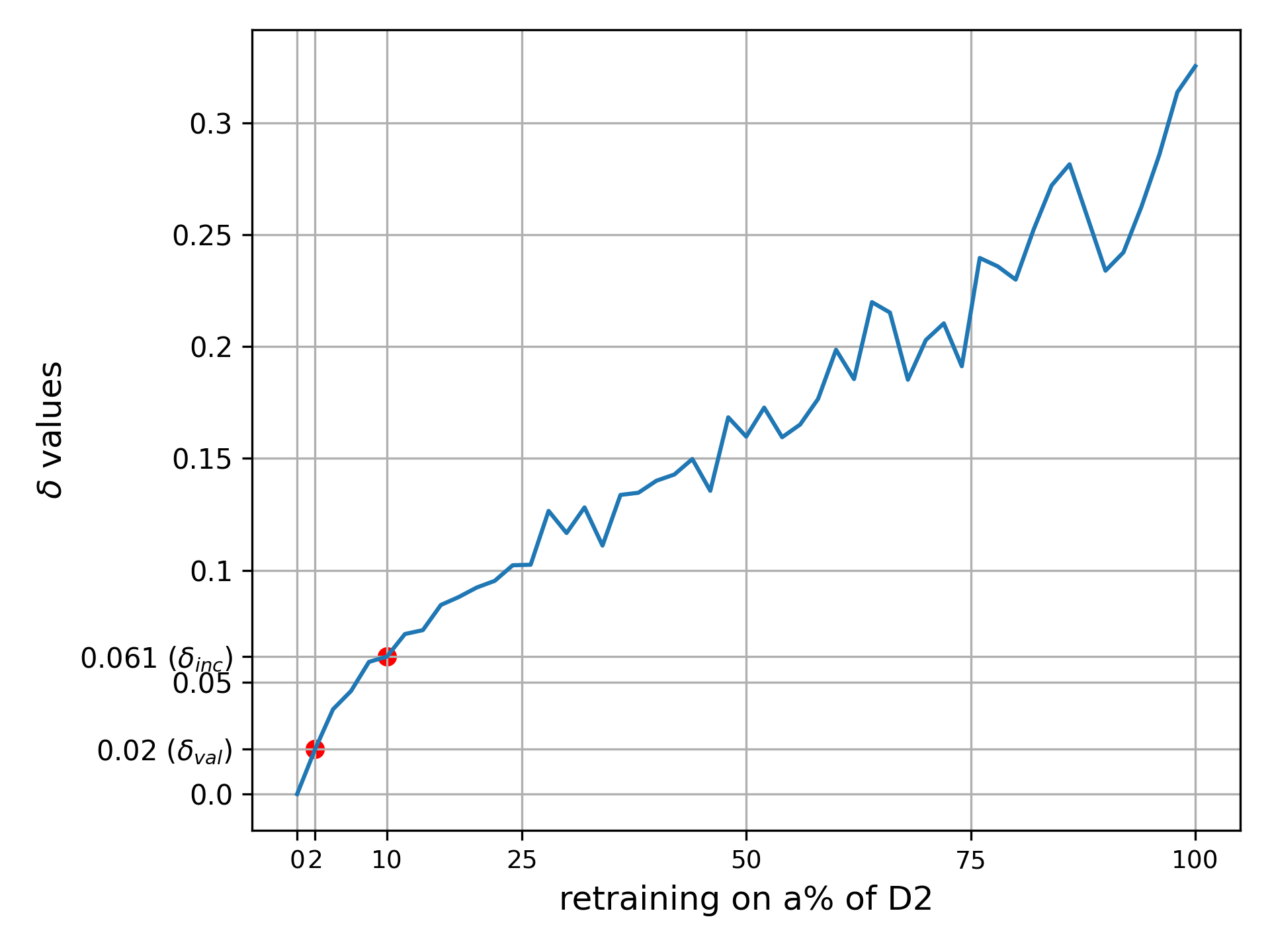}
    \caption{adult dataset}
    \label{fig:adultdelta}
  \end{subfigure}
  \begin{subfigure}{0.48\columnwidth}
    \includegraphics[width=\linewidth]{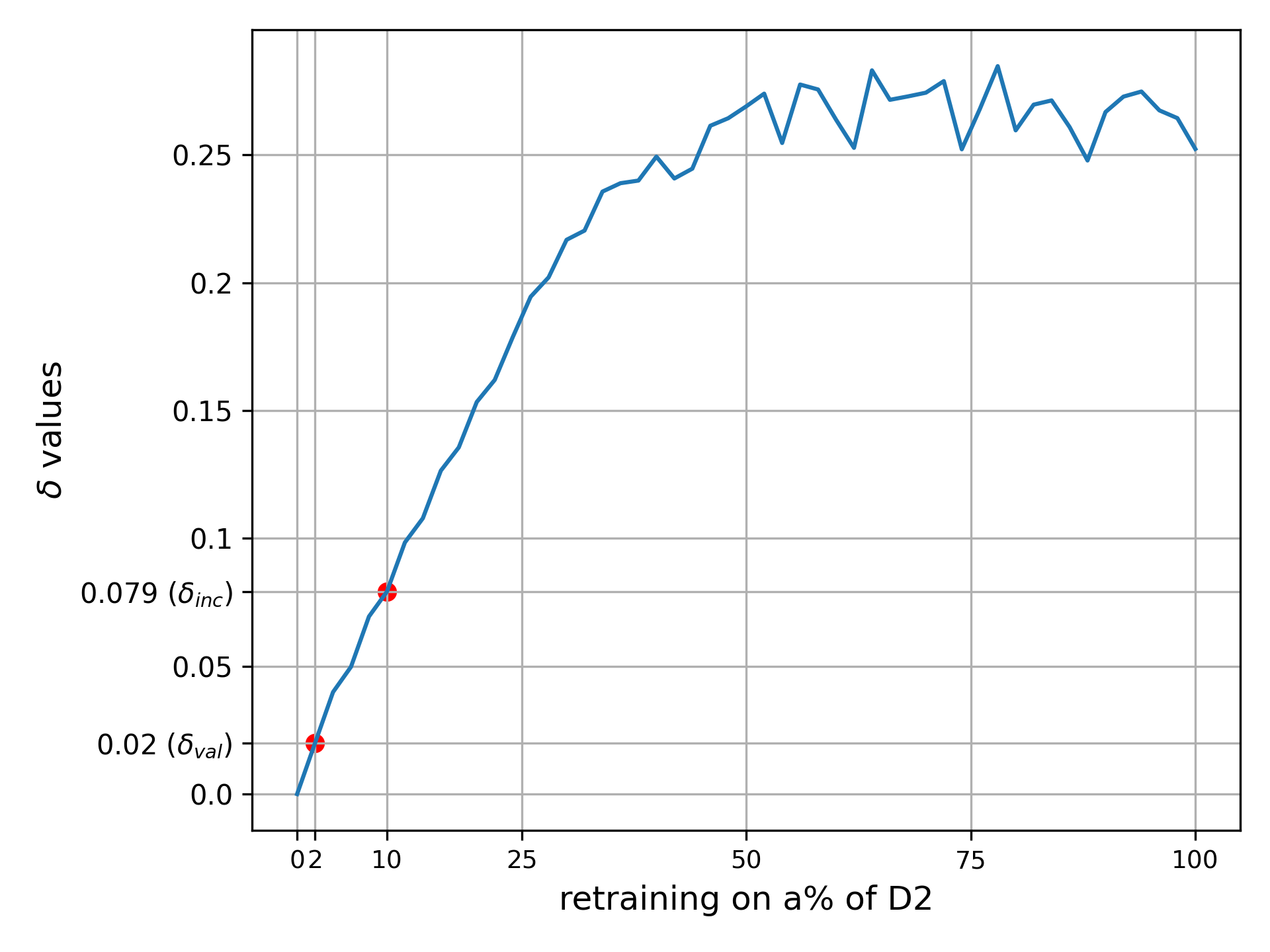}
    \caption{compas dataset}
    \label{fig:compasdelda}
  \end{subfigure}
  \begin{subfigure}{0.48\columnwidth}
    \includegraphics[width=\linewidth]{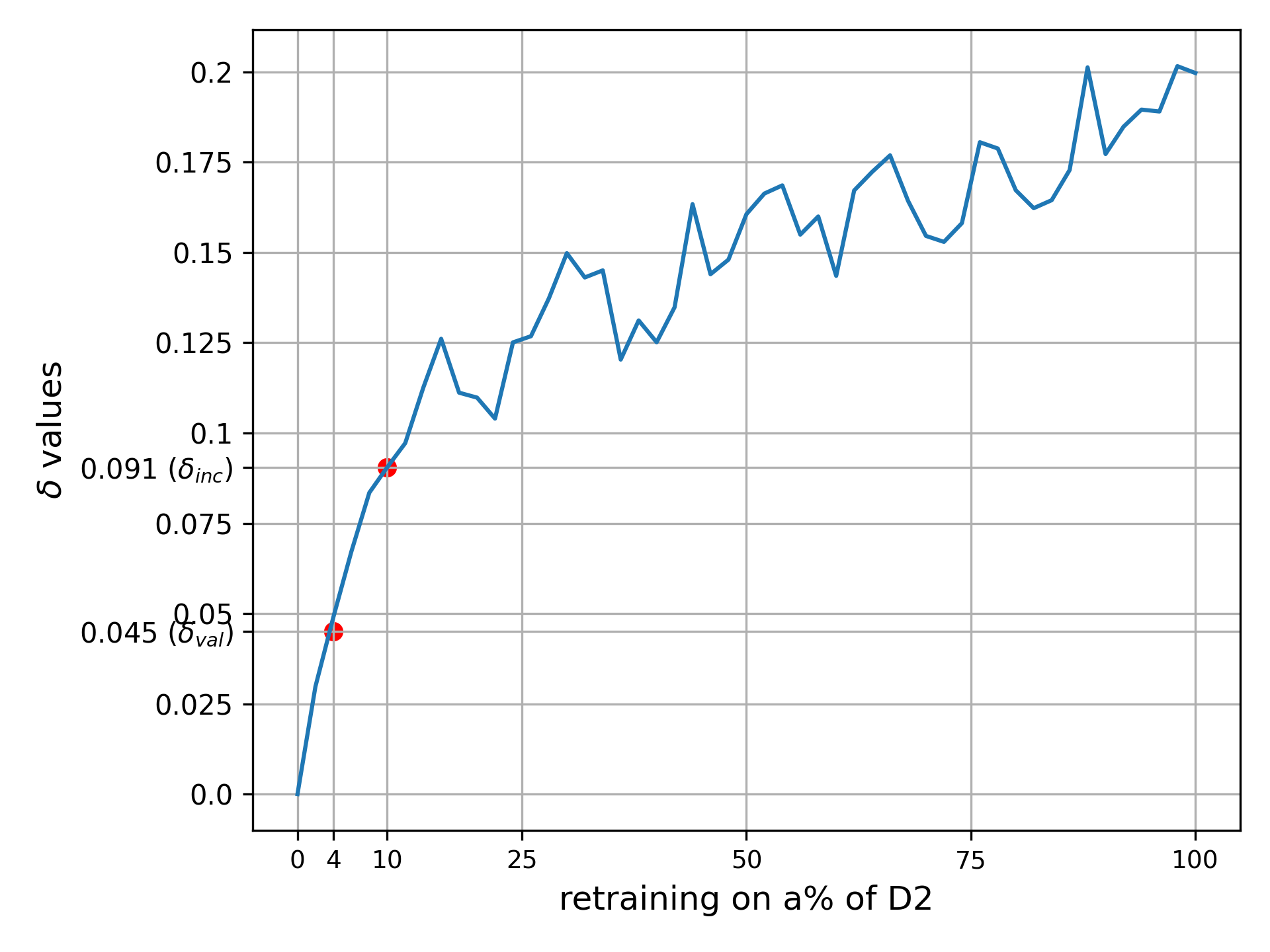}
    \caption{gmc dataset}
    \label{fig:gmcdelta}
  \end{subfigure}
  \begin{subfigure}{0.48\columnwidth}
    \includegraphics[width=\linewidth]{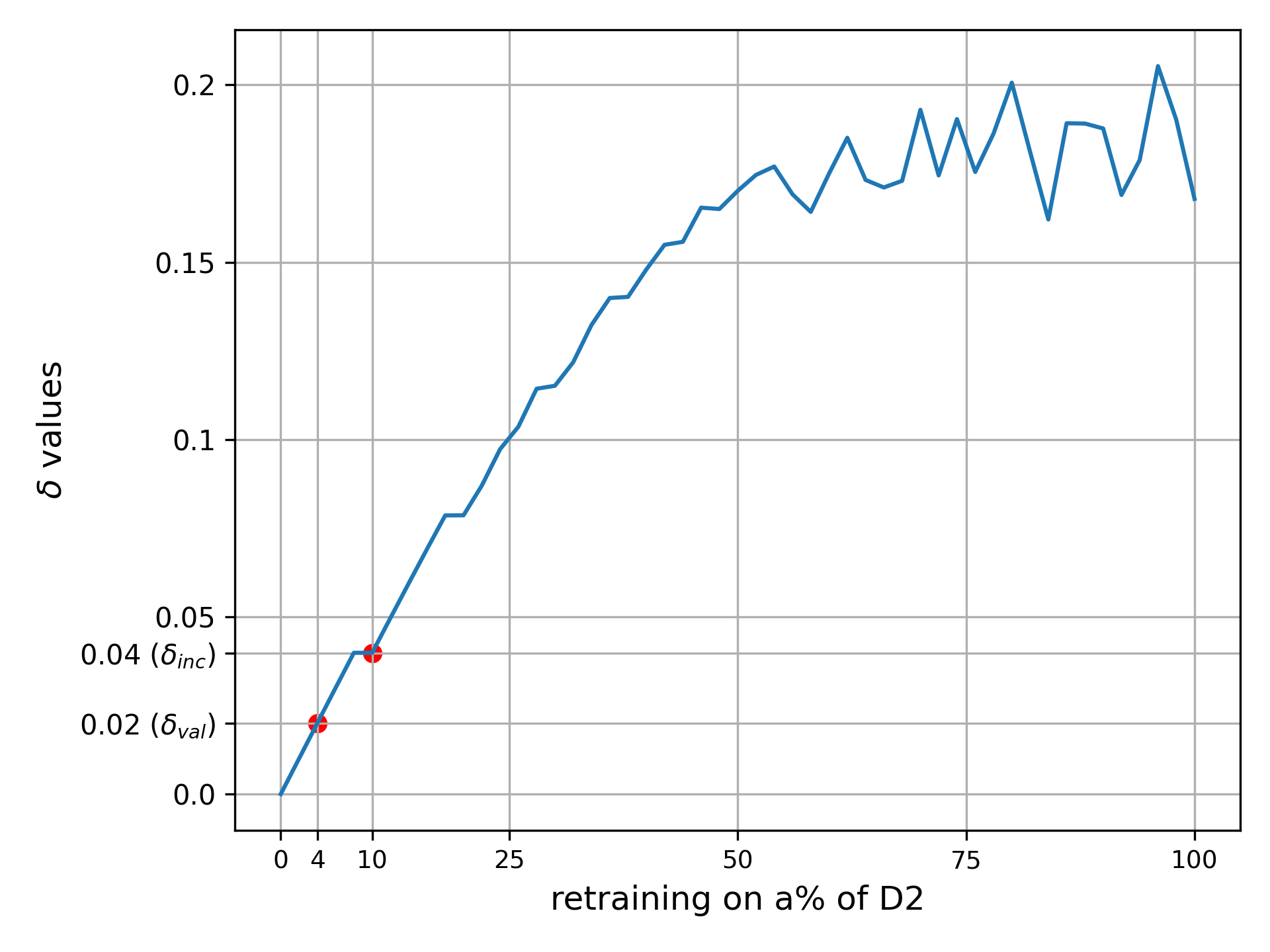}
    \caption{heloc dataset}
    \label{fig:helocdelta}
  \end{subfigure}
  
  \caption{$\delta$ values obtained by retraining on increasing portions of $\dataset_2$. The red dots highlight the $\delta_{val}$ and $\delta_{inc}$ values found by the two strategies described in Section~\ref{ssec:experiments_identify_delta_values}.}
    \label{fig:analysis_delta}
\end{figure*}

\subsection{\JJ{Verifying $\Delta$-robustness}}
\label{ssec:experiments_delta_robustness_plots}

\begin{figure*}[ht!]
\centering
\begin{subfigure}{0.8\columnwidth}
    \includegraphics[width=\linewidth]{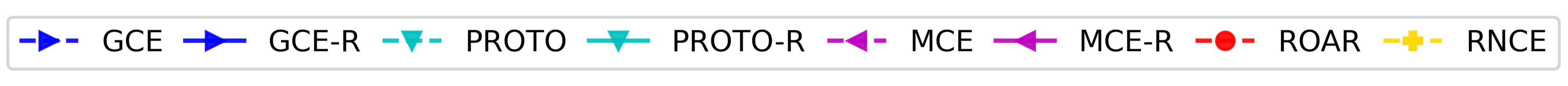}
  \end{subfigure}

%   \begin{subfigure}{0.4\columnwidth}
%     \includegraphics[width=\linewidth]{imgs/delta_validity_adult.png}
%     \caption{SOTA algorithms, adult}
%     \label{fig:diab}
%   \end{subfigure}
%   \begin{subfigure}{0.4\columnwidth}
%     \includegraphics[width=\linewidth]{imgs/delta_validity_compas.png}
%     \caption{SOTA algorithms, compas}
%     \label{fig:no2}
%   \end{subfigure}
% %   

%   \begin{subfigure}{0.4\columnwidth}
%     \includegraphics[width=\linewidth]{imgs/delta_validity_give_me_some_credit.png}
%     \caption{SOTA algorithms, gmc}
%     \label{fig:sba}
%   \end{subfigure}
%   \begin{subfigure}{0.4\columnwidth}
%     \includegraphics[width=\linewidth]{imgs/delta_validity_heloc.png}
%     \caption{SOTA algorithms, heloc}
%     \label{fig:cred}
%   \end{subfigure}
  
%   \vspace{0.5cm}
  
    \begin{subfigure}{0.49\columnwidth}
    \includegraphics[width=\linewidth]{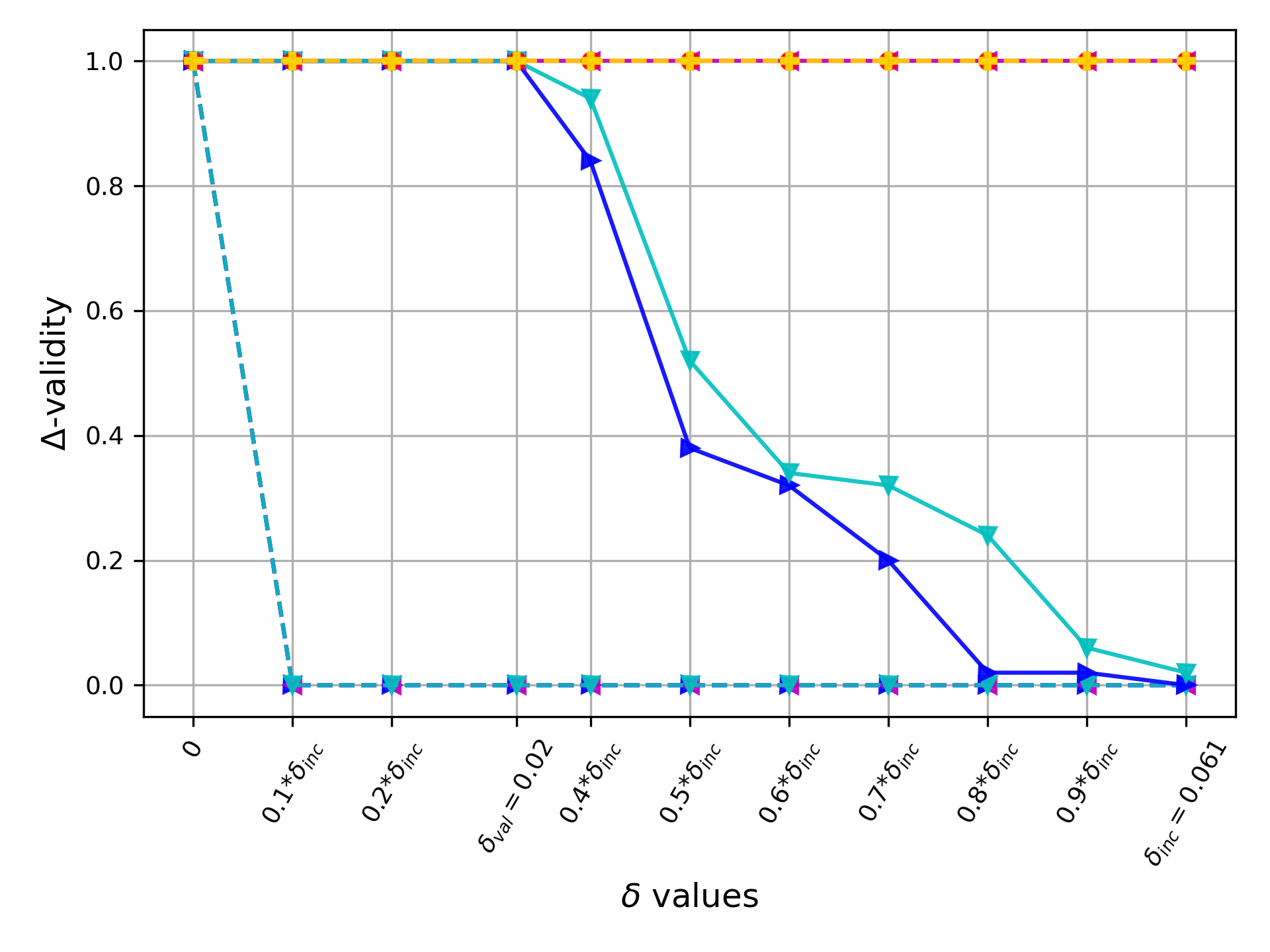}
    \caption{adult}
    \label{fig:delta_robustness_plot_adult}
  \end{subfigure}
  \begin{subfigure}{0.49\columnwidth}
    \includegraphics[width=\linewidth]{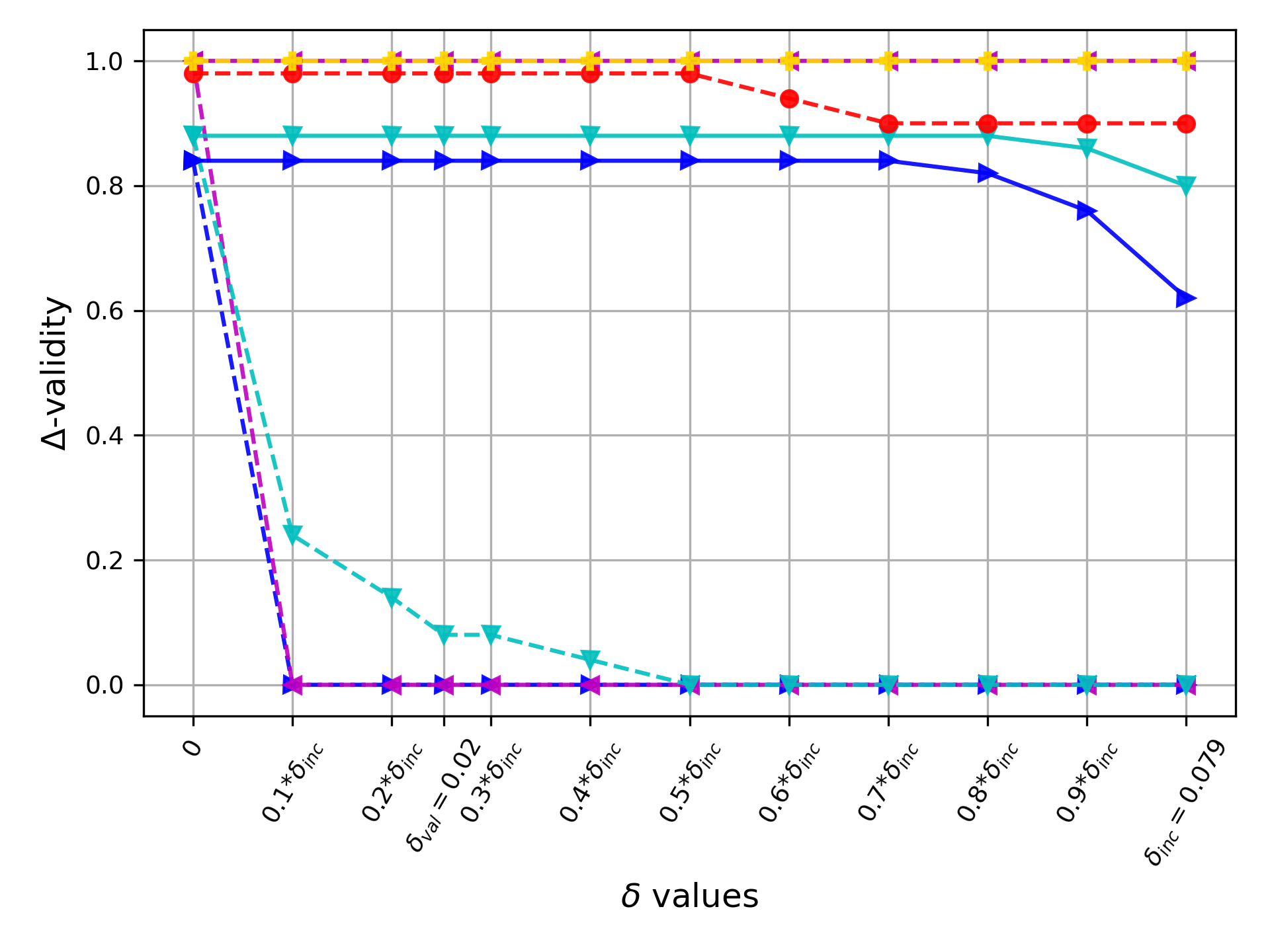}
    \caption{compas}
    \label{fig:delta_robustness_plot_compas}
  \end{subfigure}
  \begin{subfigure}{0.49\columnwidth}
    \includegraphics[width=\linewidth]{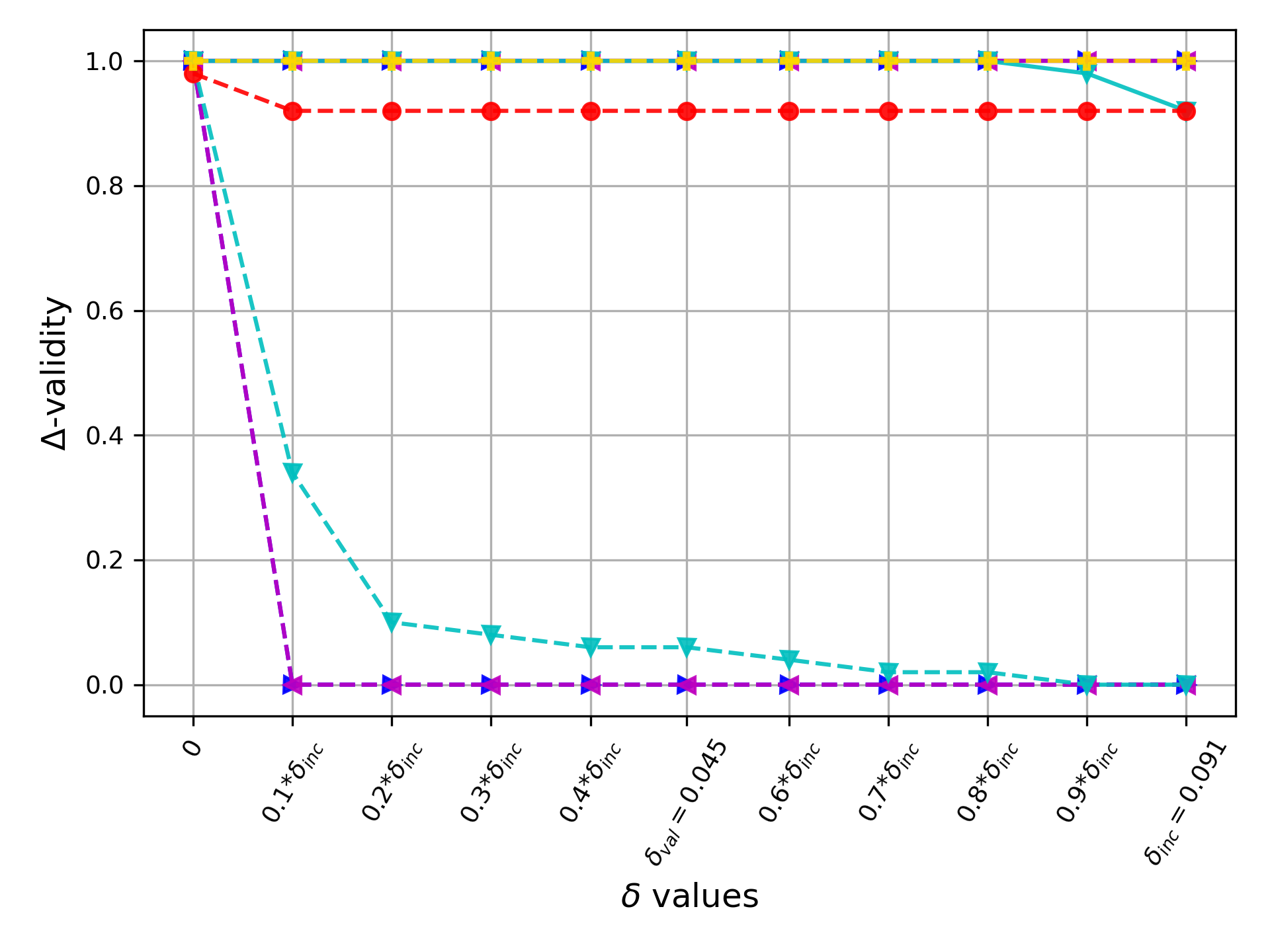}
    \caption{gmc}
    \label{fig:delta_robustness_plot_gmc}
  \end{subfigure}
  \begin{subfigure}{0.49\columnwidth}
    \includegraphics[width=\linewidth]{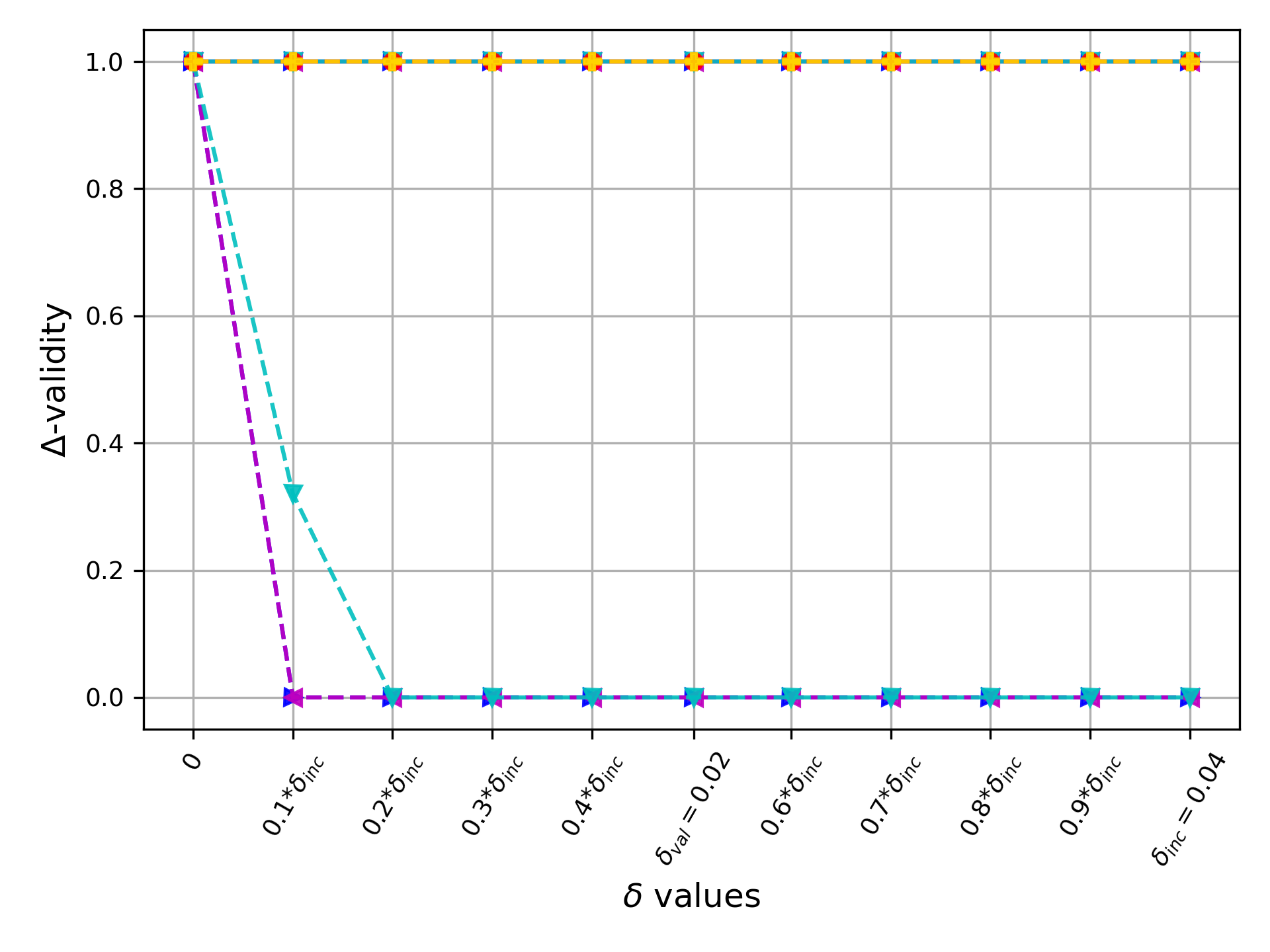}
    \caption{heloc}
    \label{fig:delta_robustness_plot_heloc}
  \end{subfigure}

  \caption{Evaluation of $\name$-validity (against increasing $\delta$ values) of the CEs found by state-of-the-art methods (GCE, PROTO, MCE, ROAR) and by our methods (GCE-R, PROTO-R, MCE-R, RNCE).}
    \label{fig:delta_robustness_plot}
\end{figure*}

In this experiment, we demonstrate how \FL{our MILP procedure} can be used as an evaluation tool to examine the robustness of CEs \FL{generated by} SOTA algorithms. We consider three traditional non-robust baselines, namely a gradient-descent based method \textit{GCE} similar to~\cite{Wachter_17}, a plausible method using gradient descent \textit{Proto}~\cite{van2021interpretable}, and 
a MILP-based method (referred to as MCE) inspired by~\cite{mohammadi2021scaling}. %The first two implement CFX search via gradient descent, while the third uses %mixed-integer linear programming
% MILP, and is thus referred to here as \textit{MILP}. 
We also include \textit{ROAR}~\cite{upadhyay2021towards}, a SOTA framework specifically designed to generate robust CEs, focusing on robustness against the same notion of plausible model changes\footnote{Different{ly} to our previous work \cite{oursaaai23}, we use the implementation of ROAR in CARLA library \cite{pawelczyk2021carla} which allows more comprehensive hyperparameter tuning.}. For our algorithms, using the iterative algorithm (Algorithm~\ref{alg:algo}) proposed in Section~\ref{sec:algorithms}, we devise robustified versions of the non-robust baselines, which we call \textit{GCE-R}, \textit{PROTO-R}, and \textit{MCE-R}. %, to demonstrate the effectiveness of Algorithm~\ref{alg:algo} for improving robustness\todo{this sentence is really convoluted}. 
We also include our RNCE-\texttt{FF} algorithm (Algorithm~\ref{alg:algo1}, \texttt{robustInit=False}, \texttt{optimal=False}) to show the guaranteed robustness results of this method.

For each dataset, we randomly select 50 test inputs for which we use the above baselines to generate CEs. We evaluate their robustness against model shifts of magnitudes up to $\delta_{inc}$ using \textit{$\Delta$-validity}, the percentage of test inputs whose CEs are $\Delta$-robust. For all robust methods, their \JJ{targeted $\Delta$} {values} are instantiated with $\delta_{inc}$.

Figures~\ref{fig:delta_robustness_plot} (a-d) report the results of our analysis \FT{for the four datasets}. As we can observe, all methods generate CEs that tend to be valid for the original model. For the non-robust baselines, $\Delta$-validities soon drop to value 0 as small model shifts are applied, revealing a lack of robustness for these baselines. ROAR exhibits a higher degree of $\name$-robustness, as expected. However, its heuristic nature does not allow to reason about all possible shifts in $\Delta$, which affects the $\name$-robustness of CFXs as $\delta$ grows larger. Also, the fact that it uses a local surrogate model to approximate the behaviour of neural networks could negatively affect the results. For compas and gmc datasets, its $\Delta$-robustness stays lower than 100\%.

For the two gradient-based non-robust baselines, our robustified versions (GCE-R, PROTO-R) successfully improved their robustness against small model shifts with lower $\delta$ values, however, the $\Delta$-robustness tends to drop (drastically for adult dataset) when it increases near $\delta_{inc}$. This is likely due to the vulnerability to local optimum solutions for the gradient-descent algorithms. The MILP-based method MCE-R which gives more exact solutions for the problem always finds robust CEs with 100\% $\Delta$-validity. With proven robustness guarantees, our RNCE algorithm also finds CEs that are 100\% robust.

\subsection{\JJ{Generating $\Delta$-robust CEs}}
\label{ssec:experiments_benchmarking}

\begin{table}[t!]
    \centering
    \resizebox{\columnwidth}{!}{
    %\small
    \begin{tabular}{l|cccc|cc}
    
        \hline 
        \multirow{2}{*}{\textbf{Method}} & 
        \multicolumn{4}{c}{\textbf{Properties}} & 
        \multicolumn{2}{|c}{\textbf{Classifiers}} \\
         & 
        Validity & 
        Proximity & 
        Plausibility & 
        Robustness &
        NN &
        LR \\
        \hline
        GCE~\cite{Wachter_17} 
        &
        $\checkmark$ & 
        $\checkmark$ & 
         & 
         &
        $\checkmark$ & 
         \\
        PROTO~\cite{van2021interpretable} &
        $\checkmark$ & 
        $\checkmark$ & 
        $\checkmark$ & 
        &
        $\checkmark$ & 
         \\
        MCE~\cite{mohammadi2021scaling} &
        $\checkmark$ & 
        $\checkmark$ & 
         & 
         &
        $\checkmark$ & 
         \\
        % AR~\cite{ustun2019actionable} 
        % &
        % $\checkmark$ & 
        % $\checkmark$ & 
        %  & 
        %  &
        %  & 
        % $\checkmark$ \\
        %\textit{CCHVAE} \cite{pawelczyk2020learning} &
        %$\checkmark$ & 
        %$\checkmark$ & 
        %$\checkmark$ & 
        % &
        %$\checkmark$ & 
        % \\
        % \textit{FACE} %\cite{poyiadzi2020face} 
        % &
        % $\checkmark$ & 
        % $\checkmark$ & 
        % $\checkmark$ & 
        %  &
        % $\checkmark$ & 
        %  \\
        % \textit{AR} %\cite{ustun2019actionable} 
        % &
        % $\checkmark$ & 
        % $\checkmark$ & 
        %  & 
        %  &
        %  & 
        % $\checkmark$ \\
        NNCE \cite{NiceNNCE} &
        $\checkmark$ & 
        $\checkmark$ & 
        $\checkmark$  & 
         &
        $\checkmark$ & 
        $\checkmark$ \\
        %
        % \textit{NNCE-OPT} &
        % $\checkmark$ & 
        % $\checkmark$ & 
        % $\checkmark$  & 
        %  &
        % $\checkmark$ & 
        % $\checkmark$ \\
%\textit{SNS}~\cite{blackconsistent} &
        %$\checkmark$ & 
        %$\checkmark$ & 
        % & 
        %$\checkmark$ &
        %$\checkmark$ & 
        % \\

        RBR \cite{nguyen2022robust} 
        &
        $\checkmark$ & 
        $\checkmark$ & 
        $\checkmark$ & 
        $\checkmark$ &
        $\checkmark$ & 
         \\
        ROAR \cite{upadhyay2021towards}
        %upadhyay2021towards} 
        &
        $\checkmark$ & 
        $\checkmark$ & 
         & 
        $\checkmark$ &
        $\checkmark$ & 
        $\checkmark$ \\
        
        ST-CE~\revision{\cite{DBLP:conf/icml/HammanNMMD23}}
        &
        $\checkmark$ & 
        $\checkmark$ & 
        $\checkmark$ & 
        $\checkmark$ &
        $\checkmark$ & \\
        \hline
        GCE-R (ours) %\cite{oursaaai23} 
        &
        $\checkmark$ & 
        $\checkmark$ & 
         & 
        $\checkmark$ &
        $\checkmark$ & 
         \\
        PROTO-R (ours) %\cite{oursaaai23} 
        &
        $\checkmark$ & 
        $\checkmark$ & 
        $\checkmark$ & 
        $\checkmark$ &
        $\checkmark$ & 
         \\
        MCE-R (ours) %\cite{oursaaai23} 
        &
        $\checkmark$ & 
        $\checkmark$ & 
         & 
        $\checkmark$ &
        $\checkmark$ & 
         \\
        RNCE (ours) %\cite{oursaaai23} 
        &
        $\checkmark$ & 
        $\checkmark$ & 
        $\checkmark$ & 
        $\checkmark$ &
        $\checkmark$ & $\checkmark$ 
         \\
        \hline
    \end{tabular}
    }
    \caption{Properties (validity, proximity, plausibility, robustness) addressed by baselines and whether the methods apply to neural network (NN) or logistic regression (LR), indicated by $\checkmark$.}
    \label{tab:baselines}
\end{table}

\begin{table*}[h!]
    \centering
    \resizebox{0.9\columnwidth}{!}{
    \begin{tabular}{cccccc|ccccc}

        \cline{2-11}
        %{\textbf{method}} 
        & 
        \textbf{vr}$\uparrow$&  
        \textbf{v$\Delta_{val}$$\uparrow$} &
        \!\!\!\textbf{v$\Delta_{inc}$$\uparrow$}\!\!\!&
        $\ell_1$$\downarrow$ &
        \!\!\! \textbf{lof$\downarrow$} \!\!\!&
        
        \textbf{vr}$\uparrow$&  
        \textbf{v$\Delta_{val}$$\uparrow$} &
        \!\!\!\textbf{v$\Delta_{inc}$$\uparrow$}\!\!\!&
        $\ell_1$$\downarrow$ &
        \!\!\! \textbf{lof$\downarrow$} \!\!\! \\

        \cline{2-11}
        &
         \multicolumn{5}{c|}{\textbf{adult}} &
         \multicolumn{5}{c}{\textbf{compas}} \\
        
        \hline
        
        \!\!\!\!GCE\!\!\!\! & 
        \red{51\%} &
        \!\!\!\red{0\%}\!\!\! &
        \red{0\%} &
         \!\!\!\red{.016}\!\!\! &
        \blue{1.29} & %%%%%%%%%%%%%%% 
        \red{26.5\%}&
        \!\!\!\red{0\%}\!\!\! &
        \red{0\%} &
         \!\!\!\red{.039}\!\!\!&
        \blue{3.05}\\

        \!\!\!\!PROTO\!\!\!\! & 
        \red{61.1\%} &
        \!\!\!\red{1\%}\!\!\! &
        \red{0\%} &
         \!\!\!\red{.011}\!\!\! &
        \blue{1.44} & 
        \red{50.7\%}&
        \!\!\!\red{6.6\%}\!\!\! &
        \red{0\%} &
         \!\!\!\red{.144}\!\!\!&
        \blue{1.66}\\

        \!\!\!\!MCE\!\!\!\! & 
        \red{48.5\%} &
        \!\!\!\red{0\%}\!\!\! &
        \red{0\%} &
         \!\!\!\red{.009}\!\!\! &
        \blue{1.41} & 
        \red{25.6\%}&
        \!\!\!\red{0\%}\!\!\! &
        \red{0\%} &
         \!\!\!\red{.019}\!\!\!&
        \blue{1.79}\\

        \!\!\!\!NNCE\!\!\!\! & 
        \red{76.1\%} &
        \!\!\!\red{2\%}\!\!\! &
        \blue{2\%} &
         \!\!\!\red{.032}\!\!\! &
        \blue{1.34} & 
        \red{43.3\%}&
        \!\!\!\red{8\%}\!\!\! &
        \blue{0\%} &
         \!\!\!\red{.028}\!\!\! &
        \blue{1.30}\\

        \hline

        \!\!\!\!ROAR\!\!\!\! & 
        \red{100\%} &
        \!\!\!\red{100\%}\!\!\! &
        \red{94.8\%} &
         \!\!\!\red{.877}\!\!\! &
        \blue{12.5} & 
        \red{100\%}&
        \!\!\!{100\%}\!\!\! &
        \red{94.7\%} &
         \!\!\!\red{.388}\!\!\!&
        \blue{8.44}\\

        \!\!\!\!RBR\!\!\!\! & 
        \red{90.1\%} &
        \!\!\!\red{0\%}\!\!\! &
        \blue{0\%} &
         \!\!\!\red{.025}\!\!\! &
        \red{1.33} & 
        \red{98.3\%} &
        \!\!\!\red{38\%}\!\!\! &
        \blue{0\%} &
         \!\!\!\red{.038}\!\!\! &
        \red{1.53}\\
        
        \!\!\!\!ST-CE\!\!\!\! & 
        98.7\% &
        \!\!\!20\%\!\!\! &
        \red{4\%} &
         \!\!\!\blue{.046}\!\!\! &
        \blue{1.27} & 
        99.9\% &
        \!\!\!74\%\!\!\! &
        \red{0\%} &
         \!\!\!\blue{.039}\!\!\! &
        \blue{1.23} \\
        
        \hline
                &
         \multicolumn{5}{c|}{target $\delta_{val}=0.02$} &
         \multicolumn{5}{c}{target $\delta_{val}=0.02$} \\
    
        \!\!\!\!GCE-R\!\!\!\! & 
        100\% &
        \!\!\!100\%\!\!\! &
        \red{0\%} &
         \!\!\!\blue{.048}\!\!\! &
        \blue{1.47} & 
        86.9\%&
        \!\!\!87\%\!\!\! &
        \red{0\%} &
         \!\!\!\blue{.055}\!\!\!&
       \blue{3.04}\\

        \!\!\!\!PROTO-R \!\!\!\! & 
        100\% &
        \!\!\!69\%\!\!\! &
        \red{0\%} &
         \!\!\!\blue{.042}\!\!\! &
        \blue{1.68} & 
        91\%&
        \!\!\!91\%\!\!\! &
        \red{0\%} &
         \!\!\!\blue{.049}\!\!\!&
       \blue{1.74}\\

        \!\!\!\!MCE-R \!\!\!\! & 
        100\% &
        \!\!\!100\%\!\!\! &
        \red{0\%} &
         \!\!\!\blue{.021}\!\!\! &
        \blue{1.65} & 
        99.8\%&
        \!\!\!100\%\!\!\! &
        \red{0\%} &
         \!\!\!\blue{.035}\!\!\!&
       \blue{1.68}\\
        
        \!\!\!\!\ouralgo{}-\texttt{FF} \!\!\!\! & 
        100\% &
        \!\!\!100\%\!\!\! &
        \red{4\%} &
         \!\!\!\red{.057}\!\!\! &
        \blue{1.32} & 
        100\%&
        \!\!\!100\%\!\!\! &
        \red{0\%} &
         \!\!\!\blue{.039}\!\!\!&
       \blue{1.26}\\

        \!\!\!\!\ouralgo{}-\texttt{FT} \!\!\!\! & 
        100\% &
        \!\!\!100\%\!\!\! &
        0\% &
         \!\!\!.049\!\!\! &
        1.28 & 
        100\%&
        \!\!\!100\%\!\!\! &
        0\% &
         \!\!\!.037\!\!\! &
        1.33\\
        
        \hline
           &
         \multicolumn{5}{c|}{target $\delta_{inc}=0.061$} &
         \multicolumn{5}{c}{target $\delta_{inc}=0.079$} \\
         
        \!\!\!\!GCE-R\!\!\!\! & 
        100\% &
        \!\!\!100\%\!\!\! &
        \red{0\%} &
         \!\!\!\blue{.051}\!\!\! &
        \blue{1.65} & 
        87\%&
        \!\!\!87\%\!\!\! &
        \red{67\%} &
         \!\!\!\blue{.109}\!\!\!&
       \blue{3.62}\\

        \!\!\!\!PROTO-R\!\!\!\! & 
        100\% &
        \!\!\!100\%\!\!\! &
        \red{9\%} &
         \!\!\!\blue{.072}\!\!\! &
        \blue{2.36} & 
        91\%&
        \!\!\!91\%\!\!\! &
        \red{84\%} &
         \!\!\!\blue{.108}\!\!\!&
       \blue{1.99}\\

        \!\!\!\!MCE-R \!\!\!\! & 
        100\% &
        \!\!\!100\%\!\!\! &
        \red{100\%} &
         \!\!\!\blue{.051}\!\!\! &
        \blue{2.91} & 
        100\%&
        \!\!\!100\%\!\!\! &
        \red{100\%} &
         \!\!\!\blue{.096}\!\!\!&
       \blue{2.81}\\
        
        \!\!\!\!\ouralgo{}-\texttt{FF} \!\!\!\! & 
        100\% &
        \!\!\!100\%\!\!\! &
        \red{100\%} &
         \!\!\!\red{.122}\!\!\! &
        \blue{2.78} & 
        100\%&
        \!\!\!100\%\!\!\! &
        \red{100\%} &
         \!\!\!\blue{.088}\!\!\!&
       \blue{1.11}\\

        \!\!\!\!\ouralgo{}-\texttt{FT} \!\!\!\! & 
        100\% &
        \!\!\!100\%\!\!\! &
        \red{100\%} &
         \!\!\!\red{.095}\!\!\! &
        \blue{2.70} & 
        100\%&
        \!\!\!100\%\!\!\! &
        100\% &
         \!\!\!.088\!\!\! &
        1.11\\
        
%%%%%%%%%%%%%%%%%%%%%%%%%%%%%%%%%%%%%%%%%%%
        \cline{2-11} &
         \multicolumn{5}{c|}{\textbf{gmc}} &
         \multicolumn{5}{c}{\textbf{heloc}} \\
        
        \hline
        \!\!\!\!GCE\!\!\!\! & 
        \red{75.8\%} &
        \!\!\!\red{0\%}\!\!\! &
        \red{0\%} &
         \!\!\!\red{.022}\!\!\! &
        \blue{1.52} & 
        \red{20.5\%}&
        \!\!\!\red{0\%}\!\!\! &
        \red{0\%} &
         \!\!\!\red{.019}\!\!\!&
        \blue{1.18}\\

        \!\!\!\!PROTO\!\!\!\! & 
        \red{89.1\%} &
        \!\!\!\red{1\%}\!\!\! &
        \red{0\%} &
         \!\!\!\red{.023}\!\!\! &
        \blue{1.36} & 
        \red{39.5\%}&
        \!\!\!\red{0\%}\!\!\! &
        \red{0\%} &
         \!\!\!\red{.024}\!\!\!&
        \blue{1.16}\\

        \!\!\!\!MCE\!\!\!\! & 
        \red{63.3\%} &
        \!\!\!\red{0\%}\!\!\! &
        \red{0\%} &
         \!\!\!\red{.016}\!\!\! &
        \blue{1.40} & 
        \red{22\%}&
        \!\!\!\red{0\%}\!\!\! &
        \red{0\%} &
         \!\!\!\red{.014}\!\!\!&
        \blue{1.40}\\

        \!\!\!\!NNCE\!\!\!\! & 
        \red{88.9\%} &
        \!\!\!\red{22\%}\!\!\! &
        \blue{1\%} &
         \!\!\!\red{.029}\!\!\! &
        \blue{1.23} & 
        \red{35.9\%}&
        \!\!\!\red{0\%}\!\!\! &
        \blue{0\%} &
         \!\!\!\red{.053}\!\!\! &
        \blue{1.05}\\

        \hline

        \!\!\!\!ROAR\!\!\!\! & 
        \red{99.3\%} &
        \!\!\!\red{98\%}\!\!\! &
        \red{98\%} &
         \!\!\!\red{.199}\!\!\! &
        \blue{23.1} & 
        \red{100\%}&
        \!\!\!{100\%}\!\!\! &
        \red{100\%} &
         \!\!\!\red{.454}\!\!\!&
        \blue{6.94}\\

        \!\!\!\!RBR\!\!\!\! & 
        \red{100\%} &
        \!\!\!\red{62\%}\!\!\! &
        \blue{0\%} &
         \!\!\!\red{.034}\!\!\! &
        \red{1.55} & 
        \red{58.7\%}&
        \!\!\!\red{0\%}\!\!\! &
        \blue{0\%} &
         \!\!\!\red{.038}\!\!\!&
        \red{1.08}\\
        
        \!\!\!\!ST-CE\!\!\!\! & 
        100\% &
        \!\!\!92\%\!\!\! &
        \red{6\%} &
         \!\!\!\blue{.041}\!\!\! &
        \blue{1.10} & 
        100\%&
        \!\!\!40\%\!\!\! &
        \red{0\%} &
         \!\!\!\blue{.078}\!\!\!&
       \blue{1.04}\\
        
        \hline
                &
         \multicolumn{5}{c|}{target $\delta_{val}=0.02$} &
         \multicolumn{5}{c}{target $\delta_{val}=0.02$} \\
    
        \!\!\!\!GCE-R\!\!\!\! & 
        100\% &
        \!\!\!100\%\!\!\! &
        \red{0\%} &
         \!\!\!\blue{.032}\!\!\! &
        \blue{1.79} & 
        100\% &
        \!\!\!100\%\!\!\! &
        \red{0\%} &
         \!\!\!\blue{.049}\!\!\!&
       \blue{1.32}\\

        \!\!\!\!PROTO-R \!\!\!\! & 
        100\% &
        \!\!\!100\%\!\!\! &
        \red{0\%} &
         \!\!\!\blue{.036}\!\!\! &
        \blue{1.45} & 
        100\% &
        \!\!\!100\%\!\!\! &
        \red{11\%} &
         \!\!\!\blue{.079}\!\!\!&
       \blue{1.62}\\

        \!\!\!\!MCE-R \!\!\!\! & 
        100\% &
        \!\!\!100\%\!\!\! &
        \red{0\%} &
         \!\!\!\blue{.022}\!\!\! &
        \blue{1.48} & 
        100\% &
        \!\!\!100\%\!\!\! &
        \red{0\%} &
         \!\!\!\blue{.031}\!\!\!&
       \blue{1.94}\\
        
        \!\!\!\!\ouralgo{}-\texttt{FF} \!\!\!\! & 
        100\% &
        \!\!\!100\%\!\!\! &
        \red{7\%} &
         \!\!\!\red{.040}\!\!\! &
        \blue{1.07} & 
        100\% &
        \!\!\!100\%\!\!\! &
        \red{0\%} &
         \!\!\!\blue{.083}\!\!\!&
       \blue{1.04}\\

        \!\!\!\!\ouralgo{}-\texttt{FT} \!\!\!\! & 
        100\% &
        \!\!\!100\%\!\!\! &
        \red{0\%} &
         \!\!\!.035\!\!\! &
        1.36 & 
        100\% &
        \!\!\!100\%\!\!\! &
        \red{0\%} &
         \!\!\!.080\!\!\! &
        1.04\\
        
        \hline
           &
         \multicolumn{5}{c|}{target $\delta_{inc}=0.091$} &
         \multicolumn{5}{c}{target $\delta_{inc}=0.04$} \\
         
        \!\!\!\!GCE-R\!\!\!\! & 
        100\% &
        \!\!\!100\%\!\!\! &
        \red{100\%} &
         \!\!\!\blue{.053}\!\!\! &
        \blue{3.43} & 
        100\% &
        \!\!\!100\%\!\!\! &
        \red{100\%} &
         \!\!\!\blue{.109}\!\!\!&
       \blue{2.07}\\

        \!\!\!\!PROTO-R \!\!\!\! & 
        100\% &
        \!\!\!100\%\!\!\! &
        \red{100\%} &
         \!\!\!\blue{.118}\!\!\! &
        \blue{2.49} & 
        100\% &
        \!\!\!100\%\!\!\! &
        \red{100\%} &
         \!\!\!\blue{.163}\!\!\!&
       \blue{2.41}\\

        \!\!\!\!MCE-R \!\!\!\! & 
        100\% &
        \!\!\!100\%\!\!\! &
        \red{100\%} &
         \!\!\!\blue{.032}\!\!\! &
        \blue{1.86} & 
        100\% &
        \!\!\!100\%\!\!\! &
        \red{100\%} &
         \!\!\!\blue{.049}\!\!\!&
       \blue{3.04}\\
        
        \!\!\!\!\ouralgo{}-\texttt{FF} \!\!\!\! & 
        100\% &
        \!\!\!100\%\!\!\! &
        \red{100\%} &
         \!\!\!\red{.084}\!\!\! &
        \blue{1.22} & 
        100\% &
        \!\!\!100\%\!\!\! &
        \red{100\%} &
         \!\!\!\blue{.150}\!\!\!&
       \blue{1.13}\\

        \!\!\!\!\ouralgo{}-\texttt{FT} \!\!\!\! & 
        100\% &
        \!\!\!100\%\!\!\! &
        \red{100\%} &
         \!\!\!.084\!\!\! &
        1.22 & 
        100\% &
        \!\!\!100\%\!\!\! &
        \red{100\%} &
         \!\!\!.150\!\!\! &
        1.13\\ \hline
    \end{tabular}
    }
    \caption{Quantitative evaluation of the compared CE generation methods on neural networks. The $\uparrow$ ($\downarrow$) following each metric indicates the higher (lower) the value, the better. Methods are separated by horizontal lines, indicating non-robust baselines, robust baselines, and our methods with robustness target $\Delta$ instantiated by $\delta_{val}$ and $\delta_{inc}$, respectively.} % Results worse (better) than our \ouralgo{}-$\texttt{F}\texttt{T}$ method are highlighted in \red{red} (\blue{blue}).}
    \label{tab:results2-NN}
\end{table*}

Next, we rigorously benchmark the performance of our CE generation methods against various robust and non-robust baselines. Apart from the CE methods used in Section~\ref{ssec:experiments_delta_robustness_plots}, we additionally include NNCE \cite{NiceNNCE}, \textit{RBR} \cite{nguyen2022robust}, and \textit{ST-CE} \cite{DBLP:conf/icml/HammanNMMD23}. \revision{We refer {the reader} to Section~\ref{ssec:related-robustcounterfactuals} for their details.}\footnote{\revision{We only included methods which either have open-source implementations or could be re-implemented {without requiring excessive effort}}.} We also instantiate RNCE-\texttt{FT} as one of our methods. The properties of all compared methods are summarised in Table~\ref{tab:baselines}.

For each dataset, we randomly select 20 test points from the test set to generate CEs using each method. We repeat the process five times with different random seeds and report the mean and standard deviation of the results. The CEs are evaluated against three aspects using the standard metrics in the literature. For proximity, we calculate the average \FL{\emph{$\ell_1$ distance}} between the test input and its CE, which captures both closeness of CEs and sparsity of changes \cite{Wachter_17}. For plausibility, we report the average local outlier factor score \emph{lof} \cite{lof} which quantifies the local data density. A {lof} score close to value 1 indicates an inlier; the more it deviates from 1, the less plausible the CE. For robustness, we report validity after retraining \textit{vr}, i.e. the percentage of CEs correctly classified to class 1, under 15 retrained classifiers using respectively complete retraining, leave-one-out retraining, and incremental retraining (with 10\% new data). We use the same $\delta_{val}$ and $\delta_{inc}$ \JJ{obtained from Section~\ref{ssec:experiments_identify_delta_values} to instantiate $\Delta$ with different model shift magnitudes} and report the respective $\Delta$-robustness, termed \textit{v$\Delta_{val}$} and \textit{v$\Delta_{inc}$}.

Table~\ref{tab:results2-NN} reports the mean results for neural network classifiers of the benchmarking study. See \ref{app:full_experiment_results} for the standard deviation results and the evaluations for logistic regression classifiers. Next, we analyse the results by their properties, and for our methods, we first consider the results when targeting $\delta_{val}$.

\paragraph{Our methods produce the most robust CEs} \JJ{Our RNCE algorithm (both configurations) generates} the most robust CEs among the compared methods, showing 100\% vr and 100\% targeted $\Delta$-validity. For the robustified methods using Algorithm~\ref{alg:algo1}, MCE-R is the most robust, finding perfectly $\Delta$-robust and 100\% empirically robust CEs in most experiments. As discussed in Section~\ref{ssec:experiments_delta_robustness_plots}, the limited search space might be the cause of the reduced robustness for GCE-R and PROTO-R in two datasets. As a result, when compared with the robust baselines, RNCE and MCE-R both give better robustness than ROAR, ST-CE, and RBR. All robust methods have better robustness performances than the non-robust baselines, as expected.

\paragraph{Cost-robustness tradeoff} This tradeoff has been discussed in several other studies \cite{upadhyay2021towards,DBLP:conf/iclr/PawelczykDHKL23}, which we have also empirically observed. The non-robust baselines always find the most proximal CEs (lowest $\ell1$ costs). Apart from them, the next best proximity results were obtained by MCE-R (when targeting the smaller $\delta_{val}$) among all the robust methods. Considering that MCE-R also finds near-perfectly $\Delta$-robust CEs, it can be concluded that MCE-R \JJ{effectively} balances the cost-robustness tradeoff. RBR also finds CEs with low costs, but their method is not as robust. Similar remarks can be made for ST-CE as this method moves more towards the robustness end of the tradeoff. Our methods GCE-R, PROTO-R, RNCE demonstrate similar $\ell1$ costs. Setting $\texttt{optimal=True}$ in our RNCE algorithm slightly improves the proximity, as can be seen when comparing RNCE-$\texttt{FF}$ and RNCE-$\texttt{FT}$. ROAR \JJ{results in CEs with high $\ell1$ cost, and the method has been identified as being }overly costly \cite{DBLP:conf/iclr/NguyenBN23} due to their gradient-based robust optimisation procedure.

\paragraph{Plausibility results} ST-CE \FL{has} the best lof scores among all the methods, \FL{followed by RNCE which yields comparable results}. This is because these two methods select CEs from in-manifold dataset points and are thus unlikely to return outliers. By inherently addressing data density estimation, RBR also finds plausible CEs. The plausibility performances of our three robustified methods and ROAR are not as good due to less regulated search spaces or lack of plausibility constraints, with ROAR having the worst lof results.

\paragraph{The effects of robustification} For Algorithm~\ref{alg:algo}, robustifying GCE, PROTO, and MCE resulted in improved empirical and $\Delta-$robustness, but this negatively affected the proximity and plausibility results. Targeting larger-magnitude plausible model shifts (instantiated with $\delta_{inc}$) pushes these tradeoffs further. For RNCE, similar trends can be identified when compared with NNCE, but in most cases, the lof score improves. When targeting $\delta_{inc}$ instead of $\delta_{val}$, the cost also increases together with robustness, but plausibility stays comparable.

\paragraph{Concluding remarks} From the \JJ{analysis above}, we can conclude that MCE-R achieved the best robustness-cost tradeoff, finding the most proximal CEs among the robust baselines while showing near-perfect robustness results, outperforming all robust baselines. RNCE, on the other hand, finds CEs with even stronger robustness guarantees and great plausibility, at slightly higher costs. However, less costly methods than RNCE are not as robust.

% In summary, our \ouralgo\ methods provide the most robust CEs with formal robustness guarantees.  They also demonstrate the most balanced results in the presence of proximity-plausibility trade-off as our CEs are more plausible than any robust baseline which finds closer CEs, and we show better proximity than any of the robust and plausible baselines.

\subsection{\revision{Computation time analysis}}
\label{ssec:comptime}

\revision{In this section, we discuss the computation time required by each method to obtain the results in Section~\ref{ssec:experiments_benchmarking}. The average computation times for generating CEs for 20 test points are presented in Table~\ref{tab:runtime-all}.}

\revision{While model complexity impacts all methods, the runtimes of gradient- and MILP-based optimisation methods {that are} not wrapped in an iterative approach (GCE, PROTO, MCE, RBR) are mostly sensitive to the number of attributes in the dataset, and those iteratively computed for robustness (ROAR, GCE-R, PROTO-R, MCE-R) are also easily affected by their hyperparameters concerning robustness. The methods %requiring 
{that require} traversing (part of) the training dataset (NNCE, RNCE, ST-CE) are sensitive to the number of data points in each dataset. }
\begin{table*}[ht!]
    \centering
    \resizebox{0.85\columnwidth}{!}{
    \begin{tabular}{cccccc}

        \cline{1-6}
        %{\textbf{method}} 
        \textbf{Model}& \textbf{Method} & 
        \textbf{adult}&  
        \textbf{compas} &
        \textbf{gmc}&
        \textbf{heloc} \\
        
        \hline

         \multirow{6}{4.5em}{Logistic Regression} 

         &\!\!\! NNCE \!\!\! 
         & 0.039 & 0.024 & 0.636 & 0.003 \\

         &\!\!\! ROAR 
         & 11.70 & 11.49 & 4.958 & 13.71 \\
        \cline{2-6}
         &\!\!\! RNCE-\texttt{FF} $\delta_{val}$ \!\!\! 
         & 0.410 & 0.195 & 0.755 & 27.71 \\

         &\!\!\! RNCE-\texttt{FF} $\delta_{val}$ \!\!\! 
         & 1.554 & 1.035 & 1.722 & 29.75 \\
        \cline{2-6} 
         &\!\!\! RNCE-\texttt{FF} $\delta_{inc}$ \!\!\! 
         & 0.627 & 0.948 & 0.754 & 1.534 \\

         &\!\!\! RNCE-\texttt{FF} $\delta_{inc}$ \!\!\! 
         & 1.779 & 2.061 & 1.681 & 3.522 \\
         \hline
         
         \multirow{17}{5em}{Neural Network} %\\(2 hidden layers)}
         
         &\!\!\! GCE \!\!\! 
         & 27.23 & 29.06 & 25.15 & 25.08 \\

         &\!\!\! PROTO \!\!\! 
         & 604.3 & 555.8 & 542.2 & 607.6 \\

         &\!\!\! MCE \!\!\! 
         & 0.934 & 0.340 & 0.163 & 1.221 \\

         &\!\!\! NNCE \!\!\! 
         & 0.095 & 0.018 & 0.514 & 0.039 \\

         &\!\!\! ROAR \!\!\! 
         & 20.03 & 2.535 & 1.740 & 2.007 \\

         &\!\!\! RBR \!\!\! 
         & 187.2 & 118.4 & 112.2 & 110.5 \\

         &\!\!\! ST-CE \!\!\! 
         & 0.579 & 1.29 & 0.347 & 1.419 \\

        \cline{2-6}
        
         &\!\!\! GCE-R $\delta_{val}$ \!\!\! 
         & 101.7 & 56.97 & 86.32 & 67.33 \\

         &\!\!\! PROTO-R $\delta_{val}$ \!\!\! 
         & 1602 & 756.9 & 690.5 & 965.2 \\

         &\!\!\! MCE-R $\delta_{val}$\!\!\! 
         & 18.77 & 2.364 & 1.281 & 13.05 \\

         &\!\!\! RNCE-\texttt{FF} $\delta_{val}$ \!\!\! 
         & 15.64 & 4.125 & 3.447 & 32.76 \\

         &\!\!\! RNCE-\texttt{FT} $\delta_{val}$ \!\!\!  
         & 27.96 & 6.744 & 5.307 & 43.50 \\

         \cline{2-6} 

         &\!\!\! GCE-R $\delta_{inc}$ \!\!\! 
         & 155.2 & 117.08 & 116.01 & 120.9 \\

         &\!\!\! PROTO-R $\delta_{inc}$ \!\!\! 
         & 1920 & 1824 & 1625 & 1469 \\

         &\!\!\! MCE-R $\delta_{inc}$\!\!\! 
         & 37.45 & 9.152 & 2.771 & 19.43 \\

         &\!\!\! RNCE-\texttt{FF} $\delta_{inc}$ \!\!\! 
         & 218.3 & 48.07 & 163.2 & 670.6 \\

         &\!\!\! RNCE-\texttt{FT} $\delta_{inc}$ \!\!\!  
         & 229.9 & 51.51 & 165.5 & 681.6 \\
         \hline

         % \multirow{10}{5em}{Neural Network \\(4 hidden layers)}
         
         % &\!\!\! GCE \!\!\! 
         % & 0 & 0 & 0 & 0 \\

%%%%%%%%%%%%%%%%%%%%%%%%%%%%%%%%%%%%%%%%%%%
    \end{tabular}
    }
    \caption{\revision{Computation time (in seconds) of the compared methods on logistic regression models and neural networks. The adult dataset has 48832 data points with 13 attributes, the compas dataset has 6172 data points with 7 attributes, the gmc dataset has 115527 data points with 10 attributes, the heloc dataset has 9871 data points and 21 attributes.}} 
    \label{tab:runtime-all}
\end{table*}

\revision{%Our three robustified methods (GCE-R, PROTO-R, MCE-R) largely inherit the runtime of their respective base method. For RNCE, see Section~\ref{ssec:rnce} and Appendix~\ref{app:a_robustinit_computation_time} for detailed discussions of factors affecting its runtime. 
MCE-R is the fastest among the methods we devised, showing better runtime than ROAR and RBR when targeting {a} smaller validation $\delta$, and producing close CEs with robustness guarantees. On the other hand, our PROTO-R method is the slowest, while {it} also {fails to give} the best CEs in terms of the desired properties (Table~\ref{tab:results2-NN}). Similar remarks can be made for GCE-R, which has comparable runtime as RBR. These three robustified methods (GCE-R, PROTO-R, MCE-R) largely inherit the runtime of their respective base method (GCE, PROTO, MCE). For RNCE, the number of attributes and model complexity largely impact the time required to compute the solution for a single MILP program, while the number of dataset points and targeted $\delta$ value jointly impact the number of MILP programs that need to be checked. Therefore, the RNCE runtime varies to a large extent in our experiments. On logistic regression classifiers, our approach is faster than the other robust baseline, ROAR. On neural networks, for {a} smaller validation $\delta$, the runtimes are in the same order of magnitude as ROAR and are much faster than RBR, but they deteriorate when switching to {a} larger $\delta_{inc}$. }

\revision{We conclude that our best performing methods, MCE-R and RNCE, when targeting the $\delta$ values found by the validation set method, are able to find better-quality and more robust CEs than the robust baselines while demonstrating comparable runtimes. When trying to achieve robustness guarantees for higher $\delta$ values, worse runtimes can be observed. This result is expected, given that proving robustness with respect to model changes has been shown to be an NP-complete problem~\cite{marzari2024}.}
%\todo{FT: as expected? NP? the conclusion seems overall negative else...also, could we put our methods in bold in the table?}

\subsection{Multi-Class Classification}
\label{ssec:experiments_multi_class}

Next, we demonstrate the applicability of our RNCE method to multi-class classification settings, which is not supported by any robust CE generation baselines listed in Table~\ref{tab:baselines}. We use the Iris dataset \cite{anderson1936species,fisher1936use}, a small-scale dataset for three-class classification tasks, and the California Housing dataset \cite{pace1997sparse}, in which we transform the regression labels into three classes. \JJ{Both datasets are available in sklearn \cite{scikit-learn}}. We trained \FL{two} neural network models with cross validation accuracy of 0.93 and 0.70, respectively. Again, we randomly select (five times with different random seeds) 20 test instances from the test sets which are classified to class 0, and we specify a desired label 2 for CEs. Following \FL{the same} experimental \FL{protocol described} in Section~\ref{ssec:experiments_benchmarking}, we use the validation set strategy to identify $\delta$ values as our targeted plausible model changes magnitude, and we report the same evaluation metrics for NNCE and RNCE.

\begin{table*}[ht!]
    \centering
    \resizebox{0.95\columnwidth}{!}{
    \begin{tabular}{cccccc}

        \cline{1-6}
        %{\textbf{method}} 
        \textbf{Dataset}& \textbf{Method} & 
        \textbf{vr}$\uparrow$&  
        \textbf{v$\Delta_{val}$$\uparrow$} &
        $\ell_1$$\downarrow$ &
        \textbf{lof$\downarrow$}  \\

        \cline{1-6}
         \multirow{3}{5.2em}{\textbf{iris\\ $\delta_{val}=0.015$}} & \!\!\!\!NNCE\!\!\!\! & 
        \red{75\%$\pm.004$} &
        \!\!\!\red{0\%$\pm.0$}\!\!\! &
         \!\!\!\red{.393$\pm.002$}\!\!\! &
        \blue{1.48$\pm.026$} \\

        & \!\!\!\!\ouralgo{}-\texttt{FF} \!\!\!\! & 
        100\%$\pm.0$ &
        \!\!\!100\%$\pm.0$\!\!\! &
         \!\!\!\red{.438$\pm.003$}\!\!\! &
        \blue{1.50$\pm.002$} \\

        & \!\!\!\!\ouralgo{}-\texttt{FT} \!\!\!\! & 
        100\%$\pm.0$ &
        \!\!\!100\%$\pm.0$\!\!\! &
         \!\!\!\red{.438$\pm.003$}\!\!\! &
        \blue{1.50$\pm.002$}\\ \hline

         \multirow{3}{6em}{\textbf{housing\\ $\delta_{val}=0.040$}} & \!\!\!\!NNCE\!\!\!\! & 
        \red{18.1\%$\pm.035$}&
        \!\!\!\red{0\%$\pm.0$}\!\!\! &
         \!\!\!\red{.032$\pm.003$}\!\!\! &
        \blue{1.10$\pm.020$} \\

        & \!\!\!\!\ouralgo{}-\texttt{FF} \!\!\!\! & 
        100\%$\pm.0$ &
        \!\!\!100\%$\pm.0$\!\!\! &
         \!\!\!\red{.053$\pm.004$}\!\!\! &
        \blue{1.33$\pm.037$}\\

        & \!\!\!\!\ouralgo{}-\texttt{FT} \!\!\!\! & 
        100\%$\pm.0$ &
        \!\!\!100\%$\pm.0$\!\!\! &
         \!\!\!\red{.052$\pm.004$}\!\!\! &
        \blue{1.31$\pm.035$}\\ \hline
        
%%%%%%%%%%%%%%%%%%%%%%%%%%%%%%%%%%%%%%%%%%%
    \end{tabular}
    }
    \caption{Quantitative evaluation of the compared methods on neural networks in multi-class classification tasks. The evaluation metrics are the same as in Table~\ref{tab:results2-NN}.} %The $\uparrow$ ($\downarrow$) following each metric indicates the higher (lower) the value, the better. Results worse (better) than our \ouralgo{}-$\texttt{F}\texttt{T}$ method are highlighted in \red{red} (\blue{blue}).}
    \label{tab:results-multiclass}
\end{table*}

Table~\ref{tab:results-multiclass} presents \FL{the results we obtained}. Similarly to the results for binary classification, the CEs computed by the NNCE method are not robust against realistically retrained models (indicated by the low \textbf{vr}). Measured by the $\Delta$-validity, they also fail to satisfy the $\Delta$-robustness tests for multi-class classifications. With a slight tradeoff with \FL{$\ell_1$} costs and lof scores, our two configurations of RNCE are both able to find perfectly robust CEs in this study. This demonstrates that the $\Delta$-robustness notion can be used to evaluate provable robustness guarantees for CEs, and our proposed algorithms succeed at finding more robust CEs in the multi-class classification setting.

\section{Discussion and Future Work}
\label{sec:conclusion}
Despite the recent advances in achieving robustness against model changes for CEs, state-of-the-art lack rigorous robustness guarantees on the CEs they produce. By introducing a formal robustness notion, $\Delta$-robustness, and \JJ{a novel interval abstraction technique}, we provided the first method in the literature to obtain CEs with certified robustness guarantees. We showed how such $\Delta$-robustness can be practically tested in binary and multi-class classification settings by solving optimisation problems via MILP. Furthermore, we proposed two algorithms to generate CEs that are provably $\Delta$-robust. To demonstrate how our methods can be used in practice, two strategies for identifying the appropriate hyperparameters in our method have been investigated, linking to three model retraining strategies. We then presented an extensive empirical evaluation involving eleven algorithms for generating CEs, including \revision{seven} algorithms specifically designed to generate robust explanations. Our results show that our MCE-R algorithm finds CEs with the lowest costs along with the best robustness results, and our \ouralgo{} outperforms existing approaches and is able to generate CEs that are both provably robust and plausible, achieving a balance in the robustness-cost and plausibility-cost tradeoffs. \revision{We also compared the runtimes and observed that our best-performing methods have comparable runtimes with the robust baselines when targeting small $\delta$ values.} We see these outcomes as important contributions towards complementing existing formal approaches for XAI and making them applicable in practice. 

\revision{Despite the positive results obtained here, we also identified a number of limitations of our approach. Firstly, white-box access to both the training dataset and the classifiers is required. As a result, our approach is {primarily aimed at expert users, e.g.} model developers. Secondly, our method only works for parametric machine learning classifiers whose forward pass can be precisely represented by a MILP program. We have considered two examples in this paper, but other models could be considered. Thirdly, the computation time for our method, when targeting high $\delta$ values on complex models and large datasets, can be high as {that for} providing exact robustness guarantees for CEs under model changes{, which} is an NP-complete problem~\cite{marzari2024}. Nevertheless, in this work we demonstrated that our method leads to better-quality CEs than existing methods.}

% \todo{can we say with comparable runtimes?}

This work opens up several promising avenues for future work. One such direction is to investigate relaxed forms of robustness, e.g., when the output intervals in $\abst{\theta}{\Delta}$ for different classes overlap, with similar interval abstraction techniques. In this work, we considered the deterministic robustness guarantees, aiming at ensuring CEs' validity against even worst-case model parameter perturbations encoded by $\Delta$. However, we observed that this notion is conservative in that the worst-case perturbations might not always occur in realistic model retraining. One potential way to overcome this is to determine the hyperparameter $\delta$ which controls the magnitude of model changes in $\Delta$ using a validation set, as illustrated in Section~\ref{ssec:experiments_identify_delta_values}. However, that requires additional computational efforts. Therefore, probabilistic relaxations of our $\Delta$-robustness notion allowing more fine-grained analyses would be valuable. %This could potentially address another limitation of our approach of high computational costs.

Another possibility is to investigate the $\Delta$-robustness of CEs in a causal setting. CEs which conform to some structural causal models are usually within data distribution, and they reflect the true causes of predictions, making causality a desirable property \cite{DBLP:conf/nips/KarimiKSV20,DBLP:conf/fat/KarimiSV21}. It has also been found that plausible CEs are likely to be more robust against model shifts \cite{pawelczyk2020counterfactual}. Therefore, there could be intrinsic links between causality and robustness. \citet{Dominguez-Olmedo_22} have proposed a framework to compute CEs robust to noisy executions under a causal setting. It would be \JJ{interesting} to also investigate the interplay of $\Delta$-robustness and causality to facilitate the development of higher-quality CEs.

Finally, it would be possible to apply similar interval abstraction techniques to provide robustness guarantees for other forms of CE robustness. In our $\Delta$-robustness tests, a fixed-value CE is fed into the interval-valued abstraction of a classification model. Intuitively, if \JJ{instead an interval-valued CE is passed as input into }a fixed-valued classification model, it would be similar to reasoning about the robustness of CEs against noisy executions. Going further, \JJ{it would be interesting to} investigate training techniques incorporating $\Delta$-robustness notions to train models which can produce robust CEs by plain gradient-based optimisation methods.

\section*{Acknowledgements}
Jiang, Rago and Toni were partially funded by J.P. Morgan and by the Royal Academy of Engineering under the Research Chairs and Senior Research Fellowships scheme. 
Leofante was funded by Imperial College London through under the Imperial College Research Fellowship scheme. 
Rago and Toni were partially funded by the European Research Council (ERC) under the European Union’s Horizon 2020 research and innovation programme (grant agreement No. 101020934). 
Any views or opinions expressed herein are solely those of the authors listed.
 
% \newpage
% \clearpage
\bibliographystyle{style} 
\bibliography{refs}

\newpage
\clearpage

\appendix
\section{Computation Time of RNCE and the Impact of \texttt{robustInit}}
\label{app:a_robustinit_computation_time}
Since querying for the nearest neighbour from a k-d tree is in logarithm time complexity wrt the tree size, the computation time bottleneck of RNCE is the total time for the more complex $\Delta$-robustness tests. Therefore, the computation time of \ouralgo{} mainly depends on two factors, the time required for each $\Delta$-robustness test, and the number of $\Delta$-robustness tests performed in total. For the former factor, each $\Delta$-robustness test is a MILP program whose problem size, affected by the number of attributes in the dataset and the architecture of the classifier in our setting, determines its computation time. The latter factor, however, is directly controlled by the algorithm parameter, \texttt{robustInit}. When  \texttt{robustInit}$=$\texttt{F}, the number of $\Delta$-robustness tests is the product of the number of inputs for which CEs are generated and the average number of times querying for the next nearest neighbour (Alg.3, lines 4-6) before reaching a $\Delta$-robust CE, which is further affected by the value of $\delta$ constructing $\Delta$. When \texttt{robustInit}$=$\texttt{T}, the number of $\Delta$-robustness tests is upper-bounded by the dataset size. 

We empirically compare the computation time when \texttt{robustInit}$=$\texttt{T} with \texttt{robustInit}$=$\texttt{F}. Specifically, since the line search controlled by the parameter \texttt{optimal} (Alg.~\ref{alg:algo3}, lines 7-10) is independent of the impact of \texttt{robustInit}, we report results for \texttt{optimal}$=$\texttt{F}, i.e., \ouralgo{}-\texttt{FF} and \ouralgo{}-\texttt{TF}, under two different settings, in Tables~\ref{tab:timeforrobustinit100} and \ref{tab:timeforrobustinitall}. 

\begin{table}[h]
    \centering
    \resizebox{1\columnwidth}{!}{
    \begin{tabular}{c|cc|cc|cc|cc}
        \hline 
        \multirow{2}{*}{\textbf{configuration}} & 
        \multicolumn{2}{c}{\textbf{adult}} &
        \multicolumn{2}{|c}{\textbf{compas}} & 
        \multicolumn{2}{|c}{\textbf{gmc}} & 
        \multicolumn{2}{|c}{\textbf{heloc}}\\ 
         & 
        NN & 
        LR & 
        NN & 
        LR &
        NN & 
        LR & 
        NN & 
        LR\\
        \hline

        \ouralgo{}-$\texttt{F}\texttt{F}$ &
        7.69 & 
        0.47 & 
        11.47 & 
        4.34 &
        4.33 & 
        0.41 &
        55.83 & 
        0.49\\
        \ouralgo{}-$\texttt{T}\texttt{F}$ incl. time for obtaining $\tree$ &
        59.55 & 
        3.84 & 
        20.21 & 
        1.77 &
        268.35 & 
        43.20 &
        44.23 & 
        3.17\\
        \ouralgo{}-$\texttt{T}\texttt{F}$ excl. time for obtaining $\tree$ &
        0.002 & 
        0.002 & 
        0.001 & 
        0.001 &
        0.008 & 
        0.007 &
        0.001 & 
        0.002\\
        
        \hline
    \end{tabular}
    }
    \caption{\ouralgo{} computation time {(seconds)} of generating robust CEs for 100 randomly selected inputs using neural networks (NN) or logistic regressions (LR), when configured with different \texttt{robustInit} {settings}.}
    \label{tab:timeforrobustinit100}
\end{table}

\begin{table}[h]
    \centering
    \resizebox{1\columnwidth}{!}{
    \begin{tabular}{c|cc|cc|cc|cc}
        \hline 
        \multirow{2}{*}{Configuration} & 
        \multicolumn{2}{c}{\textbf{adult} (15500)} &
        \multicolumn{2}{|c}{\textbf{compas} (200)} & 
        \multicolumn{2}{|c}{\textbf{gmc} (700)} & 
        \multicolumn{2}{|c}{\textbf{heloc} (2000)}\\ 
         & 
        NN & 
        LR & 
        NN & 
        LR &
        NN & 
        LR & 
        NN & 
        LR\\
        \hline
        \ouralgo{}-$\texttt{F}\texttt{F}$ &
        1306 & 
        59.55 & 
        27.16 & 
        5.63 &
        42.44 & 
        1.22 &
        1052 & 
        9.79\\
        \ouralgo{}-$\texttt{T}\texttt{F}$ incl. time for obtaining $\tree$ &
        59.75 & 
        4.15 & 
        20.21 & 
        1.77 &
        268.42 & 
        43.21 &
        44.24 & 
        3.19\\
        \ouralgo{}-$\texttt{T}\texttt{F}$ excl. time for obtaining $\tree$ &
        0.21 & 
        0.31 & 
        0.001 & 
        0.001 &
        0.08 & 
        0.02 &
        0.01 & 
        0.02\\
        \hline
    \end{tabular}
    }
    \caption{\ouralgo{} computation time {(seconds)} of generating robust CEs for all data points in the training dataset that are classified to class 0 using neural networks (NN) or logistic regressions (LR), when configured with different \texttt{robustInit} {settings}. The numbers following the dataset names represent the approximate number of inputs for each dataset.}
    \label{tab:timeforrobustinitall}
\end{table}

%In Table \ref{tab:results2-NN}, we listed the computation time results for \texttt{robustInit}$=$\texttt{F} under our experiment setting where CEs are generated for 100 inputs. 

The results empirically support our analysis. The computation times of both RNCE-\texttt{TF} and RNCE-\texttt{FF} for neural networks are much higher than those for logistic regressions due to the fact that neural networks are more complex in terms of model architecture. The computation time of RNCE-\texttt{FF} tends to be lower than RNCE-\texttt{TF} when the number of inputs is smaller because fewer $\Delta$-robustness tests are required. When \texttt{robustInit}$=$\texttt{F}, the time is almost proportional to the number of inputs. When \texttt{robustInit}$=$\texttt{T}, the computation times are almost identical regardless of the number of inputs because building the k-d tree $\tree$ is more time-consuming than querying for CEs. Our \ouralgo{}-$\texttt{T}\texttt{F}$ excl. time for obtaining $\tree$ results show that, when \texttt{robustInit}$=$\texttt{T}, after obtaining $\tree$, the time needed for querying CEs is almost negligible, since for every input, the first query will be the closest $\Delta$-robust CE.

\section{Full Experiment Results}
\label{app:full_experiment_results}

\begin{table*}[h!]
    \centering
    \resizebox{0.85\columnwidth}{!}{
    \begin{tabular}{cccccc|ccccc}

        \cline{2-11}
        %{\textbf{method}} 
        & 
        \textbf{vr}&  
        \textbf{v$\Delta_{val}$}&
        \!\!\!\textbf{v$\Delta_{inc}$}\!\!\!&
        $\ell_1$&
        \!\!\! \textbf{lof} \!\!\!&
        
        \textbf{vr}&  
        \textbf{v$\Delta_{val}$} &
        \!\!\!\textbf{v$\Delta_{inc}$}\!\!\!&
        $\ell_1$ &
        \!\!\! \textbf{lof} \!\!\! \\

        \cline{2-11}
        &
         \multicolumn{5}{c|}{\textbf{adult}} &
         \multicolumn{5}{c}{\textbf{compas}} \\
        
        \hline
        
        \hline
        \!\!\!\!GCE\!\!\!\! & 
        \red{.223} &
        \!\!\!\red{0}\!\!\! &
        \red{0} &
         \!\!\!\red{.003}\!\!\! &
        \blue{.036} & 
        \red{.011} &
        \!\!\!\red{0}\!\!\! &
        \red{0} &
         \!\!\!\red{.007}\!\!\! &
        \blue{.697}\\

        \!\!\!\!PROTO\!\!\!\! & 
        \red{.178} &
        \!\!\!.020\!\!\! &
        \red{0} &
         \!\!\!\red{.002}\!\!\! &
        \blue{.089} & 
        \red{.172} &
        \!\!\!\red{.055}\!\!\! &
        \red{0} &
         \!\!\!\red{.062}\!\!\! &
        \blue{.192}\\

        \!\!\!\!MCE\!\!\!\! & 
        \red{.201} &
        \!\!\!\red{0}\!\!\! &
        \red{0} &
         \!\!\!\red{.001}\!\!\! &
        \blue{.063} & 
        \red{.126} &
        \!\!\!\red{0}\!\!\! &
        \red{0} &
         \!\!\!\red{.006}\!\!\! &
        \blue{.047}\\

        \!\!\!\!NNCE\!\!\!\! & 
        \red{.139} &
        \!\!\!\red{.024}\!\!\! &
        \blue{.024} &
         \!\!\!\red{.004}\!\!\! &
        \blue{.114} & 
        \red{.123} &
        \!\!\!\red{.024}\!\!\! &
        \red{0} &
         \!\!\!\red{.003}\!\!\! &
        \blue{.057}\\

        \hline

        \!\!\!\!ROAR\!\!\!\! & 
        \red{0} &
        \!\!\!\red{0}\!\!\! &
        \red{.065} &
         \!\!\!\red{.208}\!\!\! &
        \blue{4.187} & 
        \red{0} &
        \!\!\!\red{0}\!\!\! &
        \red{.033} &
         \!\!\!\red{.033}\!\!\! &
        \blue{.574}\\

        \!\!\!\!RBR\!\!\!\! & 
        \red{.092} &
        \!\!\!\red{0}\!\!\! &
        \blue{0} &
         \!\!\!\red{.002}\!\!\! &
        \red{.051} & 
        \red{.008} &
        \!\!\!\red{.024}\!\!\! &
        \red{0} &
         \!\!\!\red{.003}\!\!\! &
        \blue{.170}\\
        
        \!\!\!\!ST-CE\!\!\!\! & 
        .014 &
        \!\!\!.071\!\!\! &
        \red{.020} &
         \!\!\!\blue{.008}\!\!\! &
        \blue{.087} & 
        \red{.003} &
        \!\!\!\red{.128}\!\!\! &
        \red{0} &
         \!\!\!\red{.003}\!\!\! &
        \blue{.057}\\
        
        \hline
                &
         \multicolumn{5}{c|}{target $\delta_{val}=0.02$} &
         \multicolumn{5}{c}{target $\delta_{val}=0.02$} \\
    
        \!\!\!\!GCE-R\!\!\!\! & 
        0 &
        \!\!\!0\!\!\! &
        \red{0} &
         \!\!\!\blue{.007}\!\!\! &
        \blue{.079} & 
        \red{.041} &
        \!\!\!\red{.040}\!\!\! &
        \red{0} &
         \!\!\!\red{.006}\!\!\! &
        \blue{.664}\\

        \!\!\!\!PROTO-R \!\!\!\! & 
        0&
        \!\!\!.203\!\!\! &
        \red{0} &
         \!\!\!\blue{.010}\!\!\! &
        \blue{.119} & 
        \red{.037} &
        \!\!\!\red{.037}\!\!\! &
        \red{0} &
         \!\!\!\red{.008}\!\!\! &
        \blue{.299}\\

        \!\!\!\!MCE-R \!\!\!\! & 
        0&
        \!\!\!0\!\!\! &
        \red{0} &
         \!\!\!\blue{.001}\!\!\! &
        \blue{.099} & 
        \red{.002} &
        \!\!\!\red{0}\!\!\! &
        \red{0} &
         \!\!\!\red{.003}\!\!\! &
        \blue{.161}\\
        
        \!\!\!\!\ouralgo{}-\texttt{FF} \!\!\!\! & 
        0&
        \!\!\!0\!\!\! &
        \red{.020} &
         \!\!\!\blue{.005}\!\!\! &
        \blue{.021} & 
        \red{0} &
        \!\!\!\red{0}\!\!\! &
        \red{0} &
         \!\!\!\red{.003}\!\!\! &
        \blue{.062}\\

        \!\!\!\!\ouralgo{}-\texttt{FT} \!\!\!\! & 
        0 &
        \!\!\!0\!\!\! &
        \red{0} &
         \!\!\!.003\!\!\! &
        .045 & 
        \red{0} &
        \!\!\!\red{0}\!\!\! &
        \red{0} &
         \!\!\!\red{.003}\!\!\! &
        \blue{.060}\\
        
        \hline
           &
         \multicolumn{5}{c|}{target $\delta_{inc}=0.061$} &
         \multicolumn{5}{c}{target $\delta_{inc}=0.079$} \\
         
        \!\!\!\!GCE-R\!\!\!\! & 
        0 &
        \!\!\!0\!\!\! &
        \red{0} &
         \!\!\!\blue{.008}\!\!\! &
        \blue{.130} & 
        \red{.040} &
        \!\!\!\red{.040}\!\!\! &
        \red{.051} &
         \!\!\!\red{.005}\!\!\! &
        \blue{.565}\\

        \!\!\!\!PROTO-R \!\!\!\! & 
        0&
        \!\!\!0\!\!\! &
        \red{.091} &
         \!\!\!\blue{.053}\!\!\! &
        \blue{.356} & 
        \red{.037} &
        \!\!\!\red{.037}\!\!\! &
        \red{.058} &
         \!\!\!\red{.009}\!\!\! &
        \blue{.244}\\

        \!\!\!\!MCE-R \!\!\!\! & 
        0&
        \!\!\!0\!\!\! &
        \red{0} &
         \!\!\!\blue{.005}\!\!\! &
        \blue{.054} & 
        \red{0} &
        \!\!\!\red{0}\!\!\! &
        \red{0} &
         \!\!\!\red{.004}\!\!\! &
        \blue{.235}\\
        
        \!\!\!\!\ouralgo{}-\texttt{FF} \!\!\!\! & 
        0&
        \!\!\!0\!\!\! &
        \red{0} &
         \!\!\!\blue{.006}\!\!\! &
        \blue{.168} & 
        \red{0} &
        \!\!\!\red{0}\!\!\! &
        \red{0} &
         \!\!\!\red{.004}\!\!\! &
        \blue{.086}\\

        \!\!\!\!\ouralgo{}-\texttt{FT} \!\!\!\! & 
        0 &
        \!\!\!0\!\!\! &
        \red{0} &
         \!\!\!.011\!\!\! &
        .085 & 
        \red{0} &
        \!\!\!\red{0}\!\!\! &
        \red{0} &
         \!\!\!\red{.004}\!\!\! &
        \blue{.086}\\
        
%%%%%%%%%%%%%%%%%%%%%%%%%%%%%%%%%%%%%%%%%%%
        \cline{2-11} &
         \multicolumn{5}{c|}{\textbf{gmc}} &
         \multicolumn{5}{c}{\textbf{heloc}} \\
        
        \hline
        \!\!\!\!GCE\!\!\!\! & 
        \red{.003} &
        \!\!\!\red{0}\!\!\! &
        \red{0} &
         \!\!\!\red{.005}\!\!\! &
        \blue{.064} & 
        \red{.042} &
        \!\!\!\red{0}\!\!\! &
        \red{0} &
         \!\!\!\red{.002}\!\!\! &
        \blue{.045}\\

        \!\!\!\!PROTO\!\!\!\! & 
        \red{.012} &
        \!\!\!.020\!\!\! &
        \red{0} &
         \!\!\!\red{.004}\!\!\! &
        \blue{.057} & 
        \red{.052} &
        \!\!\!\red{0}\!\!\! &
        \red{0} &
         \!\!\!\red{.002}\!\!\! &
        \blue{.036}\\

        \!\!\!\!MCE\!\!\!\! & 
        \red{.032} &
        \!\!\!\red{0}\!\!\! &
        \red{0} &
         \!\!\!\red{.003}\!\!\! &
        \blue{.103} & 
        \red{.053} &
        \!\!\!\red{0}\!\!\! &
        \red{0} &
         \!\!\!\red{.001}\!\!\! &
        \blue{.034}\\

        \!\!\!\!NNCE\!\!\!\! & 
        \red{.027} &
        \!\!\!\red{.068}\!\!\! &
        \blue{.020} &
         \!\!\!\red{.003}\!\!\! &
        \blue{.031} & 
        \red{.053} &
        \!\!\!\red{0}\!\!\! &
        \red{0} &
         \!\!\!\red{.003}\!\!\! &
        \blue{.011}\\

        \hline

        \!\!\!\!ROAR\!\!\!\! & 
        \red{.015} &
        \!\!\!\red{.025}\!\!\! &
        \red{.025} &
         \!\!\!\red{.014}\!\!\! &
        \blue{9.69} & 
        \red{0} &
        \!\!\!\red{0}\!\!\! &
        \red{0} &
         \!\!\!\red{.010}\!\!\! &
        \blue{.164}\\

        \!\!\!\!RBR\!\!\!\! & 
        \red{0} &
        \!\!\!\red{.117}\!\!\! &
        \blue{0} &
         \!\!\!\red{.004}\!\!\! &
        \red{.105} & 
        \red{.070} &
        \!\!\!\red{0}\!\!\! &
        \red{0} &
         \!\!\!\red{.003}\!\!\! &
        \blue{.022}\\
        
        \!\!\!\!ST-CE\!\!\!\! & 
        0 &
        \!\!\!.087\!\!\! &
        \red{.058} &
         \!\!\!\blue{.005}\!\!\! &
        \blue{.068} & 
        \red{0} &
        \!\!\!\red{.152}\!\!\! &
        \red{0} &
         \!\!\!\red{.004}\!\!\! &
        \blue{.016}\\
        
        \hline
                &
         \multicolumn{5}{c|}{target $\delta_{val}=0.02$} &
         \multicolumn{5}{c}{target $\delta_{val}=0.02$} \\
    
        \!\!\!\!GCE-R\!\!\!\! & 
        0 &
        \!\!\!0\!\!\! &
        \red{0} &
         \!\!\!\blue{.006}\!\!\! &
        \blue{.019} & 
        \red{0} &
        \!\!\!\red{0}\!\!\! &
        \red{0} &
         \!\!\!\red{.003}\!\!\! &
        \blue{.043}\\

        \!\!\!\!PROTO-R \!\!\!\! & 
        0&
        \!\!\!0\!\!\! &
        \red{0} &
         \!\!\!\blue{.004}\!\!\! &
        \blue{.019} & 
        \red{0} &
        \!\!\!\red{0}\!\!\! &
        \red{.073} &
         \!\!\!\red{.002}\!\!\! &
        \blue{.061}\\

        \!\!\!\!MCE-R \!\!\!\! & 
        0&
        \!\!\!0\!\!\! &
        \red{0} &
         \!\!\!\blue{.003}\!\!\! &
        \blue{.103} & 
        \red{.002} &
        \!\!\!\red{0}\!\!\! &
        \red{0} &
         \!\!\!\red{.001}\!\!\! &
        \blue{.043}\\
        
        \!\!\!\!\ouralgo{}-\texttt{FF} \!\!\!\! & 
        0&
        \!\!\!0\!\!\! &
        \red{.087} &
         \!\!\!\blue{.004}\!\!\! &
        \blue{.033} & 
        \red{0} &
        \!\!\!\red{0}\!\!\! &
        \red{0} &
         \!\!\!\red{.002}\!\!\! &
        \blue{.019}\\

        \!\!\!\!\ouralgo{}-\texttt{FT} \!\!\!\! & 
        0 &
        \!\!\!0\!\!\! &
        \red{0} &
         \!\!\!.004\!\!\! &
        .157 & 
        \red{0} &
        \!\!\!\red{0}\!\!\! &
        \red{0} &
         \!\!\!\red{.003}\!\!\! &
        \blue{.019}\\
        
        \hline
           &
         \multicolumn{5}{c|}{target $\delta_{inc}=0.091$} &
         \multicolumn{5}{c}{target $\delta_{inc}=0.04$} \\
         
        \!\!\!\!GCE-R\!\!\!\! & 
        0 &
        \!\!\!0\!\!\! &
        \red{0} &
         \!\!\!\blue{.006}\!\!\! &
        \blue{.021} & 
        \red{0} &
        \!\!\!\red{0}\!\!\! &
        \red{0} &
         \!\!\!\red{.010}\!\!\! &
        \blue{.114}\\

        \!\!\!\!PROTO-R \!\!\!\! & 
        0&
        \!\!\!0\!\!\! &
        \red{0} &
         \!\!\!\blue{.011}\!\!\! &
        \blue{.014} & 
        \red{0} &
        \!\!\!\red{0}\!\!\! &
        \red{0} &
         \!\!\!\red{.003}\!\!\! &
        \blue{.025}\\

        \!\!\!\!MCE-R \!\!\!\! & 
        0&
        \!\!\!0\!\!\! &
        \red{0} &
         \!\!\!\blue{.003}\!\!\! &
        \blue{.100} & 
        \red{0} &
        \!\!\!\red{0}\!\!\! &
        \red{0} &
         \!\!\!\red{.002}\!\!\! &
        \blue{.055}\\
        
        \!\!\!\!\ouralgo{}-\texttt{FF} \!\!\!\! & 
        0&
        \!\!\!0\!\!\! &
        \red{0} &
         \!\!\!\blue{.007}\!\!\! &
        \blue{.036} & 
        \red{0} &
        \!\!\!\red{0}\!\!\! &
        \red{0} &
         \!\!\!\red{.003}\!\!\! &
        \blue{.022}\\

        \!\!\!\!\ouralgo{}-\texttt{FT} \!\!\!\! & 
        0 &
        \!\!\!0\!\!\! &
        \red{0} &
         \!\!\!.007\!\!\! &
        .036 & 
        \red{0} &
        \!\!\!\red{0}\!\!\! &
        \red{0} &
         \!\!\!\red{.003}\!\!\! &
        \blue{.022}\\
         \hline
    \end{tabular}
    }
    \caption{Standard deviation results of the quantitative evaluation of the compared CE generation methods on neural networks.} % Results worse (better) than our \ouralgo{}-$\texttt{F}\texttt{T}$ method are highlighted in \red{red} (\blue{blue}).}
    \label{tab:results2-NN-std}
\end{table*}

\begin{table*}[h!]
    \centering
    \resizebox{0.95\columnwidth}{!}{
    \begin{tabular}{cccccc|ccccc}

        \cline{2-11}& 
        \textbf{vr}$\uparrow$&  
        \textbf{v$\Delta_{val}$$\uparrow$} &
        \!\!\!\textbf{v$\Delta_{inc}$$\uparrow$}\!\!\!&
        $\ell_1$$\downarrow$ &
        \!\!\! \textbf{lof$\downarrow$} \!\!\!&
        
        \textbf{vr}$\uparrow$&  
        \textbf{v$\Delta_{val}$$\uparrow$} &
        \!\!\!\textbf{v$\Delta_{inc}$$\uparrow$}\!\!\!&
        $\ell_1$$\downarrow$ &
        \!\!\! \textbf{lof$\downarrow$} \!\!\! \\

        \cline{2-11}
        &
         \multicolumn{5}{c|}{\textbf{adult}} &
         \multicolumn{5}{c}{\textbf{compas}} \\
        
        \hline

        \!\!\!\!NNCE\!\!\!\! & 
        \red{97.9\%} &
        \!\!\!\red{45\%}\!\!\! &
        \blue{42\%} &
         \!\!\!\red{.078}\!\!\! &
        \blue{1.99} & 
        \red{73.6\%}&
        \!\!\!\red{1\%}\!\!\! &
        \blue{49\%} &
         \!\!\!\red{.034}\!\!\! &
        \blue{1.32}\\

        \hline

        \!\!\!\!ROAR\!\!\!\! & 
        \red{100\%} &
        \!\!\!\red{100\%}\!\!\! &
        \red{100\%} &
         \!\!\!\red{.265}\!\!\! &
        \blue{1.69} & 
        \red{100\%}&
        \!\!\!{100\%}\!\!\! &
        \red{96.8\%} &
         \!\!\!\red{.220}\!\!\!&
        \blue{2.59}\\
        
        \hline
                &
         \multicolumn{5}{c|}{target $\delta_{val}=0.04$} &
         \multicolumn{5}{c}{target $\delta_{val}=0.08$} \\

        \!\!\!\!\ouralgo{}-\texttt{FF} \!\!\!\! & 
        100\% &
        \!\!\!100\%\!\!\! &
        \red{70\%} &
         \!\!\!\red{.085}\!\!\! &
        \blue{2.21} & 
        100\%&
        \!\!\!100\%\!\!\! &
        \red{100\%} &
         \!\!\!\blue{.045}\!\!\!&
       \blue{1.22}\\

        \!\!\!\!\ouralgo{}-\texttt{FT} \!\!\!\! & 
        100\%&
        \!\!\!100\%\!\!\! &
        \red{16\%} &
         \!\!\!\blue{.060}\!\!\!&
       \blue{1.71} &
        100\%&
        \!\!\!100\%\!\!\! &
        100\% &
         \!\!\!.043\!\!\! &
        1.30\\
        
        \hline
           &
         \multicolumn{5}{c|}{target $\delta_{inc}=0.063$} &
         \multicolumn{5}{c}{target $\delta_{inc}=0.01$} \\

        \!\!\!\!\ouralgo{}-\texttt{FF} \!\!\!\! & 
        100\% &
        \!\!\!100\%\!\!\! &
        \red{100\%} &
         \!\!\!\red{.087}\!\!\! &
        \blue{2.18} & 
        86.1\%&
        \!\!\!1\%\!\!\! &
        \red{100\%} &
         \!\!\!\blue{.035}\!\!\!&
       \blue{1.30}\\

        \!\!\!\!\ouralgo{}-\texttt{FT} \!\!\!\! & 
        100\% &
        \!\!\!100\%\!\!\! &
        \red{100\%} &
         \!\!\!\red{.064}\!\!\! &
        \blue{1.74} & 
        78.1\%&
        \!\!\!0\%\!\!\! &
        100\% &
         \!\!\!.033\!\!\! &
        1.34\\
        
%%%%%%%%%%%%%%%%%%%%%%%%%%%%%%%%%%%%%%%%%%%
        \cline{2-11} &
         \multicolumn{5}{c|}{\textbf{gmc}} &
         \multicolumn{5}{c}{\textbf{heloc}} \\
        
        \hline

        \!\!\!\!NNCE\!\!\!\! & 
        \red{92.7\%} &
        \!\!\!\red{57\%}\!\!\! &
        \blue{51\%} &
         \!\!\!\red{.035}\!\!\! &
        \blue{1.21} & 
        \red{77.5\%}&
        \!\!\!\red{5\%}\!\!\! &
        \blue{49\%} &
         \!\!\!\red{.057}\!\!\! &
        \blue{1.06}\\

        \hline

        \!\!\!\!ROAR\!\!\!\! & 
        \red{100\%} &
        \!\!\!\red{100\%}\!\!\! &
        \red{100\%} &
         \!\!\!\red{.119}\!\!\! &
        \blue{2.76} & 
        \red{100\%}&
        \!\!\!{100\%}\!\!\! &
        \red{100\%} &
         \!\!\!\red{.089}\!\!\!&
        \blue{1.42}\\

        \hline
                &
         \multicolumn{5}{c|}{target $\delta_{val}=0.06$} &
         \multicolumn{5}{c}{target $\delta_{val}=0.07$} \\

        \!\!\!\!\ouralgo{}-\texttt{FF} \!\!\!\! & 
        100\% &
        \!\!\!100\%\!\!\! &
        \red{92\%} &
         \!\!\!\red{.043}\!\!\! &
        \blue{1.23} & 
        100\% &
        \!\!\!100\%\!\!\! &
        \red{100\%} &
         \!\!\!\blue{.070}\!\!\!&
       \blue{1.06}\\

        \!\!\!\!\ouralgo{}-\texttt{FT} \!\!\!\! & 
        100\% &
        \!\!\!100\%\!\!\! &
        \red{35\%} &
         \!\!\!.034\!\!\! &
        1.23 & 
        100\% &
        \!\!\!100\%\!\!\! &
        \red{0\%} &
         \!\!\!.066\!\!\! &
        1.05\\
        
        \hline
           &
         \multicolumn{5}{c|}{target $\delta_{inc}=0.079$} &
         \multicolumn{5}{c}{target $\delta_{inc}=0.019$} \\
         
        \!\!\!\!\ouralgo{}-\texttt{FF} \!\!\!\! & 
        100\% &
        \!\!\!100\%\!\!\! &
        \red{100\%} &
         \!\!\!\red{.043}\!\!\! &
        \blue{1.17} & 
        98.8\% &
        \!\!\!10\%\!\!\! &
        \red{100\%} &
         \!\!\!\blue{.061}\!\!\!&
       \blue{1.06}\\

        \!\!\!\!\ouralgo{}-\texttt{FT} \!\!\!\! & 
        100\% &
        \!\!\!100\%\!\!\! &
        \red{100\%} &
         \!\!\!.035\!\!\! &
        1.23 & 
        96.7\% &
        \!\!\!0\%\!\!\! &
        \red{100\%} &
         \!\!\!.054\!\!\! &
        1.06\\ \hline
    \end{tabular}
    }
    \caption{Quantitative evaluation of the compared CE generation methods on logistic regression models. The evaluation metrics are the same as in Table~\ref{tab:results2-NN}.} % Results worse (better) than our \ouralgo{}-$\texttt{F}\texttt{T}$ method are highlighted in \red{red} (\blue{blue}).}
    \label{tab:results2-LR}
\end{table*}

\begin{table*}[h!]
    \centering
    \resizebox{0.85\columnwidth}{!}{
    \begin{tabular}{cccccc|ccccc}

        \cline{2-11}
        %{\textbf{method}} 
        & 
        \textbf{vr}&  
        \textbf{v$\Delta_{val}$}&
        \!\!\!\textbf{v$\Delta_{inc}$}\!\!\!&
        $\ell_1$&
        \!\!\! \textbf{lof} \!\!\!&
        
        \textbf{vr}&  
        \textbf{v$\Delta_{val}$} &
        \!\!\!\textbf{v$\Delta_{inc}$}\!\!\!&
        $\ell_1$ &
        \!\!\! \textbf{lof} \!\!\! \\

        \cline{2-11}
        &
         \multicolumn{5}{c|}{\textbf{adult}} &
         \multicolumn{5}{c}{\textbf{compas}} \\
        
        \hline

        \!\!\!\!NNCE\!\!\!\! & 
        \red{.035} &
        \!\!\!\red{.130}\!\!\! &
        \blue{.144} &
         \!\!\!\red{.017}\!\!\! &
        \blue{.188} & 
        \red{.025}&
        \!\!\!\red{.020}\!\!\! &
        \blue{.037} &
         \!\!\!\red{.003}\!\!\! &
        \blue{.068}\\

        \hline

        \!\!\!\!ROAR\!\!\!\! & 
        \red{0} &
        \!\!\!\red{0}\!\!\! &
        \red{0} &
         \!\!\!\red{.028}\!\!\! &
        \blue{.020} & 
        \red{0}&
        \!\!\!{0}\!\!\! &
        \red{0} &
         \!\!\!\red{.002}\!\!\!&
        \blue{.151}\\
        
        \hline
                &
         \multicolumn{5}{c|}{target $\delta_{val}=0.04$} &
         \multicolumn{5}{c}{target $\delta_{val}=0.08$} \\

        \!\!\!\!\ouralgo{}-\texttt{FF} \!\!\!\! & 
        0 &
        \!\!\!0\!\!\! &
        \red{.122} &
         \!\!\!\red{.016}\!\!\! &
        \blue{.251} & 
        0&
        \!\!\!0\!\!\! &
        \red{0} &
         \!\!\!\blue{.003}\!\!\!&
       \blue{.032}\\

        \!\!\!\!\ouralgo{}-\texttt{FT} \!\!\!\! & 
        0 &
        \!\!\!0\!\!\! &
        \red{.092} &
         \!\!\!\red{.010}\!\!\! &
        \blue{.070} & 
        0&
        \!\!\!0\!\!\! &
        \red{0} &
         \!\!\!\blue{.002}\!\!\!&
       \blue{.060}\\
        
        \hline
           &
         \multicolumn{5}{c|}{target $\delta_{inc}=0.063$} &
         \multicolumn{5}{c}{target $\delta_{inc}=0.01$} \\

        \!\!\!\!\ouralgo{}-\texttt{FF} \!\!\!\! & 
        0 &
        \!\!\!0\!\!\! &
        \red{0} &
         \!\!\!\red{.015}\!\!\! &
        \blue{.273} & 
        .015&
        \!\!\!.020\!\!\! &
        0 &
         \!\!\!.003\!\!\! &
        .068\\

        \!\!\!\!\ouralgo{}-\texttt{FT} \!\!\!\! & 
        0 &
        \!\!\!0\!\!\! &
        \red{0} &
         \!\!\!\red{.009}\!\!\! &
        \blue{.104} & 
        .002&
        \!\!\!0\!\!\! &
        0 &
         \!\!\!.003\!\!\! &
        .093\\
        
%%%%%%%%%%%%%%%%%%%%%%%%%%%%%%%%%%%%%%%%%%%
        \cline{2-11} &
         \multicolumn{5}{c|}{\textbf{gmc}} &
         \multicolumn{5}{c}{\textbf{heloc}} \\
        
        \hline

        \!\!\!\!NNCE\!\!\!\! & 
        \red{.028} &
        \!\!\!\red{.129}\!\!\! &
        \blue{.159} &
         \!\!\!\red{.003}\!\!\! &
        \blue{.021} & 
        \red{.023}&
        \!\!\!\red{.045}\!\!\! &
        \blue{.102} &
         \!\!\!\red{.003}\!\!\! &
        \blue{.013}\\

        \hline

        \!\!\!\!ROAR\!\!\!\! & 
        \red{0} &
        \!\!\!\red{0}\!\!\! &
        \red{0} &
         \!\!\!\red{.010}\!\!\! &
        \blue{.225} & 
        \red{0}&
        \!\!\!{0}\!\!\! &
        \red{0} &
         \!\!\!\red{.003}\!\!\!&
        \blue{.038}\\

        \hline
                &
         \multicolumn{5}{c|}{target $\delta_{val}=0.06$} &
         \multicolumn{5}{c}{target $\delta_{val}=0.07$} \\

        \!\!\!\!\ouralgo{}-\texttt{FF} \!\!\!\! & 
        0 &
        \!\!\!0\!\!\! &
        \red{.112} &
         \!\!\!\red{.004}\!\!\! &
        \blue{.055} & 
        0 &
        \!\!\!0\!\!\! &
        \red{0} &
         \!\!\!\blue{.003}\!\!\!&
       \blue{.021}\\

        \!\!\!\!\ouralgo{}-\texttt{FT} \!\!\!\! & 
        0 &
        \!\!\!0\!\!\! &
        \red{.055} &
         \!\!\!\red{.004}\!\!\! &
        \blue{.048} & 
        0 &
        \!\!\!0\!\!\! &
        \red{0} &
         \!\!\!\blue{.004}\!\!\!&
       \blue{.019}\\

        \hline
           &
         \multicolumn{5}{c|}{target $\delta_{inc}=0.079$} &
         \multicolumn{5}{c}{target $\delta_{inc}=0.019$} \\
         
        \!\!\!\!\ouralgo{}-\texttt{FF} \!\!\!\! & 
        0 &
        \!\!\!0\!\!\! &
        \red{0} &
         \!\!\!\red{.003}\!\!\! &
        \blue{.039} & 
        .010 &
        \!\!\!.008\!\!\! &
        0 &
         \!\!\!\blue{.003}\!\!\!&
       \blue{.016}\\

        \!\!\!\!\ouralgo{}-\texttt{FT} \!\!\!\! & 
        0 &
        \!\!\!0\!\!\! &
        \red{0} &
         \!\!\!\red{.004}\!\!\! &
        \blue{.053} & 
        .012 &
        \!\!\!0\!\!\! &
        0 &
         \!\!\!\blue{.004}\!\!\!&
       \blue{.015}\\ \hline
    \end{tabular}
    }
    \caption{Standard deviation results of the quantitative evaluation of the compared CE generation methods on logistic regression models.} % Results worse (better) than our \ouralgo{}-$\texttt{F}\texttt{T}$ method are highlighted in \red{red} (\blue{blue}).}
    \label{tab:results2-LR-std}
\end{table*}

%======================================

Standard deviation results accompanying Table~\ref{tab:results2-NN} are presented in Table~\ref{tab:results2-NN-std}.

For linear regression models, we find $\delta_{inc}$ and $\delta_{val}$ using the same strategy for neural networks as stated in Section~\ref{ssec:experiments_identify_delta_values}. We quantitatively compare the CEs found by NNCE, ROAR, and two configurations of our algorithm RNCE using the same evaluation metrics introduced in Section~\ref{ssec:experiments_benchmarking}. Tables~\ref{tab:results2-LR} and \ref{tab:results2-LR-std} report the mean and standard deviation results for logistic regression classifiers.

Slightly different from the neural network results, the $\delta_{inc}$ values, found by incrementally retraining on 10\% of $\dataset_2$, can be smaller than the $\delta_{val}$, and sometimes insufficient to induce 100\% empirical robustness (indicated by vr). Comparing RNCE with baselines, similar to the results for neural networks in Section~\ref{ssec:experiments_benchmarking}, RNCE produces more robust CEs than the NNCE method, and is less costly and more plausible than ROAR. The impact of changing \texttt{optimal} to \texttt{True} is more obvious in this set of experiments, with $\ell_1$ costs decreasing to a greater extent and plausibility remaining comparable to RNCE-\texttt{FF}.

\end{document}